\documentclass[11pt,a4paper]{article}

\usepackage[utf8]{inputenc} 
\usepackage[T1]{fontenc}    
\usepackage{hyperref}       
\usepackage{url}            
\usepackage{booktabs}       
\usepackage{amsfonts}       
\usepackage{nicefrac}       
\usepackage{microtype}      
\usepackage{xcolor}         
\usepackage{amsmath, amssymb,amsthm}

\usepackage{graphicx, graphics, epsfig, color}
\usepackage{adjustbox}
\usepackage{tikz, pgfplots}
\usetikzlibrary{plotmarks}
\pgfplotsset{compat=newest}
\usepackage[margin=1in]{geometry}
\usepackage{float}
\usepackage{multirow}
\usepackage{cuted, flushend}
\usepackage{midfloat}
\usepackage{bm}
\usepackage{fancyhdr}
\usepackage{enumerate}
\usepackage{siunitx}
\usepackage[numbers,square,sort&compress]{natbib}
\usepackage{hyperref}
\usepackage{cleveref}
\usepackage{subcaption}
\usepackage{url}
\usepackage[skins,xparse,breakable]{tcolorbox}

\DeclareMathOperator{\tr}{ {\rm tr} }
\newcommand{\RR}{\mathbb{R}}
\newcommand{\T}{ {\sf T} }
\newcommand{\E}{\mathbb{E}}

\newcommand{\cov}{{\rm Cov}}
\newcommand{\var}{{\rm Var}}
\newcommand{\x}{{\mathbf{x}}}
\newcommand{\zeros}{{\mathbf{0}}}
\newcommand{\ones}{{\mathbf{1}}}
\newcommand{\h}{{\mathbf{h}}}
\newcommand{\e}{{\mathbf{e}}}

\newcommand{\z}{{\mathbf{z}}}
\newcommand{\g}{{\mathbf{g}}}

\renewcommand{\t}{{\mathbf{t}}}
\newcommand{\X}{{\mathbf{X}}}

\newcommand{\Q}{{\mathbf{Q}}}

\newcommand{\V}{{\mathbf{V}}}
\newcommand{\W}{{\mathbf{W}}}
\newcommand{\G}{{\mathbf{G}}}
\newcommand{\M}{{\mathbf{M}}}
\newcommand{\bmu}{ \boldsymbol{\mu} }
\newcommand{\bbeta}{ \boldsymbol{\beta} }
\newcommand{\bxi}{ \boldsymbol{\xi} }

\newcommand{\bnu}{ \boldsymbol{\nu} }

\newcommand{\bpsi}{ \boldsymbol{\psi} }
\newcommand{\bepsilon}{ \boldsymbol{\epsilon} }

\newcommand{\bSigma}{ \boldsymbol{\Sigma} }

\newcommand{\I}{{\mathbf{I}}}

\renewcommand{\b}{{\mathbf{b}}}
\renewcommand{\u}{{\mathbf{u}}}

\renewcommand{\v}{{\mathbf{v}}}

\newcommand{\s}{{\mathbf{s}}}

\renewcommand{\H}{{\mathbf{H}}}

\DeclareMathOperator*{\argmin}{arg\,min}

\DeclareMathOperator{\prox}{{\rm prox}}

\definecolor{RED}{rgb}{0.7,0,0}
\definecolor{BLUE}{rgb}{0,0,0.69}
\definecolor{GREEN}{rgb}{0,0.6,0}
\definecolor{PURPLE}{rgb}{0.69,0,0.8}

\newcommand{\RED}{\color[rgb]{0.70,0,0}}
\newcommand{\BLUE}{\color[rgb]{0,0,0.69}}

\newtheorem{Theorem}{Theorem}

\newtheorem{Lemma}{Lemma}
\newtheorem{Remark}{Remark}
\newtheorem{Corollary}{Corollary}
\newtheorem{Definition}{Definition}
\newtheorem{Assumption}{Assumption}

\title{The Breakdown of Gaussian Universality in Classification of High-dimensional Linear Factor Mixtures}

\author{
Xiaoyi Mai\\
Institut de Mathématiques de Toulouse (IMT)\\
University of Toulouse-Jean Jaurès, Toulouse, France\\
\texttt{xiaoyi.mai@math.univ-toulouse.fr}
\and
Zhenyu Liao\\
School of Electronic Information and Communications (EIC)\\
Huazhong University of Science and Technology, Wuhan, China\\
\texttt{zhenyu\_liao@hust.edu.cn}
}

\begin{document}

\maketitle

\begin{abstract}
The assumption of Gaussian or Gaussian mixture data has been extensively exploited in a long series of precise performance analyses of machine learning (ML) methods, on large datasets having comparably numerous samples and features. 
To relax this restrictive assumption, subsequent efforts have been devoted to establish ``Gaussian equivalent principles'' by studying scenarios of Gaussian universality where the asymptotic performance of ML methods on non-Gaussian data remains unchanged when replaced with Gaussian data having the \emph{same mean and covariance}.
Beyond the realm of Gaussian universality, there are few exact results on how the data distribution affects the learning performance. 

In this article, we provide a precise high-dimensional characterization of empirical risk minimization, for classification under a general mixture data setting of \emph{linear factor models} that extends Gaussian mixtures. 
The Gaussian universality is shown to break down under this setting, in the sense that the asymptotic learning performance depends on the data distribution \emph{beyond} the class means and covariances.
To clarify the limitations of Gaussian universality in the classification of mixture data and to understand the impact of its breakdown, we specify conditions for Gaussian universality and discuss their implications for the choice of loss function.  
\end{abstract}

\section{Introduction}

Modern machine learning (ML) is dealing with increasingly larger datasets having high-dimensional features, using large-scale models of increasing complexity. 
Understanding the generalization ability of these large-scale ML models has become a major focus of research efforts~\citep{bartlett2020benign,loog2020brief,nakkiran2021deep}. 
One analysis approach of growing popularity is the high-dimensional asymptotic analysis~\citep{liao2019large,taheri2021fundamental,celentano2022fundamental,hastie2022surprises,loureiro2022fluctuations,celentano2023lasso}, by considering the regime where the number $n$ of samples and the dimension $p$ of feature vectors are commensurately large.  
Despite its asymptotic nature, this approach turns out to be surprisingly effective in explaining and predicting modern ML practice: the asymptotic performance curves are repetitively observed to closely match the average empirical performance curves on realistic datasets of moderate size and dimension~\citep{couillet2022random}, and are particularly \emph{different} from those offered by, e.g., classical maximum likelihood theory~\citep{bean2013optimal,sur2019modern,taheri2021sharp}.

To analytically characterize the generalization performance of large-scale ML models in the aforementioned high-dimensional regime, advanced statistical tools such as the approximate message passing~\citep{donoho2016high,barbier2019optimal}, convex Gaussian min-max theorem~\citep{thrampoulidis2018precise,salehi2019impact,deng2022model,javanmard2022precise}, replica method~\citep{huang2017asymptotic,gerace2020generalisation,maillard2020phase}, and random matrix theory (RMT)~\citep{couillet2022random,mai2019large,mai2021consistent} have been carefully adapted to take into account of the nonlinearity of ML models.
As these tools apply directly on Gaussian data, a majority of high-dimensional asymptotic analyses are performed under Gaussian data models in the context of regression~\citep{el2013robust,donoho2016high,taheri2021fundamental,celentano2022fundamental} or Gaussian mixture models (GMMs) in the context of classification~\citep{mignacco2020role,thrampoulidis2020theoretical,refinetti2021classifying}.

Despite this seemingly restrictive assumption of data Gaussianity, the derived high-dimensional asymptotic results have been empirically observed to remain valid on non-Gaussian data, including both synthetic non-Gaussian data and samples drawn from realistic (say image) datasets~\citep{loureiro2021learning,taheri2021sharp}, hinting at a phenomenon of \emph{Gaussian universality}. 
This motivated a series of recent works establishing, through, e.g., an one-directional central limit theorem (CLT) argument, a Gaussian equivalent principal (GEP) for high-dimensional ML models ranging from generalized linear models to shallow neural networks~\citep{gerace2020generalisation,goldt2022gaussian,hu2022universality,montanari2022universality,schroder2023deterministic,han2023universality}. 
According to the GEP, the performance of ML methods on non-Gaussian data is asymptotically the same under an equivalent Gaussian model with matching first-~and~second-order moments. 
Assuming a conditional one-directional CLT, \citet{dandi2024universality} put forward a conditional Gaussian equivalent principle (CGEP) stating that the asymptotic classification error for non-Gaussian mixtures remains unchanged when replaced by a Gaussian mixture model with identical class-conditional means and covariances. 
This contribution, however, did not specify the conditions required on the mixture data model for this conditional one-directional CLT to hold.

This work is driven by the need to investigate the applicability of CGEP under mixture models and to characterize the impact of non-Gaussian data variations when the CGEP does \emph{not} apply. 
By considering a more general mixture model (see \Cref{def:linear_factor}) where classes are described by linear factor models---a fundamental probabilistic framework in statistical inference and generative models~\citep[Chapter~13]{Goodfellow-et-al-2016}, our analysis extends a long line of high-dimensional asymptotic performance analyses in classification of Gaussian mixtures~\citep{dobriban2018high,huang2017asymptotic,liao2019large,mai2019high,huang2021large,kammoun2021precise,wang2021benign,pesce2023gaussian}.
We discuss the validity of CGEP under this linear factor mixture model and specify its conditions.  
On a technical level, we develop a flexible ``leave-one-out'' analysis approach, in a similar spirit to the analysis of robust linear regression by \citet{el2013robust}. 
The elementary nature of this leave-one-out procedure allows us to extend the approach of high-dimensional asymptotic analysis beyond the realm of Gaussian universality.

\paragraph{Our Contributions.}
The main findings of this paper are summarized below.

\begin{enumerate}
   \item We provide in \Cref{theo:main}  an asymptotic characterization of  ridge-regularized empirical risk minimization (ERM) for classification of data drawn from a linear factor mixture model (LFMM, see \Cref{def:linear_factor} below, that generalizes the GMM).  
  This precise characterization gives access to the asymptotic performance on  mixture data \emph{beyond Gaussian universality}.
   
  \item Based on the proposed analysis, we study Gaussian universality in the sense of CGEP and provide conditions on LFMM under which the data distribution affects the asymptotic learning behavior \emph{only} via its first two class-conditional moments.
  \begin{itemize}
      \item We first discuss in \Cref{sec:universality error} the Gaussian universality on \emph{in-distribution performance} and conclude in \Cref{cor:condition for r} that the training and generalization performances of ERM under a given LFMM remain unchanged under its equivalent GMM (with identical class means and covariances, see \Cref{def:equivalent GMM}), if all \emph{informative factors} of the LFMM that significantly correlated with the class label are \emph{class-conditional normal variables}.
      \item We then investigate in \Cref{sec:universality beta} the Gaussian universality of \emph{classifier} (see \Cref{def:gaussian_universality}) and conclude in \Cref{cor:condition for beta} that on a given test set (of arbitrary distribution), the ERM classifier trained on data drawn from an LFMM gives the same asymptotic classification error as the one trained on its equivalent GMM, whenever the square loss is used and/or in the case of class-conditional Gaussian informative factors for LFMM. 
  \end{itemize} 
   
   \item While it has been known that for high-dimensional GMM, the square loss is optimal in both unregularized~\citep{taheri2021sharp} and ridge-regularized~\citep{mai2019high} classifications, it is  \emph{no longer} the case under the general LFMM due to the breakdown of Gaussian universality. 
   We discuss in \Cref{sec:universality beta} how the suboptimality of square loss under LFMM relates to its particular effect on the Gaussian universality of the ERM classifier.
   Our analysis thus opens the door to future investigation on the optimal loss design for \emph{non-Gaussian} data.
 
\end{enumerate}

The remainder of this paper is organized as follows.  
We give in \Cref{sec: background} a more detailed discussion of related works on Gaussian universality in high dimension.
\Cref{sec: setup} introduces the problem setup and the linear factor mixture model (LFMM) under study. 
In \Cref{sec: results} we present our main theoretical results on the high-dimensional characterization of the empirical risk minimization classifier under LFMM. 
In \Cref{sec: universality} we discuss the consequences and implications of our technical results on the breakdown of Gaussian universality and the choice of loss function.
Conclusions and perspectives are placed in \Cref{sec: conclusion}.

\paragraph{Notations.}
We denote scalars by lowercase letters, vectors by bold lowercase, and matrices by bold uppercase.
For a matrix $\X\in \RR^{p\times n}$, we denote $\X^\T$, $\x_i \in \RR^p$, and $\|\X\|$ the transpose, $i$th column, and spectral norm of $\X$, respectively.
We use $\I_p$ for the identity matrix of size $p$ by $p$.
For a vector $\x \in \RR^p$, the Euclidean norm of $\x$ is given by $\| \x \| = \sqrt{\x^\T \x}$.
For a random variable $x$, we denote $\E[x]$ the expectation of $x$, and $\E[x|\zeta]$ the conditional expectation of $x$ on an event $\zeta$.
For a sequence of random variables $x_n$, we use $x_n= o_P(1)$ to denote that the random variable $x$ converges in probability as $n \to \infty$ as in standard asymptotic statistics, and $O(\cdot), \Theta(\cdot)$ for  standard Big-O and Big-Theta notations.
For $\v_1,\ldots,\v_k\in\RR^p$, the notation ${\rm Span}\{\v_1,\ldots,\v_k\}$ stands for the smallest subspace of $\RR^p$ that contains all linear combinations of $\v_1,\ldots,\v_k$.

\section{Background on Gaussian Universality in High Dimensions}
\label{sec: background}

The Gaussian universality phenomenon was observed in many high-dimensional inference or ML problems, where some key statistics such as estimation error or classification accuracy exhibit universal behaviors independent of the data distribution. 
This phenomenon was extensively exploited to relax the data Gaussianity assumption that underlined many results in high-dimensional statistics, through a universality argument often established with two key ingredients - the law of large numbers (LLN) and the CLT. 

Here we briefly review previous findings on Gaussian universality in the high-dimensional regime.

\paragraph{Universality of large sample covariance matrices.} 
It has been long known in RMT that the eigenspectra of large random matrices enjoy universal properties for Gaussian and non-Gaussian entries~\citep{tao2010random,pastur2011eigenvalue}. 
Fundamentally, \citet{marchenko1967distribution} put forward that for sample covariance matrices of the type $\frac{1}{n}\sum_{i=1}^n\x_i\x_i^\T\in\RR^{p\times p}$ obtained from $n$ i.i.d.\@ data vectors $\x_i$ of dimension $p$, the universality on the limiting eigenvalue distribution hinges on the concentration of quadratic forms of $\x_i$ around their expectations, that is
\begin{equation}
\label{eq:concentration quadratic x}
   \textstyle \lim_{n,p\to\infty} (\x_i^\T\M\x_i-\E[\x_i^\T\M\x_i])/\E[\x_i^\T\M\x_i]=0,
\end{equation}
for deterministic $\M\in\RR^{p\times p}$. 
This LLN-type result holds for a wide family of high-dimensional random vectors $\x_i$.
An important example studied by \citet{bai2008clt} is $\x_i=\bSigma^{\frac{1}{2}}\z_i$ with $\z_i$ of i.i.d.\@ standardized entries with bounded fourth moments and non-negative definite symmetric $\bSigma$.

\paragraph{Universality of empirical risk minimization.} 
In line with the universal behavior of large sample covariance matrices, it has been recently demonstrated in a series of works~\citep{gerace2020generalisation,goldt2022gaussian,hu2022universality,montanari2022universality,schroder2023deterministic} that random (and deterministic under certain conditions) feature maps can produce feature matrices that, when replaced by ``equivalent'' Gaussian features with the same first-~and~second-order moments, yield the same training and/or generalization performance for many ML methods.
This GEP has also been proven for data vectors of independent entries in the context of regularized regression \citep{montanari2017Universality,panahi2017Universal,han2023universality}. 

In the context of ERM, the GEP traced back to a CLT on the inner product $\x^\T\bbeta$ for feature vector $\x\in\RR^p$ independent of classifier $\bbeta$ living in a subspace $\mathcal{B}\subset\RR^p$ containing the ERM solution $\hat\bbeta$, i.e.,
\begin{equation}
\label{eq:clt for GEP}
    \textstyle \lim_{n,p\to\infty} \sup_{\bbeta\in\mathcal{B}} \left(\E[f(\x^\T\bbeta)]-\E[f(\g^\T\bbeta)] \right) =0,
\end{equation}
with $\g\sim\mathcal{N}(\E[\x],\cov[\x])$ being the ``equivalent'' Gaussian vector, for any bounded Lipschitz function $f \colon \RR\to\RR$. 
The one-directional CLT in \eqref{eq:clt for GEP} requires the ERM solution $\hat\bbeta$ to not particularly aligned with any non-Gaussian variation in the feature vector $\x$, in order to ensure the asymptotic normality of $\x^\T\bbeta$ per a CLT-type argument.

\paragraph{Universality of empirical risk minimization on mixture data.} 
Inspired by the one-directional CLT condition \eqref{eq:clt for GEP} for GEP in ERM, \citet{dandi2024universality} demonstrated the Gaussian universality for mixture models under a key assumption that is a conditional version of~\eqref{eq:clt for GEP}:
\begin{equation}
\label{eq:clt for CGEP}
    \textstyle \lim_{n,p\to \infty} \sup_{\bbeta\in\mathcal{B}} \left( \E\left[f(\x^\T\bbeta)\vert y_\x=C\right]-\E\Big[f\big(\g_{[C]}^\T\bbeta\big)\Big] \right)=0,
\end{equation}
where $y_\x$ is the class label of $\x$ , $C$ a class modality, and $\g_{[C]}\sim\mathcal{N}\left(\E[\x\vert y_\x=C],\cov[\x\vert y_\x=C]\right)$.
Under this conditional one-directional CLT in \eqref{eq:clt for CGEP}, \citet{dandi2024universality} showed that the asymptotic training and generalization errors only depend on the class-conditional means and covariances of the mixture model, obeying thus a conditional Gaussian equivalent principle (CGEP). 

However, the condition in \eqref{eq:clt for CGEP} is \emph{not} verifiable from the data distribution, making it essentially different from the ones given in \Cref{sec: universality}.


Other related works established  equivalences between Gaussian data and Gaussian mixtures. For classification with random labels $y_{\x}\sim{\rm Unif}(\{-1,1\})$ generated independently of $\x$, \citet{gerace2024gaussian} proved that the training loss on GMM is asymptotically equal to that on a \emph{single} Gaussian. 
\citet{pesce2023gaussian} considered a teacher-student model and showed that when the target label $y$ is generated by a teacher model \emph{uncorrelated} with cluster means, the same asymptotic performance can be obtained by replacing a homoscedastic (i.e., having identical covariance) Gaussian mixture with a single Gaussian.

\paragraph{Universality under elliptical distributions.} For ``elliptic-like'' data vectors of form $\x=a\M\u$ with $a\in\RR$ a random scaling variable, $\M\in\RR^{p\times d}$ a deterministic matrix and $\u\in\RR^d$ a vector of standardized variables satisfying the concentration of quadratic forms in \eqref{eq:concentration quadratic x} (e.g., $\u\sim\mathcal{N}(\zeros_d,\I_d)$), \citet{el2009concentration} revealed a universal limiting spectrum for the sample covariance matrix that is insensitive to the distribution of $\u$ but depends on the scaling variable $a$.

Due to the existence of the scaling variable $a$, 
the one-directional CLT in \eqref{eq:clt for GEP} can not hold unless when conditioned on $a$. 
This remark was confirmed by the findings of \citet{el2018impact,adomaityte2024classification}. 
For $\M=\I_p$ and $\u$ of i.i.d.\@ entries, \citet{el2018impact} characterized the asymptotic error of ridge-regularized regression, which is universal with respect to the distribution of $\u$ but not $a$. 
In other words, the GEP collapses while the CGEP with respect to the scaling factor $a$ can still apply in this setting. \citet{adomaityte2024classification} considered a mixture model $\x\sim\mathcal{N}(y\bmu,a\I_p)$ with label $y=\pm 1$ and random scaling factor $a$, under which the asymptotic classification error is \emph{non-universal} with respect to the distribution of $a$. Here with our analysis under LFMM, we show that Gaussian universality may breakdown even for data vectors of concentrated quadratic forms as described in \eqref{eq:concentration quadratic x}, a condition not satisfied by elliptical data vectors due to the presence of a random scaling factor $a$.

\section{Problem Setup}
\label{sec: setup}

For a set of $n$ training samples  \(\{(\x_i,y_i)\}_{i=1}^n\) with feature vectors \(\x_i\in\RR^p\) and binary labels \(y_i \in \{ \pm 1\}\), a classifier is trained by minimizing the following ridge-regularized empirical risk:
\begin{equation}\label{eq:opt-origin-reg}
   \textstyle \hat \bbeta_{\ell,\lambda}=\argmin_{\bbeta \in \RR^p} \frac1n \sum_{i=1}^n \ell(\x_i^\T \bbeta, y_i)+\frac{\lambda}{2}\Vert\bbeta\Vert^2,
\end{equation}
for some non-negative loss function $\ell \colon \RR \times \{ \pm 1\} \to \RR_+ $ that evaluates the difference between the classification score $\hat y_i = \bbeta^\T \x_i$ and the corresponding ground-truth label $y_i$. 
Data instances $\x$ with negative scores $\bbeta^\T\x$ will be assigned to the class of label $y=-1$, and those with positive scores to the class annotated by $y=1$.
The addition of the $l_2$ regularization term with  $\lambda > 0$ can improve the generalization through a better bias-variance trade-off, and also ensures the uniqueness of the solution $\hat \bbeta_{\ell,\lambda}$ in the over-parametrized regime where the feature dimension $p$ is greater than the sample size $n$.

In this paper, we consider convex and continuously differentiable loss functions.
\begin{Assumption}[Loss function]\label{ass:loss}
The function $\ell(\cdot,y) \colon \RR \to \RR_+$ in \eqref{eq:opt-origin-reg} is convex and continuously differentiable with its (first) derivative different from $0$ at the origin. 
Its second and third derivatives exist and are bounded, except on a finite set of points. 
\end{Assumption}

\Cref{ass:loss} holds for the logistic loss $\ell(\hat y, y) = -\ln(1/(1+e^{-y\hat y}))$ used in logistic regression, the square loss $\ell(\hat y, y)  = (y-\hat y)^2/2$ for least-squares classifier, and the square hinge loss $\ell(\hat y, y) = \max\{0,1-y\hat y\}^2$.
Non-smooth losses such as the hinge loss $\ell(\hat y, y) = \max\{0,1-y\hat y\}$ used in SVMs~\citep{scholkopf2018kernel}, and the absolute loss $\ell(\hat y, y) = \vert \hat y-y\vert$, fail to meet \Cref{ass:loss}.\footnote{To study these non-smooth losses, a workaround would be to evaluate instead a series of smooth functions that gradually approach the non-smooth functions, so as to retrieve their performance in some carefully taken limit.
Such consideration is however beyond the focus of this paper. }

In the following, we focus on the ERM in \eqref{eq:opt-origin-reg}, and use the shorthand notation $\hat \bbeta$ for $\hat \bbeta_{\ell,\lambda}$ in \eqref{eq:opt-origin-reg} unless there is a risk of confusion.
We consider the following linear factor mixture model.

\begin{Definition}[Linear factor mixture model, LFMM]\label{def:linear_factor}
We say that a data instance \((\x,y ) \sim \mathcal{D}_{(\x, y)}\) with class label \(y\in\{\pm 1\}\) and class priors $\Pr(y=-1) = \rho$, $\Pr(y=1) = 1-\rho$, follows a linear factor mixture model if the corresponding feature vector \(\x\in\RR^p\) can be expressed as a linear mapping of $p$ independent factors $z_1,\ldots,z_p$ as
\begin{equation}\label{eq:mixture}  
   \textstyle \x = \sum_{k=1}^p z_k \v_k=\sum_{k=1}^p (ys_k+e_k) \v_k,
\end{equation}
for linearly independent deterministic vectors $\v_1,\ldots,\v_p\in\RR^p$ and standardized independent\footnote{In other words, $\E[e_k]=0$, $\var[e_k]=1$, $\forall  k\in\{1,\ldots,p\}$.} noises $e_1,\ldots,e_p\in\RR$ of symmetric distribution.
Among the $p$ factors $z_1,\ldots,z_p$, we have
\begin{itemize}
    \item $q$ \textbf{\emph{informative factors}} $z_1,\ldots,z_q$ with deterministic signals $s_k>0$, $\forall k\in\{1,\ldots,q\}$; and
    \item $p-q$ \textbf{\emph{noise factors}} $z_{q+1}, \ldots, z_p$ with $s_k=0$, $\forall k\in\{q+1,\ldots,p\}$.
\end{itemize}
Note that \eqref{eq:mixture} can be compactly written as $\x = \V\z$, with $\V = [\v_1,\ldots,\v_p]\in\RR^{p\times p}$ and $\z=[z_1,\ldots,z_p]^\T=[y s_1 + e_1, \ldots, y s_q + e_q, e_{q+1}, \ldots, e_p]^\T \in \RR^{p}$.
The class-conditional means and covariances of $\x$ are therefore given by
\begin{align}
    \bmu &\equiv \E[\x\vert y=1]  =\textstyle \sum_{k=1}^{q} s_k \v_k \in \RR^p, \quad \E[\x\vert y=-1] = -\bmu, \label{eq:def_mu}  \\
    \bSigma &\equiv \cov[\x\vert y= \pm 1]= \V \V^\T = \textstyle \sum_{k=1}^{p}  \v_k\v_k^\T \in \RR^{p \times p}. \label{eq:def_Sigma}
\end{align}
\end{Definition}

Notice that GMM of form $\x\sim\mathcal{N}(y\bmu,\bSigma)$ is a special case of LFMM in \Cref{def:linear_factor} with exclusively Gaussian noises $e_1,\ldots,e_p$. 
See also \Cref{def:equivalent GMM} below for the associated equivalent GMM.

Linear factor models are among the most fundamental probabilistic models with latent variables, which underlie many ML methods such as PCA and ICA, and serve as building blocks of deep generative models \cite[Chapter~13]{Goodfellow-et-al-2016}. 
They are often expressed as:
\begin{align*}
    \x = \W\h+\b+{\rm noise},
\end{align*}
where $\h$ is a vector of latent variables, $\b$ a constant bias, and ${\rm noise}$ stands for an uninformative term of independent Gaussian noises. 
The LFMM in \Cref{def:linear_factor} can be related to this form minus the bias $\b$. 
Our framework requires the clusters to have opposite means (therefore $\b=\zeros_p$), which can be satisfied through a centering operation on the original data space.

Our analysis applies under the following assumption on the distribution of LFMM.

\begin{Assumption}[Distribution of LFMM]\label{ass:LFMM} We consider, for the LFMM in~\Cref{def:linear_factor}, that (i) the factors $z_1,\ldots,z_p$ have bounded fourth moments and (ii) the signal subspace ${\rm Span}\{\v_1,\ldots,\v_q\}$ is orthogonal to the noise subspace ${\rm Span}\{\v_{q+1},\ldots,\v_p\}$.
\end{Assumption}

The condition of bounded fourth moment for $z_1,\ldots,z_p$ in Item~(i) of \Cref{ass:LFMM} is required for some concentration results in our high-dimensional asymptotic analysis 
and Item~(ii) separates the informative signal subspace from the noise subspace (in which no classifier can achieve better performance than random guesses).

We position ourselves under the following high-dimensional asymptotic setting, where the feature dimension $p$ and sample size $n$ are both large and comparable.

\begin{Assumption}[High-dimensional regime]\label{ass:growth-rate}
As $n\to\infty$ with fixed $n/p \in (0,\infty)$, we have, for the LFMM in~\Cref{def:linear_factor} that (i) $\| \bmu \|, \| \bSigma \|, \| \bSigma^{-1} \|  = \Theta(1)$ and (ii) $s_1,\ldots,s_q =\Theta(1)$ with fixed $q$.
\end{Assumption}

In plain words, \Cref{ass:growth-rate} says that the ratio $n/p$, or the number of samples per dimension, remains finite in high dimensions.
Item~(i) of \Cref{ass:growth-rate} ensures, by bounding $\| \bmu \|$ and $\| \bmu \|^{-1}$, that the distance between the LFMM class centers is comparable to $1$. 
It also guarantees, by controlling $\| \bSigma \|$ and $\| \bSigma^{-1} \|$, that the variation of feature vector $\x$ on any direction in $\RR^p$ is also comparable to $1$. 
This implies that the feature vector $\x$ does not live in a subspace of dimension smaller than $p$.
The fixed number $q$ of informative factors in Item~(ii) of \Cref{ass:growth-rate} is a consequence of $\| \bmu \| = \| \sum_{k=1}^qs_k\v_k \|=\Theta(1)$.

\section{High-Dimensional Asymptotic Performance under LFMM}
\label{sec: results}

In this section, we present a self-consistent system of  equations that gives access to the high-dimensional training and generalization performances of the ERM classifier in \eqref{eq:opt-origin-reg}, under the LFMM in \Cref{def:linear_factor}. The characterization of high-dimensional asymptotic performance via a system of equations is reminiscent of previous analyses under GMM~\citep{mai2019high,mignacco2020role,pesce2023gaussian}, but our equations are different due to the collapse of the conditional one-dimensional CLT in~\eqref{eq:clt for CGEP} required for applying the CGEP.

Before presenting our system of equations, let us introduce first some mathematical objects involved in these equations. 
With the proximal operator $\prox_{\tau, f} (t)=\argmin_{a \in \RR} \left[ f(a) + \frac{1}{2\tau} (a - t)^2 \right]$ for $\tau > 0$ and convex $f\colon \RR \to \RR$, we define the mapping 
\begin{equation}
\label{eq:h}
    h_\kappa(t,y) = (\prox_{\kappa, \ell(\cdot,y)} (t)-t)/\kappa,
\end{equation}
for some constant $\kappa>0$. 
Let $r\in\RR$ be a random variable of form
\begin{equation}\label{eq:def_r}
   r = \textstyle ym +\sigma \tilde e+\sum_{k=1}^q \psi_k e_k,
\end{equation}
for constants $m,\sigma,\psi_1,\ldots,\psi_q$, with label $y$ and $e_1,\ldots,e_q$ the corresponding noise variables in the informative factors $z_1,\ldots,z_q$ of the LFMM in \Cref{def:linear_factor}, as well as $\tilde e\sim\mathcal{N}(0,1)$ independent of $y,z_1,\ldots,z_q$. 
Remark that the distribution of $r$ is parameterized by $m,\sigma^2,\psi_1,\ldots,\psi_q$.

\paragraph{Self-consistent system of equations.}
Our system of equations is on the $q+3$ deterministic constants $\theta,\eta,\gamma,\omega_1,\ldots,\omega_q$ that fully characterize the asymptotic performance of ERM classifier trained on high-dimensional LFMM\footnote{According to \Cref{ass:loss}, $\frac{\partial h(r,y)}{ \partial r}$ exists except on a finite set of points. On those points, we use the left derivative of $h(r,y)$ with respect $r$, i.e., $\lim_{t\to r_-}\left(h(t,y)-h(r,y)\right)/(t-r)$.}:
\begin{align}
   \label{eq:theta eta gamma}
   &\theta=-\E \left[ \frac{\partial h_\kappa(r,y)}{\partial r} \right],\quad\eta=\E[yh_\kappa(r,y)],\quad\gamma=\sqrt{ \E[h_\kappa^2(r,y)]},\nonumber\\
   &\omega_k=\E[h_\kappa(r,y) e_k]+ \theta \cdot \v_k^\T\Q\bxi,\quad \forall k\in\{1,\ldots,q\},
\end{align} 
where 
\begin{equation}
\label{eq:bxi Q}
    \textstyle \bxi =\eta\bmu +\sum_{k=1}^{q} \omega_k\v_k,\quad \Q=\left( \lambda\I_p+\theta\bSigma\right)^{-1},
\end{equation}
the mapping $h_\kappa(r,y)$ is as defined in \eqref{eq:h} for
\begin{equation}
\label{eq:kappa}
    \textstyle \kappa= \frac1n \tr \bSigma \Q,
\end{equation}
and the random variable $r$ as defined in \eqref{eq:def_r} with
\begin{align}
\label{eq: m sigma}
   m = \bmu^\T\Q\bxi,\quad \textstyle  \sigma^2=\frac{\gamma^2}n \tr \left(\Q\bSigma \right)^2,\quad\psi_k=\v_k^\T\Q\bxi,\quad \forall k\in\{1,\ldots,q\}.
\end{align}

We are now ready to present our \Cref{theo:main} on the asymptotic distributions of in-sample and out-of-sample predicted scores. 
The proof of \Cref{theo:main} is provided in \Cref{sm:proof-of-main-results}.

\begin{Theorem}[Asymptotic distribution of predicted scores]
\label{theo:main}
Let Assumptions~\ref{ass:loss},~\ref{ass:LFMM},~and \ref{ass:growth-rate} hold, for $\hat\bbeta$ solution to the ERM problem in \eqref{eq:opt-origin-reg} on a training set $\{ (\x_i, y_i )\}_{i=1}^n$ of size $n$ drawn i.i.d.\@ $ (\x_i, y_i )\sim\mathcal{D}_{(\x,y)}$ from the LFMM in \Cref{def:linear_factor}, we have that, for any bounded Lipschitz function $f\colon \RR\to\RR$,
\begin{equation}\label{eq:convergence test score} 
\E\big[f(\hat{\bbeta}^\T\bnu)\big]-\E\big[f(\tilde{\bbeta}^\T \bnu)\big]\to 0,
\end{equation}
for any deterministic feature vector $\bnu\in\RR^p$, and 
\begin{equation}\label{eq:convergence training score}
\E[f(\hat{\bbeta}^\T \x_i)]-\E[f(\prox_{\kappa, \ell(\cdot,y_i)} (\tilde{\bbeta}^\T \x_i))] \to 0,\quad \forall i\in\{1,\ldots,n\},
\end{equation}
where 
\begin{equation}
   \label{eq:tilde beta}
   \tilde \bbeta = \textstyle \left( \lambda\I_p+\theta\bSigma\right)^{-1}\left(\eta\bmu +\sum_{k=1}^{q} \omega_k\v_k +\gamma\bSigma^{\frac{1}{2}}\u\right), 
\end{equation}
for Gaussian vector $\u\sim\mathcal{N}(\zeros_p,\I_p/n)$ independent of $\{ (\x_i, y_i )\}_{i=1}^n$ and constants $\theta, \eta, \gamma,\omega_1,\ldots,\omega_q$ determined by the self-consistent system of equations in \eqref{eq:theta eta gamma}, with $\kappa$ given in \eqref{eq:kappa}.

\end{Theorem}

According to \eqref{eq:convergence test score} in \Cref{theo:main}, for a fresh test sample $(\x',y')$ (which might be drawn from a distribution \emph{different} from $\mathcal{D}_{(\x,y)}$ of training samples, as in the case of transfer learning), the out-of-sample predicted scores $\hat\bbeta^\T\x',\tilde\bbeta^\T\x'$ produced by the ERM classifier $\hat\bbeta$ and its high-dimensional ``equivalent'' $\tilde\bbeta$ given in \eqref{eq:tilde beta}
have asymptotically the same distribution in the sense of \eqref{eq:convergence test score}. 
Furthermore, \eqref{eq:convergence training score} tells us that the in-sample predicted score $\hat\bbeta^\T\x_i$ of $\hat\bbeta$ on a training sample $(\x_i,y_i)$ follows asymptotically the same distribution as $\prox_{\kappa, \ell(\cdot,y_i)} (\tilde{\bbeta}^\T \x_i)$. 
Since the distribution of $\tilde\bbeta$ is given in \eqref{eq:tilde beta}, we obtain directly from \Cref{theo:main} the asymptotic training and generalization errors of the ERM classifier $\hat\bbeta$.

Furthermore, it follows from LLN and CLT that $(\tilde\bbeta^\T\x,y)$ with $(\x,y)\sim\mathcal{D}_{(\x,y)}$ independent of $\tilde\bbeta$ converges in distribution to  $(r,y)$ with $r$ as defined in \eqref{eq:def_r} with $m,\sigma^2,\psi_1,\ldots,\psi_q$ given in \eqref{eq: m sigma}. 
We thus obtain the following corollary on the asymptotic classification accuracy of $\hat \bbeta$ on any training sample $(\x_i,y_i)$ and test sample $(\x',y')$ drawn from the same distribution $\mathcal{D}_{(\x,y)}$. 
The proof of \Cref{cor:performance} is deferred to \Cref{subsubsec:proof_of_cor_performance}.

\begin{Corollary}[Asymptotic generalization and training performances]
\label{cor:performance}
Under the conditions and notations of \Cref{theo:main}, we have that, for any bounded Lipschitz function $f\colon \RR\to\RR$,
\begin{equation}\label{eq:convergence tildebeta x} \E\big[f(\tilde\bbeta^\T\x)\vert y\big]-\E\big[f(r)\vert y\big]\to 0,
\end{equation}
for $(\x,y)\sim\mathcal{D}_{(\x,y)}$ independent of $\tilde\bbeta$, where $r$ is as defined in \eqref{eq:def_r} with $m,\sigma^2,\psi_1,\ldots,\psi_q$ given in \eqref{eq: m sigma}. Consequently, we have 
\begin{equation}
\label{eq:generalization error}
\Pr (y'\hat\bbeta^\T\x'>0) - \Pr (yr>0) \to 0,
\end{equation}
for some test sample $(\x',y')\sim\mathcal{D}_{(\x,y)}$ independent of $\{ (\x_i, y_i )\}_{i=1}^n$, and
\begin{equation}
\label{eq:training error}
 \Pr(y_i\hat\bbeta^\T\x_i >0) - \Pr(y\prox_{\kappa, \ell(\cdot,y)}(r)>0 )  \to 0,\quad \forall i\in\{1,\ldots,n\}.
\end{equation}
\end{Corollary}

\begin{Remark}[On classifier bias under GMM and LFMM]\normalfont
Taking expectation on both sides of \eqref{eq:tilde beta}, we get $\E[\tilde{\bbeta}]=\left( \lambda\I_p+\theta\bSigma\right)^{-1}\left(\eta\bmu +\sum_{k=1}^{q} \omega_k\v_k\right)$. 
It follows from \eqref{eq:theta eta gamma} and Stein's lemma~\citep{ingersoll1987theory} that $\omega_1,\ldots,\omega_q=0$ in the case of Gaussian informative factors $z_1,\ldots,z_q$. As GMM is a special case of LFMM with exclusively Gaussian  (informative and noise) factors, we have that $\tilde{\bbeta}$ aligns, in expectation, with $\left( \lambda\I_p+\theta\bSigma\right)^{-1}\bmu$ under GMM. 
For non-Gaussian informative factors, we generally have non-zero $\omega_1,\ldots,\omega_q$ that account for the non-Gaussian variation in data, making $\hat\bbeta$ more or less aligned with the directions $\v_1,\ldots,\v_q$ of the informative factors.
\end{Remark}

\section{Conditions and Implications of Gaussian Universality}
\label{sec: universality}

In this section, we exploit our high-dimensional asymptotic analysis in \Cref{sec: results} to derive the conditions of Gaussian universality under LFMM. 
To discuss the Gaussian universality in classification of mixture data, we introduce the notion of equivalent Gaussian mixture model (to a given LFMM), in a similar spirit to~\citet{dandi2024universality}.

\begin{Definition}[Equivalent Gaussian mixture model]
\label{def:equivalent GMM}
For an LFMM $\mathcal{D}_{(\x,y)}$ as in  \Cref{def:linear_factor}, we define its equivalent Gaussian mixture model (GMM) $\mathcal{D}_{(\g,y)}$ as the GMM with the same class-conditional means and covariances as the LFMM $\mathcal{D}_{(\x,y)}$. 
Namely,
\begin{equation}
    \g\sim\mathcal{N}(y\bmu,\bSigma),
\end{equation}
for $\bmu, \bSigma$ given in \eqref{eq:def_mu}~and~\eqref{eq:def_Sigma} of \Cref{def:linear_factor}, respectively.
We denote by $\hat\bbeta^\g$ the ERM solution to \eqref{eq:opt-origin-reg} obtained on $n$ i.i.d.\@ equivalent GMM samples $(\g_1,y_1),\ldots,(\g_n,y_n)\sim \mathcal{D}_{(\g,y)}$, and similarly its high-dimensional ``equivalent'' $\tilde\bbeta^\g$ as in \eqref{eq:tilde beta} of \Cref{theo:main}.
\end{Definition}

Notice importantly from \Cref{def:linear_factor} that the equivalent GMM $\mathcal{D}_{(\g,y)}$ to an LFMM $\mathcal{D}_{(\x,y)}$ can be obtained by taking $e_1,\ldots,e_p$ of the LFMM $\mathcal{D}_{(\x,y)}$ to be standard Gaussian variables. 

We define two types of Gaussian universality considered in this paper as follows.

\begin{Definition}[Gaussian universality under LFMM]
\label{def:gaussian_universality}
For an ERM solution $\hat\bbeta$ obtained on a general LFMM $\mathcal{D}_{(\x,y)}$ in \Cref{def:linear_factor} and an ERM solution $\hat\bbeta^\g$ obtained on the equivalent GMM in \Cref{def:equivalent GMM}, we say that the Gaussian universality holds
\begin{itemize}
    \item on \textbf{\emph{classifier}} if $\hat\bbeta$ has asymptotically the same predictive ability as $\hat\bbeta^\g$ on a given test set, as a consequence of their high-dimensional equivalents $\tilde\bbeta_{\ell,\lambda}, \tilde\bbeta_{\ell,\lambda}^{\g}$ provided in \Cref{theo:main} following the \emph{same} distribution; 
    \item on \textbf{\emph{in-distribution performance}} if the respective training and generalization performances under $\mathcal{D}_{(\x,y)}$ are asymptotically the same as under $\mathcal{D}_{(\g,y)}$, that is
\begin{equation}
\Pr(y_i\x_i^\T\hat\bbeta >0) - \Pr(y_i\g_i^\T\hat\bbeta^\g>0 )  \to 0, 
\end{equation}
 and
\begin{equation}
\Pr (y'\x^{\prime\T}\hat\bbeta>0) - \Pr (y'\g^{\prime\T}\hat\bbeta^\g>0) \to 0,
\end{equation}
for $(\x',y')\sim\mathcal{D}_{(\x,y)}$ a test sample independent of $\{ (\x_i, y_i )\}_{i=1}^n$, and $(\g',y')\sim\mathcal{D}_{(\g,y)}$ independent of $\{(\g_i, y_i )\}_{i=1}^n$.
\end{itemize}
\end{Definition}

In the following, we study first in~\Cref{sec:universality error} the Gaussian universality in the sense of in-distribution performance, and discuss our results with respect to the conditional one-directional CLT and the CGEP in~\citet{dandi2024universality}. 
We then reveal in \Cref{sec:universality beta} the key role of square loss in inducing the Gaussian universality of classifier, and discuss its implication for the choice of loss function.

Throughout this section, our discussions are illustrated through numerical experiments on datasets of moderately large size, with $n,p$ \emph{only in hundreds}. 
A close match is consistently observed between the proposed asymptotic analysis and the empirical results.

\subsection{Gaussian Universality of In-distribution Performance}
\label{sec:universality error}

Notice from \eqref{eq:generalization error} in \Cref{cor:performance} that the in-distribution generalization performance of $\hat\bbeta$ under an LFMM~$\mathcal{D}_{(\x,y)}$ is determined by the random variable $r$ in \eqref{eq:def_r}, the distribution of which depends solely on (i) the distributions of $y,e_1,\ldots,e_q$ in the LFMM and 
(ii) the values of $m,\sigma^2,\psi_1,\ldots,\psi_q$ given in \eqref{eq: m sigma}. 
Remark also that the values of $m,\sigma^2,\psi_1,\ldots,\psi_q$ in \eqref{eq: m sigma} are determined by the system of equations in \eqref{eq:theta eta gamma}, which concerns only the distributions of $r, y, e_1,\ldots,e_q$, as well as the deterministic parameters $\bmu,\bSigma, \v_1,\ldots,\v_q$ of the LFMM.

We thus conclude that the distribution of $r$ is \emph{insensitive} to the distributions of noise factors $z_{q+1},\ldots,z_p$. 
In other words, an LFMM with Gaussian noises $e_1,\ldots,e_q$ in its informative factors $z_1,\ldots,z_q$ has the \emph{same} asymptotic generalization performance as its equivalent GMM in \Cref{def:equivalent GMM}, \emph{regardless of} the distributions of the noise factors $z_{q+1},\ldots,z_p$.

A similar conclusion can be drawn from \eqref{eq:training error} of \Cref{cor:performance} on the asymptotic in-distribution training performance, by studying also the distribution of $r$ but through a proximal mapping $\prox_{\kappa, \ell(\cdot,y)}$.
We formalize these conclusions on the universality of in-distribution performance in \Cref{cor:condition for r}, the proof of which is given in \Cref{subsubsec:proof_of_condition for r}.

\begin{Corollary}[Condition of Gaussian universality on in-distribution performance]\label{cor:condition for r}
Under the settings and notations of \Cref{theo:main}~and~\Cref{def:equivalent GMM}, the 
Gaussian universality of in-distribution performance in \Cref{def:gaussian_universality} holds if and only if noises $e_1,\ldots,e_q$ of LFMM informative factors in \eqref{eq:mixture}  are Gaussian. 
\end{Corollary}

\Cref{fig: universality r} provides numerical illustrations of \Cref{cor:condition for r}, where we compare the empirical histograms and the asymptotic distributions of the out-of-sample predicted scores $\hat \bbeta^\T \x'$ for data drawn from three different LFMMs: 
an LFMM satisfying the in-distribution performance universality condition in \Cref{cor:condition for r}  (\textbf{left}), 
an LFMM sharing the same parameters ($\bmu,\bSigma,\rho$) with the first but violating the condition in \Cref{cor:condition for r} (\textbf{right}),
and their equivalent GMM in the sense of \Cref{def:equivalent GMM} (\textbf{middle}).

\begin{figure}
   \centering
   \begin{tabular}{ccc}
   \quad\quad\quad Universality & GMM & Breakdown\\
    \begin{tikzpicture}[font=\footnotesize]
   \renewcommand{\axisdefaulttryminticks}{4} 
   \tikzstyle{every major grid}+=[style=densely dashed ] \tikzstyle{every axis y label}+=[yshift=-10pt] 
   \tikzstyle{every axis x label}+=[yshift=5pt]
   \tikzstyle{every axis legend}+=[cells={anchor=west},fill=white,
   at={(1,1)}, anchor=north east, font=\tiny, legend style={mark options={scale=0.1}}]
   \begin{axis}[
    bar width=1pt,
   width=.35\linewidth,
   height=.3\linewidth,
   ymax=0.016,
  ymin=0,
    xmin=-200,
   xmax=200,
   tick pos=left,
   grid=major, 
   ymajorgrids=false,
   scaled ticks=true,
   xticklabels={},
   ylabel={Distribution of $\hat\bbeta^\T\x'$},
   ]
   
    \addplot[smooth,color=RED!100!white,line width=1pt] plot coordinates{
   (-142.1500,0.000000)(-136.4500,0.000000)(-130.7500,0.000000)(-125.0500,0.000000)(-119.3500,0.000001)(-113.6500,0.000002)(-107.9500,0.000005)(-102.2500,0.000013)(-96.5500,0.000032)(-90.8500,0.000072)(-85.1500,0.000156)(-79.4500,0.000319)(-73.7500,0.000597)(-68.0500,0.001067)(-62.3500,0.001758)(-56.6500,0.002730)(-50.9500,0.003946)(-45.2500,0.005366)(-39.5500,0.006860)(-33.8500,0.008256)(-28.1500,0.009413)(-22.4500,0.010194)(-16.7500,0.010609)(-11.0500,0.010780)(-5.3500,0.010805)(0.3500,0.010794)(6.0500,0.010804)(11.7500,0.010777)(17.4500,0.010592)(23.1500,0.010108)(28.8500,0.009278)(34.5500,0.008095)(40.2500,0.006686)(45.9500,0.005196)(51.6500,0.003785)(57.3500,0.002576)(63.0500,0.001661)(68.7500,0.000996)(74.4500,0.000558)(80.1500,0.000295)(85.8500,0.000145)(91.5500,0.000066)(97.2500,0.000028)(102.9500,0.000011)(108.6500,0.000004)(114.3500,0.000002)(120.0500,0.000001)(125.7500,0.000000)(131.4500,0.000000)(137.1500,0.000000)
      
   };
   \addplot+[ybar,mark=none,draw=white,fill=BLUE,area legend,opacity=0.5] coordinates{
    (-132.3000,0.000000)(-126.9000,0.000000)(-121.5000,0.000000)(-116.1000,0.000000)(-110.7000,0.000002)(-105.3000,0.000004)(-99.9000,0.000015)(-94.5000,0.000027)(-89.1000,0.000069)(-83.7000,0.000152)(-78.3000,0.000308)(-72.9000,0.000543)(-67.5000,0.001007)(-62.1000,0.001645)(-56.7000,0.002550)(-51.3000,0.003696)(-45.9000,0.005066)(-40.5000,0.006505)(-35.1000,0.008023)(-29.7000,0.009298)(-24.3000,0.010161)(-18.9000,0.010687)(-13.5000,0.010895)(-8.1000,0.011011)(-2.7000,0.010998)(2.7000,0.011043)(8.1000,0.011008)(13.5000,0.010754)(18.9000,0.010698)(24.3000,0.010149)(29.7000,0.009217)(35.1000,0.008040)(40.5000,0.006566)(45.9000,0.005064)(51.3000,0.003735)(56.7000,0.002522)(62.1000,0.001635)(67.5000,0.000996)(72.9000,0.000538)(78.3000,0.000295)(83.7000,0.000151)(89.1000,0.000065)(94.5000,0.000029)(99.9000,0.000011)(105.3000,0.000004)(110.7000,0.000001)(116.1000,0.000000)(121.5000,0.000000)(126.9000,0.000000)(132.3000,0.000000)
   };
   \end{axis}
   \end{tikzpicture} &\begin{tikzpicture}[font=\footnotesize]
   \renewcommand{\axisdefaulttryminticks}{4} 
   \tikzstyle{every major grid}+=[style=densely dashed ] \tikzstyle{every axis y label}+=[yshift=-10pt] 
   \tikzstyle{every axis x label}+=[yshift=5pt]
   \tikzstyle{every axis legend}+=[cells={anchor=west},fill=white,
   at={(1,0.2)}, anchor=north east, font=\small]
   \begin{axis}[
    bar width=1pt,
   width=.35\linewidth,
   height=.3\linewidth,
   ymax=0.016,
  ymin=0,
    xmin=-200,
   xmax=200,
   tick pos=left,
   grid=major, 
   ymajorgrids=false,
   scaled ticks=true,
   xticklabels={},,
   yticklabels={},
   ylabel={},
   ]
    \addplot[smooth,color=RED!100!white,line width=1pt] plot coordinates{
   (-142.1500,0.000000)(-136.4500,0.000000)(-130.7500,0.000000)(-125.0500,0.000000)(-119.3500,0.000001)(-113.6500,0.000002)(-107.9500,0.000005)(-102.2500,0.000013)(-96.5500,0.000032)(-90.8500,0.000072)(-85.1500,0.000156)(-79.4500,0.000319)(-73.7500,0.000597)(-68.0500,0.001067)(-62.3500,0.001758)(-56.6500,0.002730)(-50.9500,0.003946)(-45.2500,0.005366)(-39.5500,0.006860)(-33.8500,0.008256)(-28.1500,0.009413)(-22.4500,0.010194)(-16.7500,0.010609)(-11.0500,0.010780)(-5.3500,0.010805)(0.3500,0.010794)(6.0500,0.010804)(11.7500,0.010777)(17.4500,0.010592)(23.1500,0.010108)(28.8500,0.009278)(34.5500,0.008095)(40.2500,0.006686)(45.9500,0.005196)(51.6500,0.003785)(57.3500,0.002576)(63.0500,0.001661)(68.7500,0.000996)(74.4500,0.000558)(80.1500,0.000295)(85.8500,0.000145)(91.5500,0.000066)(97.2500,0.000028)(102.9500,0.000011)(108.6500,0.000004)(114.3500,0.000002)(120.0500,0.000001)(125.7500,0.000000)(131.4500,0.000000)(137.1500,0.000000)
      
   };
   \addplot+[ybar,mark=none,draw=white,fill=BLUE,area legend,opacity=0.5] coordinates{
    (-122.4000,0.000000)(-117.2000,0.000000)(-112.0000,0.000001)(-106.8000,0.000003)(-101.6000,0.000005)(-96.4000,0.000015)(-91.2000,0.000036)(-86.0000,0.000084)(-80.8000,0.000174)(-75.6000,0.000330)(-70.4000,0.000615)(-65.2000,0.001084)(-60.0000,0.001709)(-54.8000,0.002680)(-49.6000,0.003838)(-44.4000,0.005222)(-39.2000,0.006765)(-34.0000,0.008216)(-28.8000,0.009402)(-23.6000,0.010435)(-18.4000,0.010949)(-13.2000,0.011352)(-8.0000,0.011396)(-2.8000,0.011362)(2.4000,0.011401)(7.6000,0.011333)(12.8000,0.011350)(18.0000,0.010999)(23.2000,0.010435)(28.4000,0.009501)(33.6000,0.008289)(38.8000,0.006964)(44.0000,0.005367)(49.2000,0.003956)(54.4000,0.002777)(59.6000,0.001815)(64.8000,0.001101)(70.0000,0.000658)(75.2000,0.000346)(80.4000,0.000186)(85.6000,0.000084)(90.8000,0.000042)(96.0000,0.000021)(101.2000,0.000008)(106.4000,0.000002)(111.6000,0.000000)(116.8000,0.000000)(122.0000,0.000000)(127.2000,0.000000)(132.4000,0.000000)
   };

   \end{axis}
   \end{tikzpicture} &\begin{tikzpicture}[font=\footnotesize]
   \renewcommand{\axisdefaulttryminticks}{4} 
   \tikzstyle{every major grid}+=[style=densely dashed ] \tikzstyle{every axis y label}+=[yshift=-10pt] 
   \tikzstyle{every axis x label}+=[yshift=5pt]
    \tikzstyle{every axis legend}+=[cells={anchor=west},fill=white,
   at={(1,1)}, anchor=north east, font=\scriptsize, legend style={mark options={scale=0.1}}]
   \begin{axis}[
    bar width=1pt,
   width=.35\linewidth,
   height=.3\linewidth,
   ymax=0.016,
  ymin=0,
    xmin=-200,
   xmax=200,
   tick pos=left,
   grid=major, 
   ymajorgrids=false,
   scaled ticks=true,
   xticklabels={},
   ylabel={},
   yticklabels={},
   ]
   \addplot[smooth,color=RED!100!white,line width=1pt] plot coordinates{
    (-109.7700,0.000000)(-105.3100,0.000000)(-100.8500,0.000000)(-96.3900,0.000001)(-91.9300,0.000006)(-87.4700,0.000017)(-83.0100,0.000051)(-78.5500,0.000130)(-74.0900,0.000295)(-69.6300,0.000599)(-65.1700,0.001116)(-60.7100,0.001885)(-56.2500,0.002882)(-51.7900,0.004072)(-47.3300,0.005310)(-42.8700,0.006455)(-38.4100,0.007444)(-33.9500,0.008208)(-29.4900,0.008795)(-25.0300,0.009274)(-20.5700,0.009811)(-16.1100,0.010338)(-11.6500,0.010911)(-7.1900,0.011445)(-2.7300,0.011766)(1.7300,0.011775)(6.1900,0.011523)(10.6500,0.011046)(15.1100,0.010483)(19.5700,0.009908)(24.0300,0.009391)(28.4900,0.008895)(32.9500,0.008346)(37.4100,0.007633)(41.8700,0.006704)(46.3300,0.005573)(50.7900,0.004349)(55.2500,0.003141)(59.7100,0.002083)(64.1700,0.001258)(68.6300,0.000695)(73.0900,0.000344)(77.5500,0.000158)(82.0100,0.000064)(86.4700,0.000024)(90.9300,0.000008)(95.3900,0.000002)(99.8500,0.000001)(104.3100,0.000000)(108.7700,0.000000)
      
   };
    \addlegendentry{{Theo}}
   \addplot[ybar,mark=none,draw=white,fill=BLUE,area legend,opacity=0.5] coordinates{
   (-109.8950,0.000000)(-105.6850,0.000000)(-101.4750,0.000000)(-97.2650,0.000000)(-93.0550,0.000000)(-88.8450,0.000002)(-84.6350,0.000010)(-80.4250,0.000031)(-76.2150,0.000082)(-72.0050,0.000218)(-67.7950,0.000468)(-63.5850,0.000948)(-59.3750,0.001713)(-55.1650,0.002668)(-50.9550,0.003951)(-46.7450,0.005248)(-42.5350,0.006488)(-38.3250,0.007652)(-34.1150,0.008326)(-29.9050,0.008963)(-25.6950,0.009437)(-21.4850,0.009850)(-17.2750,0.010483)(-13.0650,0.011044)(-8.8550,0.011718)(-4.6450,0.012176)(-0.4350,0.012373)(3.7750,0.012241)(7.9850,0.011759)(12.1950,0.011134)(16.4050,0.010555)(20.6150,0.009980)(24.8250,0.009517)(29.0350,0.009008)(33.2450,0.008513)(37.4550,0.007751)(41.6650,0.006692)(45.8750,0.005505)(50.0850,0.004209)(54.2950,0.002909)(58.5050,0.001886)(62.7150,0.001054)(66.9250,0.000538)(71.1350,0.000267)(75.3450,0.000106)(79.5550,0.000039)(83.7650,0.000014)(87.9750,0.000004)(92.1850,0.000001)(96.3950,0.000000)
   };
   \addlegendentry{{Emp}}
   \end{axis}
   \end{tikzpicture}
\end{tabular}
   \caption{Theoretical and empirical distribution of predicted scores $\hat\bbeta^\T\x'$ for some fresh $(\x',y')\sim\mathcal{D}_{(\x,y)}$ independent of $\hat\bbeta$. 
  The theoretical probability densities (\textbf{\RED red}) are obtained from \Cref{theo:main}, 
  and the empirical histograms (\textbf{\BLUE blue}) are the values of $\hat\bbeta^\T\x'$ over $10^6$ independent copies of $\x'$, for three different LFMMs as in \Cref{def:linear_factor} with $n=600$, $p=200$, $\rho=0.5$, $s=[\sqrt{2};\mathbf{0}_{p-1}]$ (so that $q = 1$), and Haar distributed $\V$. 
  \textbf{Left}: normal $e_1$ and uniformly distributed $e_2,\ldots,e_p$; \textbf{Middle}: normal $e_1,\ldots,e_p$; 
  \textbf{Right}: uniformly distributed $e_1$, and normal $e_2,\ldots,e_p$.}
   \label{fig: universality r}
\end{figure}

\begin{Remark}[Connection to conditional one-directional CLT in \eqref{eq:clt for CGEP}]\normalfont
Our universality results on the in-distribution performance in \Cref{cor:condition for r} are related to the CGEP proven by \citet{dandi2024universality} under the presumed validity of a conditional one-directional CLT in \eqref{eq:clt for CGEP}. Under our notations, the conditional one-directional CLT in \eqref{eq:clt for CGEP} translates to the convergence of $y'\hat\bbeta^\T\x'$ and $y'\hat\bbeta_\g^\T\x'$ to the same normal distribution, i.e.,
\begin{equation}
\frac{y'\x^{\prime\T}\hat\bbeta-m}{\sqrt{\sigma^2+\sum_{k=1}^{q}\psi_k^2}}\overset{\rm d}{\to} \mathcal{N}\left(0,1\right), 
\quad \frac{y'\g^{\prime\T}\hat\bbeta^\g-m}{\sqrt{\sigma^2+\sum_{k=1}^{q}\psi_k^2}}\overset{\rm d}{\to} \textstyle \mathcal{N}\left(0,1\right),
\end{equation}
as it can be shown from \eqref{eq:convergence test score},\eqref{eq:convergence tildebeta x} that $y'\hat\bbeta^\T\x'$ has asymptotically the same distribution as $yr$, which is of mean $m$ and variance $\sigma^2+\sum_{k=1}^{q}\psi_k^2$. It is easy to see from \eqref{eq:def_r} that $yr$ is normally distributed if and only if $e_1,\ldots,e_q$ are Gaussian, which is exactly the condition of universality stated in \Cref{cor:condition for r}.
\end{Remark}

\subsection{Gaussian Universality of Classifier and Implication for Choice of Loss}
\label{sec:universality beta}

As discussed in \Cref{sec:universality error}, the system of equations in \eqref{eq:theta eta gamma} does not depend on the distributions of noise factors $z_{q+1},\ldots,z_{p}$. 
As the distribution of the high-dimensional equivalent $\tilde\bbeta$ to $\hat\bbeta$ given in \eqref{eq:tilde beta} is controlled by the constants $\theta,\eta,\gamma,\omega_1,\ldots,\omega_q$ that are determined by \eqref{eq:theta eta gamma}, it is therefore also \emph{universal} over the distributions of $z_{q+1},\ldots,z_{p}$. Then, by a similar reasoning to \Cref{cor:condition for r}, we conclude that the Gaussian universality of classifier in \Cref{def:gaussian_universality} holds in the case of normally distributed $e_1,\ldots,e_q$.

This is however \emph{not} the only case of Gaussian universality on classifier. 
Note from \eqref{eq:def_r} and \eqref{eq:theta eta gamma} that, even though the system of equations in \eqref{eq:theta eta gamma} \emph{does} depend on the distributions of $e_1,\ldots,e_q$, it only involves their means and variances if $h_\kappa(r,y)$ is linear in $r$. 
Remark also from \eqref{eq:h} that $h_\kappa(r,y)$ varies linearly with $r$ if and only if the square loss $\ell(\hat y,y)=(\hat y -y)^2/2$ (or its rescaled version) is used.

These two conditions for the Gaussian universality of classifier as understood in \Cref{def:gaussian_universality} are made precise in the following result, proven in \Cref{subsubsec:proof_of_condition for beta}.

\begin{Corollary}[Condition of Gaussian universality on classifier]\label{cor:condition for beta}
Under the settings and notations of \Cref{theo:main} and \Cref{def:equivalent GMM}, the Gaussian universality of classifier in \Cref{def:gaussian_universality} holds if and only if one of the following two conditions is met: (i) $e_1,\ldots,e_q$ in \eqref{eq:mixture} are normally distributed; (ii) $\partial\ell(\hat y, y)/\partial\hat y$ is a linear function of $\hat y$, e.g., $\ell(\hat y, y)=(\hat y -y)^2/2$.
\end{Corollary}

\begin{Remark}[Limitation of square loss]\label{rem:square loss}
\normalfont
As an important consequence of \Cref{cor:condition for beta}, any classifier $\hat\bbeta$ trained using the square loss on generic LFMM samples $\{(\x_i, y_i)\}_{i=1}^n\sim\mathcal{D}_{(\x,y)}$ and $\hat\bbeta^\g$ trained on equivalent GMM samples $\{(\g_i, y_i)\}_{i=1}^n\sim\mathcal{D}_{(\g,y)}$ have asymptotically the same probability of correctly classifying a fresh LFMM test sample $(\x',y')\sim \mathcal{D}_{(\x,y)}$.
That is, ERM classifiers trained with square loss are \emph{unable} to adapt to non-Gaussian informative factors of LFMM, contrarily to other (non-square) losses.
\end{Remark}

The particular effect of square loss discussed in \Cref{rem:square loss} is numerically demonstrated in \Cref{fig: universality beta}.
On the left hand side, the square loss $\ell_{\rm sqr}(\hat y, y)=(\hat y -y)^2/2$ is used, and $h_\kappa(r,y)$ varies linearly with $r$ as in the top left plot (the two elongated scatter plots are associated respectively with $y=\pm 1$), so that the distribution of $\x^{\prime\T} \hat\bbeta_{\ell_{\rm sqr},\lambda}$ and  $\x^{\prime\T}\hat\bbeta_{\ell_{\rm sqr},\lambda}^{\g}$ are \emph{indistinguishable} in the bottom left plot of \Cref{fig: universality beta};
On the right hand, the square hinge loss $\ell_{\rm shg}(\hat y,y)=\max\{0,1-\hat y y\}^2$ is used, and we observe drastically different behaviors for $\x^{\prime\T} \hat\bbeta_{\ell_{\rm shg},\lambda}$ and $\x^{\prime\T}\hat\bbeta_{\ell_{\rm shg},\lambda}^{\g}$  in the right column of \Cref{fig: universality beta}, when the points $\big[r,h_\kappa(r,\pm 1)\big]$ are highly nonlinear.

\begin{figure}
   \centering
   \begin{tabular}{cc}
      Universality&Breakdown\\
\begin{tikzpicture}[font=\footnotesize]
      \renewcommand{\axisdefaulttryminticks}{4}   
      \pgfplotsset{every axis legend/.append style={cells={anchor=west},fill=white, at={(1,1)}, anchor=north east, font=\tiny}}
      \begin{axis}[
      height=0.3\linewidth,
      width=0.5\linewidth,
      grid=major,
      ymajorgrids=false,
      scaled ticks=true,
        xmin=-1.5,
        xmax=1.5,
        ymin=-.8,
        ymax=.8,
      xtick={-1,0,1},
        ytick={-0.5,0,0.5},
        yticklabels={-0.5,0,0.5},
      ]
      \addplot[mark = *, only marks, mark size=1.5pt,color=cyan,line width=0.2pt] plot coordinates{
      (-0.114,0.4616)(0.220,0.3229)(-0.098,0.4550)(0.098,0.3737)(-0.185,0.4910)(0.201,0.3309)(0.809,0.0793)(-0.056,0.4375)(0.718,0.1167)(0.697,0.1257)(0.721,0.1157)(0.544,0.1891)(0.101,0.3722)(-0.075,0.4453)(0.559,0.1825)(0.842,0.0656)(0.825,0.0727)(0.494,0.2097)(-0.133,0.4694)(0.204,0.3298)(0.225,0.3210)(0.364,0.2634)(0.076,0.3829)(-0.392,0.5765)(1.085,-0.0352)(0.017,0.4073)(-0.350,0.5592)(0.518,0.1997)(0.553,0.1853)(0.448,0.2286)(0.562,0.1814)(0.084,0.3796)(0.251,0.3101)(-0.032,0.4276)(-0.075,0.4454)(-0.338,0.5542)(0.320,0.2817)(1.008,-0.0034)(0.753,0.1023)(0.726,0.1133)(0.157,0.3492)(0.668,0.1375)(0.468,0.2203)(0.269,0.3028)(0.669,0.1372)(0.726,0.1134)(0.466,0.2213)(-0.108,0.4589)(0.559,0.1828)(1.261,-0.1083)(0.394,0.2509)(0.420,0.2401)(0.645,0.1472)(0.783,0.0898)(0.530,0.1949)(0.268,0.3032)(0.741,0.1075)(0.120,0.3645)(0.957,0.0178)(0.040,0.3976)(0.465,0.2214)(0.594,0.1680)(0.769,0.0956)(0.521,0.1985)(-0.146,0.4746)(0.293,0.2930)(0.179,0.3400)(0.118,0.3655)(-0.102,0.4566)(0.389,0.2530)(0.151,0.3516)(-0.269,0.5258)(0.367,0.2623)(0.059,0.3898)(0.308,0.2867)(0.190,0.3354)(0.290,0.2942)(0.121,0.3639)(0.345,0.2713)(0.680,0.1324)(0.398,0.2493)(0.086,0.3784)(1.158,-0.0653)(-0.507,0.6241)(0.408,0.2454)(0.885,0.0474)(-0.149,0.4760)(0.042,0.3969)(0.976,0.0101)(0.119,0.3650)(0.504,0.2054)(0.542,0.1899)(0.434,0.2346)(0.253,0.3094)(-0.078,0.4465)(0.397,0.2496)(0.527,0.1958)(-0.317,0.5457)(-0.167,0.4832)(-0.163,0.4818)
      };
      \addplot[mark=o, only marks, mark size=1.5pt,color=orange,line width=0.2pt] plot coordinates{
     (-0.229,-0.3193)(-0.155,-0.3501)(-0.759,-0.0996)(-0.387,-0.2539)(-0.623,-0.1560)(-0.484,-0.2136)(-0.211,-0.3267)(-0.922,-0.0322)(-0.065,-0.3875)(-0.432,-0.2354)(-0.881,-0.0491)(-0.108,-0.3694)(-0.148,-0.3528)(-0.677,-0.1340)(-0.168,-0.3445)(-0.864,-0.0561)(-0.078,-0.3818)(-0.831,-0.0701)(-0.591,-0.1692)(0.236,-0.5118)(0.079,-0.4468)(0.217,-0.5042)(0.046,-0.4333)(0.019,-0.4221)(0.099,-0.4554)(-0.270,-0.3025)(0.302,-0.5393)(-0.252,-0.3097)(-1.151,0.0626)(-0.609,-0.1618)(-1.115,0.0478)(-0.038,-0.3987)(-0.087,-0.3783)(-0.317,-0.2829)(-0.320,-0.2816)(-0.008,-0.4109)(-0.729,-0.1124)(-0.142,-0.3552)(-0.218,-0.3238)(-0.674,-0.1352)(-0.857,-0.0590)(-0.146,-0.3539)(-0.156,-0.3496)(-0.744,-0.1060)(-0.153,-0.3510)(0.114,-0.4613)(-0.928,-0.0298)(0.469,-0.6085)(-0.388,-0.2534)(-0.015,-0.4081)(-1.118,0.0491)(-0.287,-0.2954)(-1.094,0.0388)(0.114,-0.4615)(-0.190,-0.3354)(-0.807,-0.0800)(-0.222,-0.3222)(0.485,-0.6152)(-0.732,-0.1111)(-0.306,-0.2876)(-0.444,-0.2305)(-0.329,-0.2779)(-1.033,0.0138)(0.152,-0.4770)(0.247,-0.5165)(-0.418,-0.2410)(-0.329,-0.2779)(-0.393,-0.2515)(-0.385,-0.2549)(-0.132,-0.3596)(0.189,-0.4924)(-0.435,-0.2341)(0.124,-0.4656)(0.482,-0.6137)(-0.918,-0.0340)(0.325,-0.5489)(-0.262,-0.3056)(-0.795,-0.0849)(0.042,-0.4314)(-0.366,-0.2625)(0.115,-0.4618)(-0.726,-0.1134)(-0.639,-0.1495)(0.166,-0.4831)(-0.946,-0.0225)(0.297,-0.5374)(-0.104,-0.3710)(-0.212,-0.3265)(0.362,-0.5641)(-0.193,-0.3343)(-0.644,-0.1476)(-1.011,0.0048)(-0.344,-0.2717)(-0.608,-0.1622)(-0.557,-0.1837)(0.570,-0.6505)(-0.724,-0.1141)(-0.345,-0.2715)(0.279,-0.5299)(-0.091,-0.3765)
        };
      \end{axis}
      \end{tikzpicture}&   
        \begin{tikzpicture}[font=\footnotesize]
      \renewcommand{\axisdefaulttryminticks}{4}   
      \pgfplotsset{every axis legend/.append style={cells={anchor=west},fill=white, at={(1,1)}, anchor=north east, font=\footnotesize}}
      \begin{axis}[
      height=0.3\linewidth,
      width=0.5\linewidth,
      grid=major,
      ymajorgrids=false,
      scaled ticks=true,
        xmin=-6.5,
        xmax=6.5,
        ymin=-0.4,
        ymax=0.4,
      xtick={-4,0,4},
        ytick={-0.3,0,0.3},
        yticklabels={-0.3,0,0.3},
      ]
      \addplot[mark = *, only marks, mark size=1.5pt,color=cyan,line width=0.2pt] plot coordinates{
      (3.380,0.0000)(3.723,0.0000)(1.251,0.0000)(0.882,0.0195)(0.874,0.0209)(0.746,0.0421)(3.845,0.0000)(3.126,0.0000)(4.701,0.0000)(3.762,0.0000)(0.236,0.1267)(3.896,0.0000)(2.947,0.0000)(3.584,0.0000)(4.283,0.0000)(2.170,0.0000)(4.351,0.0000)(3.040,0.0000)(3.230,0.0000)(3.001,0.0000)(0.007,0.1647)(3.258,0.0000)(0.284,0.1189)(4.032,0.0000)(0.034,0.1603)(0.683,0.0526)(3.027,0.0000)(0.564,0.0724)(2.638,0.0000)(-0.556,0.2583)(4.500,0.0000)(3.503,0.0000)(3.132,0.0000)(0.482,0.0859)(1.317,0.0000)(-0.559,0.2587)(4.371,0.0000)(2.029,0.0000)(1.687,0.0000)(1.071,0.0000)(0.413,0.0973)(1.157,0.0000)(1.585,0.0000)(3.545,0.0000)(3.765,0.0000)(1.819,0.0000)(-0.145,0.1900)(4.189,0.0000)(0.575,0.0706)(4.352,0.0000)(3.603,0.0000)(3.809,0.0000)(4.136,0.0000)(2.966,0.0000)(-0.303,0.2162)(3.980,0.0000)(1.656,0.0000)(1.032,0.0000)(0.730,0.0448)(3.136,0.0000)(0.698,0.0502)(5.110,0.0000)(4.477,0.0000)(0.823,0.0294)(0.100,0.1494)(1.340,0.0000)(0.151,0.1408)(2.238,0.0000)(0.881,0.0197)(0.896,0.0172)(-0.619,0.2687)(0.252,0.1241)(3.784,0.0000)(-0.229,0.2039)(4.503,0.0000)(1.620,0.0000)(3.257,0.0000)(0.707,0.0486)(2.088,0.0000)(1.000,0.0000)(1.462,0.0000)(0.148,0.1414)(1.038,0.0000)(-0.367,0.2269)(3.151,0.0000)(-0.877,0.3115)(0.789,0.0349)(3.698,0.0000)(0.337,0.1100)(0.728,0.0451)(0.785,0.0356)(0.956,0.0074)(0.637,0.0603)(-0.566,0.2598)(-0.019,0.1690)(2.777,0.0000)(3.972,0.0000)(0.753,0.0410)(0.825,0.0291)(0.288,0.1182)
      };
      \addlegendentry{$\big[r,h_\kappa(r,1)\big]$}

      \addplot[mark=o, only marks, mark size=1.5pt,color=orange,line width=0.2pt] plot coordinates{
       (-4.053,0.0000)(-4.382,0.0000)(-0.750,-0.0415)(-3.185,0.0000)(-4.924,0.0000)(-3.209,0.0000)(-5.056,0.0000)(-5.665,0.0000)(-4.662,0.0000)(-1.696,0.0000)(-3.150,0.0000)(-4.402,0.0000)(-5.700,0.0000)(-4.295,0.0000)(-4.579,0.0000)(-0.945,-0.0092)(-3.200,0.0000)(-4.441,0.0000)(0.335,-0.2215)(-5.315,0.0000)(-4.503,0.0000)(-1.802,0.0000)(-1.786,0.0000)(-0.295,-0.1169)(-0.402,-0.0992)(-4.318,0.0000)(-0.671,-0.0546)(-0.369,-0.1047)(-1.771,0.0000)(-4.122,0.0000)(-3.924,0.0000)(-3.709,0.0000)(-3.722,0.0000)(0.518,-0.2519)(-0.410,-0.0979)(-3.984,0.0000)(0.861,-0.3089)(-4.745,0.0000)(-3.924,0.0000)(-0.830,-0.0282)(-3.499,0.0000)(-4.865,0.0000)(-4.845,0.0000)(-0.056,-0.1567)(0.112,-0.1846)(0.263,-0.2097)(0.095,-0.1817)(-1.158,0.0000)(0.599,-0.2654)(-1.147,0.0000)(1.007,-0.3330)(-4.713,0.0000)(-0.017,-0.1631)(-3.091,0.0000)(-0.931,-0.0114)(-4.666,0.0000)(-0.345,-0.1088)(-0.323,-0.1123)(-1.718,0.0000)(-5.045,0.0000)(-0.134,-0.1436)(-3.791,0.0000)(-4.523,0.0000)(-4.179,0.0000)(-3.312,0.0000)(-3.049,0.0000)(-3.818,0.0000)(-4.383,0.0000)(-3.358,0.0000)(-3.660,0.0000)(-3.209,0.0000)(-0.161,-0.1392)(-4.131,0.0000)(-5.328,0.0000)(-4.342,0.0000)(-3.591,0.0000)(-3.518,0.0000)(-0.764,-0.0391)(-3.949,0.0000)(-0.172,-0.1374)(-3.455,0.0000)(-0.229,-0.1279)(-3.810,0.0000)(-0.399,-0.0997)(-3.666,0.0000)(-4.217,0.0000)(-4.114,0.0000)(-0.424,-0.0957)(-0.992,-0.0014)(-4.055,0.0000)(-1.489,0.0000)(-3.762,0.0000)(-0.831,-0.0280)(0.001,-0.1660)(-4.819,0.0000)(-2.177,0.0000)(-3.927,0.0000)(-5.384,0.0000)(-4.021,0.0000)(-1.562,0.0000)
        };
      \addlegendentry{$\big[r,h_\kappa(r,-1)\big]$}
   
      \end{axis}
      \end{tikzpicture}\\  
\begin{tikzpicture}[font=\footnotesize]
      \renewcommand{\axisdefaulttryminticks}{4} 
      \pgfplotsset{every axis legend/.append style={cells={anchor=west},fill=white, at={(1,1)}, anchor=north east, font=\footnotesize}}
      \begin{axis}[
      width=.5\linewidth,
      height=.3\linewidth,
      bar width=2pt,
      grid=major,
      ymajorgrids=false,
       ymax=0.9,
      ymin=0,
      ytick={0,0.4,0.8},
      ]
        \addplot[ybar,mark=none,draw=white,fill=BLUE,area legend,opacity=0.5] coordinates{
       (-1.9590,0.0000)(-1.8770,0.0001)(-1.7950,0.0003)(-1.7130,0.0006)(-1.6310,0.0020)(-1.5490,0.0044)(-1.4670,0.0093)(-1.3850,0.0180)(-1.3030,0.0329)(-1.2210,0.0560)(-1.1390,0.0879)(-1.0570,0.1363)(-0.9750,0.1890)(-0.8930,0.2476)(-0.8110,0.3056)(-0.7290,0.3600)(-0.6470,0.4103)(-0.5650,0.4500)(-0.4830,0.4927)(-0.4010,0.5271)(-0.3190,0.5710)(-0.2370,0.6130)(-0.1550,0.6445)(-0.0730,0.6738)(0.0090,0.6785)(0.0910,0.6685)(0.1730,0.6361)(0.2550,0.5991)(0.3370,0.5610)(0.4190,0.5246)(0.5010,0.4825)(0.5830,0.4418)(0.6650,0.3980)(0.7470,0.3506)(0.8290,0.2920)(0.9110,0.2341)(0.9930,0.1767)(1.0750,0.1243)(1.1570,0.0831)(1.2390,0.0509)(1.3210,0.0291)(1.4030,0.0165)(1.4850,0.0082)(1.5670,0.0040)(1.6490,0.0018)(1.7310,0.0007)(1.8130,0.0003)(1.8950,0.0001)(1.9770,0.0000)(2.0590,0.0000)
       };
       \addplot[ybar,mark=none,draw=white,fill=yellow!100!black!100,area legend,opacity=0.6] coordinates{
         (-2.0370,0.0000)(-1.9510,0.0001)(-1.8650,0.0003)(-1.7790,0.0008)(-1.6930,0.0018)(-1.6070,0.0039)(-1.5210,0.0074)(-1.4350,0.0161)(-1.3490,0.0276)(-1.2630,0.0489)(-1.1770,0.0775)(-1.0910,0.1162)(-1.0050,0.1673)(-0.9190,0.2250)(-0.8330,0.2848)(-0.7470,0.3452)(-0.6610,0.3978)(-0.5750,0.4466)(-0.4890,0.4917)(-0.4030,0.5350)(-0.3170,0.5722)(-0.2310,0.6164)(-0.1450,0.6492)(-0.0590,0.6653)(0.0270,0.6683)(0.1130,0.6546)(0.1990,0.6248)(0.2850,0.5932)(0.3710,0.5439)(0.4570,0.5075)(0.5430,0.4619)(0.6290,0.4149)(0.7150,0.3591)(0.8010,0.3075)(0.8870,0.2462)(0.9730,0.1886)(1.0590,0.1340)(1.1450,0.0946)(1.2310,0.0584)(1.3170,0.0341)(1.4030,0.0203)(1.4890,0.0106)(1.5750,0.0052)(1.6610,0.0020)(1.7470,0.0009)(1.8330,0.0004)(1.9190,0.0001)(2.0050,0.0000)(2.0910,0.0000)(2.1770,0.0000)
   };

     \addplot[smooth,color=RED!100!white,line width=2pt] plot coordinates{
      (-2.3075,0.0000)(-2.2986,0.0000)(-2.2897,0.0000)(-2.2808,0.0000)(-2.2719,0.0000)(-2.2630,0.0000)(-2.2541,0.0000)(-2.2452,0.0000)(-2.2363,0.0000)(-2.2274,0.0000)(-2.2184,0.0000)(-2.2095,0.0000)(-2.2006,0.0000)(-2.1917,0.0000)(-2.1828,0.0000)(-2.1739,0.0000)(-2.1650,0.0000)(-2.1561,0.0000)(-2.1472,0.0000)(-2.1383,0.0000)(-2.1293,0.0000)(-2.1204,0.0000)(-2.1115,0.0000)(-2.1026,0.0000)(-2.0937,0.0000)(-2.0848,0.0000)(-2.0759,0.0000)(-2.0670,0.0000)(-2.0581,0.0000)(-2.0492,0.0000)(-2.0402,0.0000)(-2.0313,0.0000)(-2.0224,0.0000)(-2.0135,0.0000)(-2.0046,0.0000)(-1.9957,0.0000)(-1.9868,0.0000)(-1.9779,0.0000)(-1.9690,0.0000)(-1.9601,0.0000)(-1.9511,0.0000)(-1.9422,0.0000)(-1.9333,0.0000)(-1.9244,0.0000)(-1.9155,0.0000)(-1.9066,0.0000)(-1.8977,0.0000)(-1.8888,0.0000)(-1.8799,0.0000)(-1.8710,0.0000)(-1.8620,0.0000)(-1.8531,0.0000)(-1.8442,0.0000)(-1.8353,0.0000)(-1.8264,0.0000)(-1.8175,0.0001)(-1.8086,0.0001)(-1.7997,0.0001)(-1.7908,0.0001)(-1.7819,0.0001)(-1.7729,0.0001)(-1.7640,0.0001)(-1.7551,0.0001)(-1.7462,0.0002)(-1.7373,0.0002)(-1.7284,0.0002)(-1.7195,0.0002)(-1.7106,0.0003)(-1.7017,0.0003)(-1.6928,0.0003)(-1.6838,0.0004)(-1.6749,0.0004)(-1.6660,0.0004)(-1.6571,0.0005)(-1.6482,0.0006)(-1.6393,0.0006)(-1.6304,0.0007)(-1.6215,0.0008)(-1.6126,0.0008)(-1.6037,0.0009)(-1.5947,0.0011)(-1.5858,0.0012)(-1.5769,0.0013)(-1.5680,0.0014)(-1.5591,0.0016)(-1.5502,0.0018)(-1.5413,0.0020)(-1.5324,0.0022)(-1.5235,0.0024)(-1.5146,0.0026)(-1.5056,0.0029)(-1.4967,0.0032)(-1.4878,0.0036)(-1.4789,0.0038)(-1.4700,0.0042)(-1.4611,0.0046)(-1.4522,0.0051)(-1.4433,0.0054)(-1.4344,0.0060)(-1.4255,0.0066)(-1.4165,0.0072)(-1.4076,0.0078)(-1.3987,0.0085)(-1.3898,0.0092)(-1.3809,0.0100)(-1.3720,0.0108)(-1.3631,0.0118)(-1.3542,0.0127)(-1.3453,0.0137)(-1.3364,0.0149)(-1.3274,0.0159)(-1.3185,0.0175)(-1.3096,0.0186)(-1.3007,0.0199)(-1.2918,0.0216)(-1.2829,0.0232)(-1.2740,0.0251)(-1.2651,0.0268)(-1.2562,0.0287)(-1.2473,0.0307)(-1.2383,0.0328)(-1.2294,0.0350)(-1.2205,0.0375)(-1.2116,0.0399)(-1.2027,0.0425)(-1.1938,0.0453)(-1.1849,0.0481)(-1.1760,0.0511)(-1.1671,0.0545)(-1.1582,0.0577)(-1.1492,0.0609)(-1.1403,0.0645)(-1.1314,0.0686)(-1.1225,0.0723)(-1.1136,0.0765)(-1.1047,0.0805)(-1.0958,0.0852)(-1.0869,0.0894)(-1.0780,0.0945)(-1.0691,0.0996)(-1.0601,0.1042)(-1.0512,0.1093)(-1.0423,0.1144)(-1.0334,0.1200)(-1.0245,0.1255)(-1.0156,0.1316)(-1.0067,0.1373)(-0.9978,0.1435)(-0.9889,0.1492)(-0.9800,0.1563)(-0.9710,0.1623)(-0.9621,0.1689)(-0.9532,0.1751)(-0.9443,0.1827)(-0.9354,0.1887)(-0.9265,0.1956)(-0.9176,0.2032)(-0.9087,0.2110)(-0.8998,0.2180)(-0.8909,0.2249)(-0.8819,0.2327)(-0.8730,0.2396)(-0.8641,0.2472)(-0.8552,0.2544)(-0.8463,0.2619)(-0.8374,0.2696)(-0.8285,0.2778)(-0.8196,0.2851)(-0.8107,0.2922)(-0.8018,0.3003)(-0.7928,0.3073)(-0.7839,0.3147)(-0.7750,0.3221)(-0.7661,0.3302)(-0.7572,0.3379)(-0.7483,0.3434)(-0.7394,0.3507)(-0.7305,0.3577)(-0.7216,0.3646)(-0.7127,0.3729)(-0.7037,0.3792)(-0.6948,0.3855)(-0.6859,0.3918)(-0.6770,0.3985)(-0.6681,0.4050)(-0.6592,0.4114)(-0.6503,0.4171)(-0.6414,0.4228)(-0.6325,0.4290)(-0.6236,0.4343)(-0.6146,0.4399)(-0.6057,0.4464)(-0.5968,0.4516)(-0.5879,0.4564)(-0.5790,0.4620)(-0.5701,0.4671)(-0.5612,0.4719)(-0.5523,0.4770)(-0.5434,0.4815)(-0.5345,0.4874)(-0.5255,0.4923)(-0.5166,0.4963)(-0.5077,0.5011)(-0.4988,0.5050)(-0.4899,0.5097)(-0.4810,0.5145)(-0.4721,0.5184)(-0.4632,0.5224)(-0.4543,0.5269)(-0.4454,0.5319)(-0.4364,0.5361)(-0.4275,0.5402)(-0.4186,0.5448)(-0.4097,0.5496)(-0.4008,0.5532)(-0.3919,0.5571)(-0.3830,0.5627)(-0.3741,0.5662)(-0.3652,0.5706)(-0.3563,0.5743)(-0.3473,0.5802)(-0.3384,0.5838)(-0.3295,0.5881)(-0.3206,0.5911)(-0.3117,0.5962)(-0.3028,0.6023)(-0.2939,0.6053)(-0.2850,0.6104)(-0.2761,0.6143)(-0.2672,0.6185)(-0.2582,0.6232)(-0.2493,0.6266)(-0.2404,0.6315)(-0.2315,0.6362)(-0.2226,0.6402)(-0.2137,0.6431)(-0.2048,0.6477)(-0.1959,0.6514)(-0.1870,0.6568)(-0.1781,0.6596)(-0.1691,0.6636)(-0.1602,0.6663)(-0.1513,0.6710)(-0.1424,0.6725)(-0.1335,0.6755)(-0.1246,0.6797)(-0.1157,0.6837)(-0.1068,0.6850)(-0.0979,0.6873)(-0.0890,0.6893)(-0.0800,0.6928)(-0.0711,0.6943)(-0.0622,0.6947)(-0.0533,0.6982)(-0.0444,0.6989)(-0.0355,0.7001)(-0.0266,0.6999)(-0.0177,0.7003)(-0.0088,0.7011)(0.0001,0.7009)(0.0091,0.7019)(0.0180,0.7010)(0.0269,0.6999)(0.0358,0.6989)(0.0447,0.6984)(0.0536,0.6964)(0.0625,0.6962)(0.0714,0.6935)(0.0803,0.6920)(0.0892,0.6902)(0.0982,0.6875)(0.1071,0.6839)(0.1160,0.6832)(0.1249,0.6804)(0.1338,0.6773)(0.1427,0.6744)(0.1516,0.6706)(0.1605,0.6664)(0.1694,0.6636)(0.1783,0.6593)(0.1873,0.6548)(0.1962,0.6509)(0.2051,0.6482)(0.2140,0.6448)(0.2229,0.6404)(0.2318,0.6368)(0.2407,0.6318)(0.2496,0.6269)(0.2585,0.6227)(0.2674,0.6188)(0.2764,0.6146)(0.2853,0.6094)(0.2942,0.6050)(0.3031,0.6011)(0.3120,0.5971)(0.3209,0.5926)(0.3298,0.5888)(0.3387,0.5836)(0.3476,0.5791)(0.3565,0.5745)(0.3655,0.5714)(0.3744,0.5653)(0.3833,0.5624)(0.3922,0.5564)(0.4011,0.5524)(0.4100,0.5494)(0.4189,0.5453)(0.4278,0.5404)(0.4367,0.5368)(0.4456,0.5309)(0.4546,0.5278)(0.4635,0.5233)(0.4724,0.5185)(0.4813,0.5136)(0.4902,0.5099)(0.4991,0.5053)(0.5080,0.5013)(0.5169,0.4965)(0.5258,0.4906)(0.5347,0.4870)(0.5437,0.4812)(0.5526,0.4761)(0.5615,0.4710)(0.5704,0.4671)(0.5793,0.4619)(0.5882,0.4562)(0.5971,0.4505)(0.6060,0.4456)(0.6149,0.4405)(0.6238,0.4335)(0.6328,0.4288)(0.6417,0.4225)(0.6506,0.4174)(0.6595,0.4098)(0.6684,0.4049)(0.6773,0.3985)(0.6862,0.3916)(0.6951,0.3847)(0.7040,0.3789)(0.7129,0.3722)(0.7219,0.3647)(0.7308,0.3578)(0.7397,0.3504)(0.7486,0.3427)(0.7575,0.3355)(0.7664,0.3294)(0.7753,0.3223)(0.7842,0.3147)(0.7931,0.3063)(0.8020,0.2995)(0.8110,0.2922)(0.8199,0.2848)(0.8288,0.2772)(0.8377,0.2703)(0.8466,0.2627)(0.8555,0.2551)(0.8644,0.2467)(0.8733,0.2396)(0.8822,0.2320)(0.8911,0.2254)(0.9001,0.2171)(0.9090,0.2103)(0.9179,0.2033)(0.9268,0.1962)(0.9357,0.1895)(0.9446,0.1826)(0.9535,0.1754)(0.9624,0.1685)(0.9713,0.1627)(0.9802,0.1559)(0.9892,0.1492)(0.9981,0.1437)(1.0070,0.1373)(1.0159,0.1313)(1.0248,0.1252)(1.0337,0.1196)(1.0426,0.1143)(1.0515,0.1093)(1.0604,0.1038)(1.0693,0.0992)(1.0783,0.0938)(1.0872,0.0892)(1.0961,0.0848)(1.1050,0.0805)(1.1139,0.0765)(1.1228,0.0725)(1.1317,0.0680)(1.1406,0.0646)(1.1495,0.0609)(1.1584,0.0574)(1.1674,0.0544)(1.1763,0.0508)(1.1852,0.0480)(1.1941,0.0451)(1.2030,0.0426)(1.2119,0.0399)(1.2208,0.0376)(1.2297,0.0349)(1.2386,0.0330)(1.2475,0.0306)(1.2565,0.0287)(1.2654,0.0267)(1.2743,0.0248)(1.2832,0.0229)(1.2921,0.0214)(1.3010,0.0198)(1.3099,0.0185)(1.3188,0.0171)(1.3277,0.0160)(1.3366,0.0147)(1.3456,0.0137)(1.3545,0.0127)(1.3634,0.0118)(1.3723,0.0108)(1.3812,0.0100)(1.3901,0.0092)(1.3990,0.0084)(1.4079,0.0078)(1.4168,0.0072)(1.4257,0.0066)(1.4347,0.0061)(1.4436,0.0055)(1.4525,0.0050)(1.4614,0.0046)(1.4703,0.0042)(1.4792,0.0038)(1.4881,0.0035)(1.4970,0.0031)(1.5059,0.0029)(1.5148,0.0026)(1.5238,0.0024)(1.5327,0.0021)(1.5416,0.0019)(1.5505,0.0017)(1.5594,0.0016)(1.5683,0.0015)(1.5772,0.0013)(1.5861,0.0012)(1.5950,0.0010)(1.6039,0.0010)(1.6129,0.0009)(1.6218,0.0008)(1.6307,0.0007)(1.6396,0.0006)(1.6485,0.0005)(1.6574,0.0005)(1.6663,0.0004)(1.6752,0.0004)(1.6841,0.0004)(1.6930,0.0003)(1.7020,0.0003)(1.7109,0.0003)(1.7198,0.0002)(1.7287,0.0002)(1.7376,0.0002)(1.7465,0.0001)(1.7554,0.0001)(1.7643,0.0001)(1.7732,0.0001)(1.7821,0.0001)(1.7911,0.0001)(1.8000,0.0001)(1.8089,0.0001)(1.8178,0.0001)(1.8267,0.0000)(1.8356,0.0000)(1.8445,0.0001)(1.8534,0.0000)(1.8623,0.0000)(1.8712,0.0000)(1.8802,0.0000)(1.8891,0.0000)(1.8980,0.0000)(1.9069,0.0000)(1.9158,0.0000)(1.9247,0.0000)(1.9336,0.0000)(1.9425,0.0000)(1.9514,0.0000)(1.9603,0.0000)(1.9693,0.0000)(1.9782,0.0000)(1.9871,0.0000)(1.9960,0.0000)(2.0049,0.0000)(2.0138,0.0000)(2.0227,0.0000)(2.0316,0.0000)(2.0405,0.0000)(2.0494,0.0000)(2.0584,0.0000)(2.0673,0.0000)(2.0762,0.0000)(2.0851,0.0000)(2.0940,0.0000)(2.1029,0.0000)(2.1118,0.0000)(2.1207,0.0000)(2.1296,0.0000)(2.1385,0.0000)
       };
      \addplot[dotted,color=green!60!black!100,line width=3pt] plot coordinates{
      (-2.3075,0.0000)(-2.2986,0.0000)(-2.2897,0.0000)(-2.2808,0.0000)(-2.2719,0.0000)(-2.2630,0.0000)(-2.2541,0.0000)(-2.2452,0.0000)(-2.2363,0.0000)(-2.2274,0.0000)(-2.2184,0.0000)(-2.2095,0.0000)(-2.2006,0.0000)(-2.1917,0.0000)(-2.1828,0.0000)(-2.1739,0.0000)(-2.1650,0.0000)(-2.1561,0.0000)(-2.1472,0.0000)(-2.1383,0.0000)(-2.1293,0.0000)(-2.1204,0.0000)(-2.1115,0.0000)(-2.1026,0.0000)(-2.0937,0.0000)(-2.0848,0.0000)(-2.0759,0.0000)(-2.0670,0.0000)(-2.0581,0.0000)(-2.0492,0.0000)(-2.0402,0.0000)(-2.0313,0.0000)(-2.0224,0.0000)(-2.0135,0.0000)(-2.0046,0.0000)(-1.9957,0.0000)(-1.9868,0.0000)(-1.9779,0.0000)(-1.9690,0.0000)(-1.9601,0.0000)(-1.9511,0.0000)(-1.9422,0.0000)(-1.9333,0.0000)(-1.9244,0.0000)(-1.9155,0.0000)(-1.9066,0.0000)(-1.8977,0.0000)(-1.8888,0.0000)(-1.8799,0.0000)(-1.8710,0.0000)(-1.8620,0.0000)(-1.8531,0.0000)(-1.8442,0.0000)(-1.8353,0.0000)(-1.8264,0.0000)(-1.8175,0.0001)(-1.8086,0.0001)(-1.7997,0.0001)(-1.7908,0.0001)(-1.7819,0.0001)(-1.7729,0.0001)(-1.7640,0.0001)(-1.7551,0.0001)(-1.7462,0.0002)(-1.7373,0.0002)(-1.7284,0.0002)(-1.7195,0.0002)(-1.7106,0.0003)(-1.7017,0.0003)(-1.6928,0.0003)(-1.6838,0.0004)(-1.6749,0.0004)(-1.6660,0.0004)(-1.6571,0.0005)(-1.6482,0.0006)(-1.6393,0.0006)(-1.6304,0.0007)(-1.6215,0.0008)(-1.6126,0.0008)(-1.6037,0.0009)(-1.5947,0.0011)(-1.5858,0.0012)(-1.5769,0.0013)(-1.5680,0.0014)(-1.5591,0.0016)(-1.5502,0.0018)(-1.5413,0.0020)(-1.5324,0.0022)(-1.5235,0.0024)(-1.5146,0.0026)(-1.5056,0.0029)(-1.4967,0.0032)(-1.4878,0.0036)(-1.4789,0.0038)(-1.4700,0.0042)(-1.4611,0.0046)(-1.4522,0.0051)(-1.4433,0.0054)(-1.4344,0.0060)(-1.4255,0.0066)(-1.4165,0.0072)(-1.4076,0.0078)(-1.3987,0.0085)(-1.3898,0.0092)(-1.3809,0.0100)(-1.3720,0.0108)(-1.3631,0.0118)(-1.3542,0.0127)(-1.3453,0.0137)(-1.3364,0.0149)(-1.3274,0.0159)(-1.3185,0.0175)(-1.3096,0.0186)(-1.3007,0.0199)(-1.2918,0.0216)(-1.2829,0.0232)(-1.2740,0.0251)(-1.2651,0.0268)(-1.2562,0.0287)(-1.2473,0.0307)(-1.2383,0.0328)(-1.2294,0.0350)(-1.2205,0.0375)(-1.2116,0.0399)(-1.2027,0.0425)(-1.1938,0.0453)(-1.1849,0.0481)(-1.1760,0.0511)(-1.1671,0.0545)(-1.1582,0.0577)(-1.1492,0.0609)(-1.1403,0.0645)(-1.1314,0.0686)(-1.1225,0.0723)(-1.1136,0.0765)(-1.1047,0.0805)(-1.0958,0.0852)(-1.0869,0.0894)(-1.0780,0.0945)(-1.0691,0.0996)(-1.0601,0.1042)(-1.0512,0.1093)(-1.0423,0.1144)(-1.0334,0.1200)(-1.0245,0.1255)(-1.0156,0.1316)(-1.0067,0.1373)(-0.9978,0.1435)(-0.9889,0.1492)(-0.9800,0.1563)(-0.9710,0.1623)(-0.9621,0.1689)(-0.9532,0.1751)(-0.9443,0.1827)(-0.9354,0.1887)(-0.9265,0.1956)(-0.9176,0.2032)(-0.9087,0.2110)(-0.8998,0.2180)(-0.8909,0.2249)(-0.8819,0.2327)(-0.8730,0.2396)(-0.8641,0.2472)(-0.8552,0.2544)(-0.8463,0.2619)(-0.8374,0.2696)(-0.8285,0.2778)(-0.8196,0.2851)(-0.8107,0.2922)(-0.8018,0.3003)(-0.7928,0.3073)(-0.7839,0.3147)(-0.7750,0.3221)(-0.7661,0.3302)(-0.7572,0.3379)(-0.7483,0.3434)(-0.7394,0.3507)(-0.7305,0.3577)(-0.7216,0.3646)(-0.7127,0.3729)(-0.7037,0.3792)(-0.6948,0.3855)(-0.6859,0.3918)(-0.6770,0.3985)(-0.6681,0.4050)(-0.6592,0.4114)(-0.6503,0.4171)(-0.6414,0.4228)(-0.6325,0.4290)(-0.6236,0.4343)(-0.6146,0.4399)(-0.6057,0.4464)(-0.5968,0.4516)(-0.5879,0.4564)(-0.5790,0.4620)(-0.5701,0.4671)(-0.5612,0.4719)(-0.5523,0.4770)(-0.5434,0.4815)(-0.5345,0.4874)(-0.5255,0.4923)(-0.5166,0.4963)(-0.5077,0.5011)(-0.4988,0.5050)(-0.4899,0.5097)(-0.4810,0.5145)(-0.4721,0.5184)(-0.4632,0.5224)(-0.4543,0.5269)(-0.4454,0.5319)(-0.4364,0.5361)(-0.4275,0.5402)(-0.4186,0.5448)(-0.4097,0.5496)(-0.4008,0.5532)(-0.3919,0.5571)(-0.3830,0.5627)(-0.3741,0.5662)(-0.3652,0.5706)(-0.3563,0.5743)(-0.3473,0.5802)(-0.3384,0.5838)(-0.3295,0.5881)(-0.3206,0.5911)(-0.3117,0.5962)(-0.3028,0.6023)(-0.2939,0.6053)(-0.2850,0.6104)(-0.2761,0.6143)(-0.2672,0.6185)(-0.2582,0.6232)(-0.2493,0.6266)(-0.2404,0.6315)(-0.2315,0.6362)(-0.2226,0.6402)(-0.2137,0.6431)(-0.2048,0.6477)(-0.1959,0.6514)(-0.1870,0.6568)(-0.1781,0.6596)(-0.1691,0.6636)(-0.1602,0.6663)(-0.1513,0.6710)(-0.1424,0.6725)(-0.1335,0.6755)(-0.1246,0.6797)(-0.1157,0.6837)(-0.1068,0.6850)(-0.0979,0.6873)(-0.0890,0.6893)(-0.0800,0.6928)(-0.0711,0.6943)(-0.0622,0.6947)(-0.0533,0.6982)(-0.0444,0.6989)(-0.0355,0.7001)(-0.0266,0.6999)(-0.0177,0.7003)(-0.0088,0.7011)(0.0001,0.7009)(0.0091,0.7019)(0.0180,0.7010)(0.0269,0.6999)(0.0358,0.6989)(0.0447,0.6984)(0.0536,0.6964)(0.0625,0.6962)(0.0714,0.6935)(0.0803,0.6920)(0.0892,0.6902)(0.0982,0.6875)(0.1071,0.6839)(0.1160,0.6832)(0.1249,0.6804)(0.1338,0.6773)(0.1427,0.6744)(0.1516,0.6706)(0.1605,0.6664)(0.1694,0.6636)(0.1783,0.6593)(0.1873,0.6548)(0.1962,0.6509)(0.2051,0.6482)(0.2140,0.6448)(0.2229,0.6404)(0.2318,0.6368)(0.2407,0.6318)(0.2496,0.6269)(0.2585,0.6227)(0.2674,0.6188)(0.2764,0.6146)(0.2853,0.6094)(0.2942,0.6050)(0.3031,0.6011)(0.3120,0.5971)(0.3209,0.5926)(0.3298,0.5888)(0.3387,0.5836)(0.3476,0.5791)(0.3565,0.5745)(0.3655,0.5714)(0.3744,0.5653)(0.3833,0.5624)(0.3922,0.5564)(0.4011,0.5524)(0.4100,0.5494)(0.4189,0.5453)(0.4278,0.5404)(0.4367,0.5368)(0.4456,0.5309)(0.4546,0.5278)(0.4635,0.5233)(0.4724,0.5185)(0.4813,0.5136)(0.4902,0.5099)(0.4991,0.5053)(0.5080,0.5013)(0.5169,0.4965)(0.5258,0.4906)(0.5347,0.4870)(0.5437,0.4812)(0.5526,0.4761)(0.5615,0.4710)(0.5704,0.4671)(0.5793,0.4619)(0.5882,0.4562)(0.5971,0.4505)(0.6060,0.4456)(0.6149,0.4405)(0.6238,0.4335)(0.6328,0.4288)(0.6417,0.4225)(0.6506,0.4174)(0.6595,0.4098)(0.6684,0.4049)(0.6773,0.3985)(0.6862,0.3916)(0.6951,0.3847)(0.7040,0.3789)(0.7129,0.3722)(0.7219,0.3647)(0.7308,0.3578)(0.7397,0.3504)(0.7486,0.3427)(0.7575,0.3355)(0.7664,0.3294)(0.7753,0.3223)(0.7842,0.3147)(0.7931,0.3063)(0.8020,0.2995)(0.8110,0.2922)(0.8199,0.2848)(0.8288,0.2772)(0.8377,0.2703)(0.8466,0.2627)(0.8555,0.2551)(0.8644,0.2467)(0.8733,0.2396)(0.8822,0.2320)(0.8911,0.2254)(0.9001,0.2171)(0.9090,0.2103)(0.9179,0.2033)(0.9268,0.1962)(0.9357,0.1895)(0.9446,0.1826)(0.9535,0.1754)(0.9624,0.1685)(0.9713,0.1627)(0.9802,0.1559)(0.9892,0.1492)(0.9981,0.1437)(1.0070,0.1373)(1.0159,0.1313)(1.0248,0.1252)(1.0337,0.1196)(1.0426,0.1143)(1.0515,0.1093)(1.0604,0.1038)(1.0693,0.0992)(1.0783,0.0938)(1.0872,0.0892)(1.0961,0.0848)(1.1050,0.0805)(1.1139,0.0765)(1.1228,0.0725)(1.1317,0.0680)(1.1406,0.0646)(1.1495,0.0609)(1.1584,0.0574)(1.1674,0.0544)(1.1763,0.0508)(1.1852,0.0480)(1.1941,0.0451)(1.2030,0.0426)(1.2119,0.0399)(1.2208,0.0376)(1.2297,0.0349)(1.2386,0.0330)(1.2475,0.0306)(1.2565,0.0287)(1.2654,0.0267)(1.2743,0.0248)(1.2832,0.0229)(1.2921,0.0214)(1.3010,0.0198)(1.3099,0.0185)(1.3188,0.0171)(1.3277,0.0160)(1.3366,0.0147)(1.3456,0.0137)(1.3545,0.0127)(1.3634,0.0118)(1.3723,0.0108)(1.3812,0.0100)(1.3901,0.0092)(1.3990,0.0084)(1.4079,0.0078)(1.4168,0.0072)(1.4257,0.0066)(1.4347,0.0061)(1.4436,0.0055)(1.4525,0.0050)(1.4614,0.0046)(1.4703,0.0042)(1.4792,0.0038)(1.4881,0.0035)(1.4970,0.0031)(1.5059,0.0029)(1.5148,0.0026)(1.5238,0.0024)(1.5327,0.0021)(1.5416,0.0019)(1.5505,0.0017)(1.5594,0.0016)(1.5683,0.0015)(1.5772,0.0013)(1.5861,0.0012)(1.5950,0.0010)(1.6039,0.0010)(1.6129,0.0009)(1.6218,0.0008)(1.6307,0.0007)(1.6396,0.0006)(1.6485,0.0005)(1.6574,0.0005)(1.6663,0.0004)(1.6752,0.0004)(1.6841,0.0004)(1.6930,0.0003)(1.7020,0.0003)(1.7109,0.0003)(1.7198,0.0002)(1.7287,0.0002)(1.7376,0.0002)(1.7465,0.0001)(1.7554,0.0001)(1.7643,0.0001)(1.7732,0.0001)(1.7821,0.0001)(1.7911,0.0001)(1.8000,0.0001)(1.8089,0.0001)(1.8178,0.0001)(1.8267,0.0000)(1.8356,0.0000)(1.8445,0.0001)(1.8534,0.0000)(1.8623,0.0000)(1.8712,0.0000)(1.8802,0.0000)(1.8891,0.0000)(1.8980,0.0000)(1.9069,0.0000)(1.9158,0.0000)(1.9247,0.0000)(1.9336,0.0000)(1.9425,0.0000)(1.9514,0.0000)(1.9603,0.0000)(1.9693,0.0000)(1.9782,0.0000)(1.9871,0.0000)(1.9960,0.0000)(2.0049,0.0000)(2.0138,0.0000)(2.0227,0.0000)(2.0316,0.0000)(2.0405,0.0000)(2.0494,0.0000)(2.0584,0.0000)(2.0673,0.0000)(2.0762,0.0000)(2.0851,0.0000)(2.0940,0.0000)(2.1029,0.0000)(2.1118,0.0000)(2.1207,0.0000)(2.1296,0.0000)(2.1385,0.0000)
       };
      \end{axis}
      \end{tikzpicture}&
      \begin{tikzpicture}[font=\footnotesize]
      \renewcommand{\axisdefaulttryminticks}{4} 
        \tikzstyle{every axis legend}+=[cells={anchor=west},fill=white,
   at={(1,1)}, anchor=north east, font=\tiny, legend style={mark options={scale=0.1}}]
      \begin{axis}[
      width=.5\linewidth,
      height=.3\linewidth,
      bar width=2pt,
      grid=major,
      ymajorgrids=false,
      xmax=12,
        xmin=-12,
        ymax=0.22,
        ymin=0,
      ytick={0,0.1,0.2},
      xlabel={},
      ]
    
       \addplot[ybar,mark=none,draw=white,fill=BLUE,area legend,opacity=0.5] coordinates{
       (-7.4490,0.0000)(-7.1470,0.0000)(-6.8450,0.0001)(-6.5430,0.0003)(-6.2410,0.0012)(-5.9390,0.0037)(-5.6370,0.0100)(-5.3350,0.0224)(-5.0330,0.0436)(-4.7310,0.0704)(-4.4290,0.1025)(-4.1270,0.1240)(-3.8250,0.1294)(-3.5230,0.1151)(-3.2210,0.0897)(-2.9190,0.0598)(-2.6170,0.0379)(-2.3150,0.0298)(-2.0130,0.0344)(-1.7110,0.0561)(-1.4090,0.0845)(-1.1070,0.1183)(-0.8050,0.1480)(-0.5030,0.1676)(-0.2010,0.1757)(0.1010,0.1780)(0.4030,0.1710)(0.7050,0.1556)(1.0070,0.1297)(1.3090,0.0971)(1.6110,0.0644)(1.9130,0.0400)(2.2150,0.0291)(2.5170,0.0336)(2.8190,0.0513)(3.1210,0.0800)(3.4230,0.1075)(3.7250,0.1267)(4.0270,0.1286)(4.3290,0.1107)(4.6310,0.0813)(4.9330,0.0528)(5.2350,0.0284)(5.5370,0.0131)(5.8390,0.0053)(6.1410,0.0018)(6.4430,0.0005)(6.7450,0.0001)(7.0470,0.0000)(7.3490,0.0000)
   };

       \addplot[ybar,mark=none,draw=white,fill=yellow,area legend,opacity=0.6] coordinates{
       (-9.7080,0.0000)(-9.3240,0.0000)(-8.9400,0.0000)(-8.5560,0.0001)(-8.1720,0.0002)(-7.7880,0.0004)(-7.4040,0.0012)(-7.0200,0.0023)(-6.6360,0.0049)(-6.2520,0.0091)(-5.8680,0.0160)(-5.4840,0.0258)(-5.1000,0.0373)(-4.7160,0.0507)(-4.3320,0.0636)(-3.9480,0.0747)(-3.5640,0.0819)(-3.1800,0.0861)(-2.7960,0.0867)(-2.4120,0.0894)(-2.0280,0.0925)(-1.6440,0.1003)(-1.2600,0.1116)(-0.8760,0.1244)(-0.4920,0.1351)(-0.1080,0.1396)(0.2760,0.1372)(0.6600,0.1303)(1.0440,0.1197)(1.4280,0.1070)(1.8120,0.0964)(2.1960,0.0896)(2.5800,0.0880)(2.9640,0.0870)(3.3480,0.0845)(3.7320,0.0785)(4.1160,0.0706)(4.5000,0.0577)(4.8840,0.0448)(5.2680,0.0319)(5.6520,0.0206)(6.0360,0.0126)(6.4200,0.0070)(6.8040,0.0037)(7.1880,0.0017)(7.5720,0.0008)(7.9560,0.0002)(8.3400,0.0001)(8.7240,0.0000)(9.1080,0.0000)
   };

     \addplot[smooth,color=RED!100!white,line width=2pt] plot coordinates{
      (-7.8442,0.0000)(-7.8126,0.0000)(-7.7810,0.0000)(-7.7494,0.0000)(-7.7178,0.0000)(-7.6862,0.0000)(-7.6546,0.0000)(-7.6230,0.0000)(-7.5914,0.0000)(-7.5598,0.0000)(-7.5282,0.0000)(-7.4966,0.0000)(-7.4650,0.0000)(-7.4334,0.0000)(-7.4018,0.0000)(-7.3702,0.0000)(-7.3386,0.0000)(-7.3070,0.0000)(-7.2754,0.0000)(-7.2438,0.0000)(-7.2122,0.0000)(-7.1806,0.0000)(-7.1490,0.0000)(-7.1174,0.0000)(-7.0858,0.0000)(-7.0542,0.0000)(-7.0226,0.0000)(-6.9910,0.0000)(-6.9594,0.0000)(-6.9278,0.0000)(-6.8962,0.0000)(-6.8646,0.0000)(-6.8330,0.0000)(-6.8014,0.0000)(-6.7698,0.0001)(-6.7382,0.0001)(-6.7066,0.0001)(-6.6750,0.0001)(-6.6434,0.0001)(-6.6118,0.0001)(-6.5802,0.0002)(-6.5486,0.0002)(-6.5170,0.0002)(-6.4854,0.0002)(-6.4538,0.0003)(-6.4222,0.0003)(-6.3906,0.0004)(-6.3590,0.0005)(-6.3274,0.0005)(-6.2958,0.0006)(-6.2642,0.0007)(-6.2326,0.0008)(-6.2010,0.0009)(-6.1694,0.0011)(-6.1378,0.0013)(-6.1062,0.0014)(-6.0746,0.0016)(-6.0430,0.0019)(-6.0114,0.0021)(-5.9798,0.0024)(-5.9482,0.0027)(-5.9166,0.0031)(-5.8850,0.0035)(-5.8534,0.0039)(-5.8218,0.0044)(-5.7902,0.0049)(-5.7586,0.0055)(-5.7270,0.0061)(-5.6954,0.0069)(-5.6638,0.0077)(-5.6322,0.0085)(-5.6006,0.0094)(-5.5690,0.0104)(-5.5374,0.0115)(-5.5058,0.0127)(-5.4742,0.0140)(-5.4426,0.0153)(-5.4110,0.0168)(-5.3794,0.0185)(-5.3478,0.0201)(-5.3162,0.0220)(-5.2846,0.0239)(-5.2530,0.0259)(-5.2214,0.0281)(-5.1898,0.0303)(-5.1582,0.0327)(-5.1266,0.0350)(-5.0950,0.0378)(-5.0634,0.0405)(-5.0318,0.0434)(-5.0002,0.0465)(-4.9686,0.0495)(-4.9370,0.0526)(-4.9054,0.0561)(-4.8738,0.0595)(-4.8422,0.0629)(-4.8106,0.0665)(-4.7790,0.0700)(-4.7474,0.0734)(-4.7158,0.0774)(-4.6842,0.0812)(-4.6526,0.0847)(-4.6210,0.0882)(-4.5894,0.0919)(-4.5578,0.0957)(-4.5262,0.0993)(-4.4946,0.1030)(-4.4630,0.1065)(-4.4314,0.1097)(-4.3998,0.1133)(-4.3682,0.1162)(-4.3366,0.1192)(-4.3050,0.1218)(-4.2734,0.1248)(-4.2418,0.1275)(-4.2102,0.1297)(-4.1786,0.1316)(-4.1470,0.1335)(-4.1154,0.1352)(-4.0838,0.1367)(-4.0522,0.1374)(-4.0206,0.1383)(-3.9890,0.1391)(-3.9574,0.1395)(-3.9258,0.1391)(-3.8942,0.1391)(-3.8626,0.1385)(-3.8310,0.1378)(-3.7994,0.1365)(-3.7678,0.1356)(-3.7362,0.1339)(-3.7046,0.1322)(-3.6730,0.1300)(-3.6414,0.1279)(-3.6098,0.1253)(-3.5782,0.1227)(-3.5466,0.1200)(-3.5150,0.1169)(-3.4834,0.1138)(-3.4518,0.1106)(-3.4202,0.1073)(-3.3886,0.1037)(-3.3570,0.1005)(-3.3254,0.0965)(-3.2938,0.0931)(-3.2622,0.0896)(-3.2306,0.0857)(-3.1990,0.0819)(-3.1674,0.0785)(-3.1358,0.0746)(-3.1042,0.0711)(-3.0726,0.0676)(-3.0410,0.0642)(-3.0094,0.0608)(-2.9778,0.0575)(-2.9462,0.0542)(-2.9146,0.0512)(-2.8830,0.0482)(-2.8514,0.0453)(-2.8198,0.0427)(-2.7882,0.0404)(-2.7566,0.0377)(-2.7250,0.0355)(-2.6934,0.0336)(-2.6618,0.0316)(-2.6302,0.0299)(-2.5986,0.0283)(-2.5670,0.0269)(-2.5354,0.0256)(-2.5038,0.0246)(-2.4722,0.0234)(-2.4406,0.0227)(-2.4090,0.0222)(-2.3774,0.0217)(-2.3458,0.0215)(-2.3142,0.0214)(-2.2826,0.0216)(-2.2510,0.0218)(-2.2194,0.0222)(-2.1878,0.0229)(-2.1562,0.0236)(-2.1246,0.0248)(-2.0930,0.0258)(-2.0614,0.0271)(-2.0298,0.0286)(-1.9982,0.0303)(-1.9666,0.0322)(-1.9350,0.0340)(-1.9034,0.0362)(-1.8718,0.0386)(-1.8402,0.0410)(-1.8086,0.0435)(-1.7770,0.0463)(-1.7454,0.0491)(-1.7138,0.0525)(-1.6822,0.0554)(-1.6506,0.0586)(-1.6190,0.0621)(-1.5874,0.0656)(-1.5558,0.0692)(-1.5242,0.0730)(-1.4926,0.0766)(-1.4610,0.0805)(-1.4294,0.0842)(-1.3978,0.0883)(-1.3662,0.0922)(-1.3346,0.0963)(-1.3030,0.1001)(-1.2714,0.1040)(-1.2398,0.1082)(-1.2082,0.1122)(-1.1766,0.1156)(-1.1450,0.1194)(-1.1134,0.1232)(-1.0818,0.1266)(-1.0502,0.1302)(-1.0186,0.1339)(-0.9870,0.1372)(-0.9554,0.1403)(-0.9238,0.1431)(-0.8922,0.1463)(-0.8606,0.1486)(-0.8290,0.1513)(-0.7974,0.1538)(-0.7658,0.1561)(-0.7342,0.1578)(-0.7026,0.1600)(-0.6710,0.1619)(-0.6394,0.1635)(-0.6078,0.1650)(-0.5762,0.1662)(-0.5446,0.1675)(-0.5130,0.1685)(-0.4814,0.1695)(-0.4498,0.1702)(-0.4182,0.1709)(-0.3866,0.1715)(-0.3550,0.1716)(-0.3234,0.1724)(-0.2918,0.1727)(-0.2602,0.1730)(-0.2286,0.1737)(-0.1970,0.1731)(-0.1654,0.1735)(-0.1338,0.1738)(-0.1022,0.1736)(-0.0706,0.1740)(-0.0390,0.1737)(-0.0074,0.1743)(0.0242,0.1742)(0.0558,0.1737)(0.0874,0.1738)(0.1190,0.1739)(0.1506,0.1735)(0.1822,0.1737)(0.2138,0.1735)(0.2454,0.1733)(0.2770,0.1730)(0.3086,0.1725)(0.3402,0.1722)(0.3718,0.1717)(0.4034,0.1714)(0.4350,0.1704)(0.4666,0.1696)(0.4982,0.1689)(0.5298,0.1677)(0.5614,0.1667)(0.5930,0.1656)(0.6246,0.1639)(0.6562,0.1623)(0.6878,0.1610)(0.7194,0.1592)(0.7510,0.1572)(0.7826,0.1546)(0.8142,0.1524)(0.8458,0.1499)(0.8774,0.1473)(0.9090,0.1445)(0.9406,0.1417)(0.9722,0.1387)(1.0038,0.1353)(1.0354,0.1320)(1.0670,0.1286)(1.0986,0.1249)(1.1302,0.1214)(1.1618,0.1176)(1.1934,0.1138)(1.2250,0.1097)(1.2566,0.1058)(1.2882,0.1020)(1.3198,0.0980)(1.3514,0.0941)(1.3830,0.0900)(1.4146,0.0863)(1.4462,0.0822)(1.4778,0.0783)(1.5094,0.0745)(1.5410,0.0711)(1.5726,0.0673)(1.6042,0.0636)(1.6358,0.0602)(1.6674,0.0570)(1.6990,0.0536)(1.7306,0.0506)(1.7622,0.0477)(1.7938,0.0448)(1.8254,0.0420)(1.8570,0.0396)(1.8886,0.0372)(1.9202,0.0351)(1.9518,0.0329)(1.9834,0.0311)(2.0150,0.0294)(2.0466,0.0278)(2.0782,0.0263)(2.1098,0.0252)(2.1414,0.0241)(2.1730,0.0232)(2.2046,0.0226)(2.2362,0.0219)(2.2678,0.0218)(2.2994,0.0214)(2.3310,0.0216)(2.3626,0.0216)(2.3942,0.0220)(2.4258,0.0225)(2.4574,0.0232)(2.4890,0.0240)(2.5206,0.0251)(2.5522,0.0262)(2.5838,0.0275)(2.6154,0.0291)(2.6470,0.0308)(2.6786,0.0326)(2.7102,0.0346)(2.7418,0.0368)(2.7734,0.0390)(2.8050,0.0416)(2.8366,0.0443)(2.8682,0.0470)(2.8998,0.0497)(2.9314,0.0530)(2.9630,0.0559)(2.9946,0.0593)(3.0262,0.0625)(3.0578,0.0659)(3.0894,0.0694)(3.1210,0.0731)(3.1526,0.0768)(3.1842,0.0804)(3.2158,0.0840)(3.2474,0.0877)(3.2790,0.0915)(3.3106,0.0946)(3.3422,0.0987)(3.3738,0.1021)(3.4054,0.1055)(3.4370,0.1090)(3.4686,0.1123)(3.5002,0.1157)(3.5318,0.1185)(3.5634,0.1213)(3.5950,0.1241)(3.6266,0.1267)(3.6582,0.1289)(3.6898,0.1309)(3.7214,0.1330)(3.7530,0.1348)(3.7846,0.1360)(3.8162,0.1374)(3.8478,0.1382)(3.8794,0.1387)(3.9110,0.1390)(3.9426,0.1393)(3.9742,0.1391)(4.0058,0.1386)(4.0374,0.1380)(4.0690,0.1371)(4.1006,0.1356)(4.1322,0.1347)(4.1638,0.1325)(4.1954,0.1308)(4.2270,0.1282)(4.2586,0.1261)(4.2902,0.1234)(4.3218,0.1208)(4.3534,0.1179)(4.3850,0.1146)(4.4166,0.1114)(4.4482,0.1082)(4.4798,0.1046)(4.5114,0.1012)(4.5430,0.0975)(4.5746,0.0940)(4.6062,0.0901)(4.6378,0.0864)(4.6694,0.0827)(4.7010,0.0789)(4.7326,0.0752)(4.7642,0.0716)(4.7958,0.0678)(4.8274,0.0645)(4.8590,0.0609)(4.8906,0.0576)(4.9222,0.0542)(4.9538,0.0509)(4.9854,0.0479)(5.0170,0.0449)(5.0486,0.0419)(5.0802,0.0391)(5.1118,0.0364)(5.1434,0.0339)(5.1750,0.0315)(5.2066,0.0290)(5.2382,0.0268)(5.2698,0.0247)(5.3014,0.0228)(5.3330,0.0209)(5.3646,0.0192)(5.3962,0.0176)(5.4278,0.0161)(5.4594,0.0146)(5.4910,0.0133)(5.5226,0.0120)(5.5542,0.0110)(5.5858,0.0099)(5.6174,0.0089)(5.6490,0.0080)(5.6806,0.0072)(5.7122,0.0065)(5.7438,0.0058)(5.7754,0.0052)(5.8070,0.0046)(5.8386,0.0041)(5.8702,0.0037)(5.9018,0.0033)(5.9334,0.0029)(5.9650,0.0026)(5.9966,0.0023)(6.0282,0.0020)(6.0598,0.0017)(6.0914,0.0015)(6.1230,0.0013)(6.1546,0.0012)(6.1862,0.0010)(6.2178,0.0009)(6.2494,0.0007)(6.2810,0.0007)(6.3126,0.0006)(6.3442,0.0005)(6.3758,0.0004)(6.4074,0.0004)(6.4390,0.0003)(6.4706,0.0003)(6.5022,0.0002)(6.5338,0.0002)(6.5654,0.0002)(6.5970,0.0001)(6.6286,0.0001)(6.6602,0.0001)(6.6918,0.0001)(6.7234,0.0001)(6.7550,0.0001)(6.7866,0.0000)(6.8182,0.0000)(6.8498,0.0000)(6.8814,0.0000)(6.9130,0.0000)(6.9446,0.0000)(6.9762,0.0000)(7.0078,0.0000)(7.0394,0.0000)(7.0710,0.0000)(7.1026,0.0000)(7.1342,0.0000)(7.1658,0.0000)(7.1974,0.0000)(7.2290,0.0000)(7.2606,0.0000)(7.2922,0.0000)(7.3238,0.0000)(7.3554,0.0000)(7.3870,0.0000)(7.4186,0.0000)(7.4502,0.0000)(7.4818,0.0000)(7.5134,0.0000)(7.5450,0.0000)(7.5766,0.0000)(7.6082,0.0000)(7.6398,0.0000)(7.6714,0.0000)(7.7030,0.0000)(7.7346,0.0000)(7.7662,0.0000)(7.7978,0.0000)(7.8294,0.0000)(7.8610,0.0000)(7.8926,0.0000)(7.9242,0.0000)
       };
      \addplot[dotted,color=green!60!black!100,line width=3pt] plot coordinates{
      (-10.4591,0.0000)(-10.4173,0.0000)(-10.3755,0.0000)(-10.3336,0.0000)(-10.2918,0.0000)(-10.2500,0.0000)(-10.2082,0.0000)(-10.1664,0.0000)(-10.1245,0.0000)(-10.0827,0.0000)(-10.0409,0.0000)(-9.9991,0.0000)(-9.9573,0.0000)(-9.9154,0.0000)(-9.8736,0.0000)(-9.8318,0.0000)(-9.7900,0.0000)(-9.7481,0.0000)(-9.7063,0.0000)(-9.6645,0.0000)(-9.6227,0.0000)(-9.5809,0.0000)(-9.5390,0.0000)(-9.4972,0.0000)(-9.4554,0.0000)(-9.4136,0.0000)(-9.3718,0.0000)(-9.3300,0.0000)(-9.2881,0.0000)(-9.2463,0.0000)(-9.2045,0.0000)(-9.1627,0.0000)(-9.1209,0.0000)(-9.0790,0.0000)(-9.0372,0.0000)(-8.9954,0.0000)(-8.9536,0.0000)(-8.9118,0.0000)(-8.8699,0.0000)(-8.8281,0.0000)(-8.7863,0.0000)(-8.7445,0.0000)(-8.7027,0.0000)(-8.6608,0.0000)(-8.6190,0.0000)(-8.5772,0.0000)(-8.5354,0.0000)(-8.4935,0.0000)(-8.4517,0.0000)(-8.4099,0.0000)(-8.3681,0.0000)(-8.3263,0.0000)(-8.2844,0.0001)(-8.2426,0.0001)(-8.2008,0.0001)(-8.1590,0.0001)(-8.1172,0.0001)(-8.0754,0.0001)(-8.0335,0.0001)(-7.9917,0.0001)(-7.9499,0.0001)(-7.9081,0.0001)(-7.8663,0.0002)(-7.8244,0.0002)(-7.7826,0.0002)(-7.7408,0.0002)(-7.6990,0.0002)(-7.6572,0.0003)(-7.6153,0.0003)(-7.5735,0.0004)(-7.5317,0.0004)(-7.4899,0.0004)(-7.4481,0.0005)(-7.4062,0.0005)(-7.3644,0.0006)(-7.3226,0.0006)(-7.2808,0.0007)(-7.2390,0.0008)(-7.1971,0.0008)(-7.1553,0.0009)(-7.1135,0.0010)(-7.0717,0.0011)(-7.0299,0.0013)(-6.9880,0.0014)(-6.9462,0.0015)(-6.9044,0.0017)(-6.8626,0.0018)(-6.8208,0.0020)(-6.7789,0.0022)(-6.7371,0.0023)(-6.6953,0.0025)(-6.6535,0.0028)(-6.6117,0.0031)(-6.5698,0.0033)(-6.5280,0.0035)(-6.4862,0.0038)(-6.4444,0.0042)(-6.4026,0.0045)(-6.3607,0.0049)(-6.3189,0.0052)(-6.2771,0.0057)(-6.2353,0.0061)(-6.1935,0.0066)(-6.1516,0.0070)(-6.1098,0.0076)(-6.0680,0.0081)(-6.0262,0.0087)(-5.9844,0.0092)(-5.9425,0.0099)(-5.9007,0.0106)(-5.8589,0.0112)(-5.8171,0.0120)(-5.7753,0.0128)(-5.7334,0.0136)(-5.6916,0.0145)(-5.6498,0.0153)(-5.6080,0.0163)(-5.5662,0.0172)(-5.5243,0.0182)(-5.4825,0.0192)(-5.4407,0.0203)(-5.3989,0.0214)(-5.3571,0.0226)(-5.3152,0.0237)(-5.2734,0.0250)(-5.2316,0.0262)(-5.1898,0.0276)(-5.1480,0.0288)(-5.1061,0.0300)(-5.0643,0.0316)(-5.0225,0.0329)(-4.9807,0.0344)(-4.9389,0.0357)(-4.8970,0.0373)(-4.8552,0.0387)(-4.8134,0.0404)(-4.7716,0.0418)(-4.7298,0.0433)(-4.6879,0.0450)(-4.6461,0.0466)(-4.6043,0.0481)(-4.5625,0.0496)(-4.5207,0.0513)(-4.4788,0.0529)(-4.4370,0.0544)(-4.3952,0.0561)(-4.3534,0.0575)(-4.3116,0.0590)(-4.2697,0.0608)(-4.2279,0.0621)(-4.1861,0.0636)(-4.1443,0.0650)(-4.1025,0.0667)(-4.0606,0.0679)(-4.0188,0.0693)(-3.9770,0.0705)(-3.9352,0.0719)(-3.8934,0.0730)(-3.8515,0.0745)(-3.8097,0.0756)(-3.7679,0.0768)(-3.7261,0.0776)(-3.6843,0.0788)(-3.6424,0.0798)(-3.6006,0.0807)(-3.5588,0.0817)(-3.5170,0.0824)(-3.4752,0.0832)(-3.4333,0.0838)(-3.3915,0.0846)(-3.3497,0.0851)(-3.3079,0.0857)(-3.2661,0.0863)(-3.2242,0.0867)(-3.1824,0.0871)(-3.1406,0.0877)(-3.0988,0.0880)(-3.0570,0.0882)(-3.0151,0.0886)(-2.9733,0.0889)(-2.9315,0.0892)(-2.8897,0.0892)(-2.8479,0.0896)(-2.8060,0.0898)(-2.7642,0.0901)(-2.7224,0.0902)(-2.6806,0.0903)(-2.6388,0.0907)(-2.5969,0.0909)(-2.5551,0.0910)(-2.5133,0.0910)(-2.4715,0.0914)(-2.4297,0.0918)(-2.3878,0.0922)(-2.3460,0.0924)(-2.3042,0.0929)(-2.2624,0.0931)(-2.2206,0.0936)(-2.1787,0.0941)(-2.1369,0.0945)(-2.0951,0.0951)(-2.0533,0.0957)(-2.0115,0.0965)(-1.9696,0.0972)(-1.9278,0.0980)(-1.8860,0.0987)(-1.8442,0.0994)(-1.8024,0.1004)(-1.7605,0.1015)(-1.7187,0.1026)(-1.6769,0.1037)(-1.6351,0.1048)(-1.5933,0.1058)(-1.5514,0.1071)(-1.5096,0.1085)(-1.4678,0.1094)(-1.4260,0.1112)(-1.3842,0.1126)(-1.3423,0.1140)(-1.3005,0.1154)(-1.2587,0.1171)(-1.2169,0.1183)(-1.1751,0.1197)(-1.1332,0.1213)(-1.0914,0.1229)(-1.0496,0.1245)(-1.0078,0.1259)(-0.9660,0.1275)(-0.9241,0.1291)(-0.8823,0.1303)(-0.8405,0.1320)(-0.7987,0.1335)(-0.7569,0.1348)(-0.7150,0.1361)(-0.6732,0.1375)(-0.6314,0.1387)(-0.5896,0.1398)(-0.5478,0.1408)(-0.5059,0.1421)(-0.4641,0.1432)(-0.4223,0.1442)(-0.3805,0.1447)(-0.3387,0.1458)(-0.2968,0.1464)(-0.2550,0.1471)(-0.2132,0.1475)(-0.1714,0.1478)(-0.1296,0.1482)(-0.0877,0.1484)(-0.0459,0.1487)(-0.0041,0.1488)(0.0377,0.1486)(0.0795,0.1485)(0.1214,0.1482)(0.1632,0.1481)(0.2050,0.1476)(0.2468,0.1471)(0.2886,0.1463)(0.3305,0.1457)(0.3723,0.1449)(0.4141,0.1442)(0.4559,0.1433)(0.4977,0.1424)(0.5396,0.1414)(0.5814,0.1401)(0.6232,0.1388)(0.6650,0.1378)(0.7068,0.1365)(0.7487,0.1351)(0.7905,0.1339)(0.8323,0.1323)(0.8741,0.1306)(0.9159,0.1292)(0.9578,0.1276)(0.9996,0.1263)(1.0414,0.1248)(1.0832,0.1232)(1.1250,0.1216)(1.1669,0.1202)(1.2087,0.1187)(1.2505,0.1172)(1.2923,0.1156)(1.3341,0.1140)(1.3760,0.1128)(1.4178,0.1113)(1.4596,0.1101)(1.5014,0.1089)(1.5432,0.1073)(1.5851,0.1060)(1.6269,0.1052)(1.6687,0.1036)(1.7105,0.1028)(1.7523,0.1018)(1.7942,0.1007)(1.8360,0.0996)(1.8778,0.0988)(1.9196,0.0979)(1.9614,0.0974)(2.0033,0.0964)(2.0451,0.0957)(2.0869,0.0952)(2.1287,0.0947)(2.1705,0.0938)(2.2124,0.0937)(2.2542,0.0931)(2.2960,0.0930)(2.3378,0.0924)(2.3796,0.0922)(2.4215,0.0919)(2.4633,0.0914)(2.5051,0.0911)(2.5469,0.0913)(2.5887,0.0910)(2.6306,0.0906)(2.6724,0.0903)(2.7142,0.0903)(2.7560,0.0902)(2.7978,0.0900)(2.8397,0.0897)(2.8815,0.0894)(2.9233,0.0891)(2.9651,0.0890)(3.0069,0.0887)(3.0488,0.0884)(3.0906,0.0881)(3.1324,0.0879)(3.1742,0.0874)(3.2160,0.0869)(3.2579,0.0864)(3.2997,0.0858)(3.3415,0.0852)(3.3833,0.0846)(3.4251,0.0839)(3.4670,0.0833)(3.5088,0.0826)(3.5506,0.0815)(3.5924,0.0808)(3.6342,0.0799)(3.6761,0.0788)(3.7179,0.0779)(3.7597,0.0769)(3.8015,0.0759)(3.8433,0.0745)(3.8852,0.0733)(3.9270,0.0722)(3.9688,0.0708)(4.0106,0.0696)(4.0524,0.0682)(4.0943,0.0666)(4.1361,0.0654)(4.1779,0.0640)(4.2197,0.0624)(4.2615,0.0608)(4.3034,0.0593)(4.3452,0.0579)(4.3870,0.0563)(4.4288,0.0547)(4.4706,0.0532)(4.5125,0.0515)(4.5543,0.0500)(4.5961,0.0484)(4.6379,0.0466)(4.6797,0.0452)(4.7216,0.0436)(4.7634,0.0421)(4.8052,0.0405)(4.8470,0.0390)(4.8888,0.0375)(4.9307,0.0361)(4.9725,0.0346)(5.0143,0.0332)(5.0561,0.0318)(5.0979,0.0305)(5.1398,0.0291)(5.1816,0.0276)(5.2234,0.0264)(5.2652,0.0252)(5.3070,0.0240)(5.3489,0.0228)(5.3907,0.0216)(5.4325,0.0206)(5.4743,0.0194)(5.5161,0.0185)(5.5580,0.0174)(5.5998,0.0164)(5.6416,0.0155)(5.6834,0.0146)(5.7252,0.0138)(5.7671,0.0130)(5.8089,0.0122)(5.8507,0.0114)(5.8925,0.0107)(5.9343,0.0101)(5.9762,0.0095)(6.0180,0.0088)(6.0598,0.0082)(6.1016,0.0076)(6.1434,0.0071)(6.1853,0.0066)(6.2271,0.0062)(6.2689,0.0058)(6.3107,0.0054)(6.3525,0.0050)(6.3944,0.0046)(6.4362,0.0042)(6.4780,0.0039)(6.5198,0.0036)(6.5616,0.0033)(6.6035,0.0031)(6.6453,0.0028)(6.6871,0.0026)(6.7289,0.0024)(6.7707,0.0022)(6.8126,0.0020)(6.8544,0.0018)(6.8962,0.0017)(6.9380,0.0015)(6.9798,0.0014)(7.0217,0.0013)(7.0635,0.0012)(7.1053,0.0011)(7.1471,0.0010)(7.1889,0.0009)(7.2308,0.0008)(7.2726,0.0007)(7.3144,0.0007)(7.3562,0.0006)(7.3980,0.0005)(7.4399,0.0005)(7.4817,0.0004)(7.5235,0.0004)(7.5653,0.0003)(7.6071,0.0003)(7.6490,0.0003)(7.6908,0.0003)(7.7326,0.0002)(7.7744,0.0002)(7.8162,0.0002)(7.8581,0.0002)(7.8999,0.0001)(7.9417,0.0001)(7.9835,0.0001)(8.0253,0.0001)(8.0672,0.0001)(8.1090,0.0001)(8.1508,0.0001)(8.1926,0.0001)(8.2344,0.0001)(8.2763,0.0000)(8.3181,0.0000)(8.3599,0.0000)(8.4017,0.0000)(8.4435,0.0000)(8.4854,0.0000)(8.5272,0.0000)(8.5690,0.0000)(8.6108,0.0000)(8.6526,0.0000)(8.6945,0.0000)(8.7363,0.0000)(8.7781,0.0000)(8.8199,0.0000)(8.8617,0.0000)(8.9036,0.0000)(8.9454,0.0000)(8.9872,0.0000)(9.0290,0.0000)(9.0708,0.0000)(9.1127,0.0000)(9.1545,0.0000)(9.1963,0.0000)(9.2381,0.0000)(9.2799,0.0000)(9.3218,0.0000)(9.3636,0.0000)(9.4054,0.0000)(9.4472,0.0000)(9.4890,0.0000)(9.5309,0.0000)(9.5727,0.0000)(9.6145,0.0000)(9.6563,0.0000)(9.6981,0.0000)(9.7400,0.0000)(9.7818,0.0000)(9.8236,0.0000)(9.8654,0.0000)(9.9072,0.0000)(9.9491,0.0000)(9.9909,0.0000)(10.0327,0.0000)(10.0745,0.0000)(10.1163,0.0000)(10.1582,0.0000)(10.2000,0.0000)(10.2418,0.0000)(10.2836,0.0000)(10.3254,0.0000)(10.3673,0.0000)(10.4091,0.0000)
       };
        \legend{,,$(\x^{\prime})^\T\tilde\bbeta$,$(\x^{\prime})^\T\tilde\bbeta^\g$};
      \end{axis}
      \end{tikzpicture}
   \end{tabular}
   \caption{Empirical and theoretical results under an LFMM with $p=200$, $\rho=0.5$, $s=[\sqrt{2};\mathbf{0}_{p-1}]$, Rademacher $e_1$, normal $e_2,\ldots,e_p$, and Haar distributed $\V$. 
  \textbf{Top}: scatter plot of $200$ independent $\big[r,h_\kappa(r,\pm1)\big]$.
  \textbf{Bottom}: histograms of predicted scores on $10^6$ fresh samples $(\x',y')\sim\mathcal{D}_{(\x,y)}$ given by $\hat\bbeta$ and $\hat\bbeta^\g$, versus theoretical densities obtained from \Cref{theo:main}. 
  \textbf{Left}: $n=100$, square loss $\ell(\hat y,y)=(\hat y-y)^2/2$. 
  \textbf{Right}: $n=600$, square hinge loss $\ell(\hat y,y)=\max\{0,(1-\hat y y)\}^2$.  }
   \label{fig: universality beta}
\end{figure}

\Cref{rem:square loss}, supported by the numerically results in \Cref{fig: universality beta},
points to the insensitivity of least-square classifiers to the distributions of non-Gaussian informative factors, despite their non-universal impact on the in-distribution performance as discussed in \Cref{sec:universality error}. 
The incapacity of square loss to account for non-Gaussian variations in the informative factors sheds light on the suboptimality of square loss observed in the right display of \Cref{fig: optimal loss}, where the logistic loss yields better performance than the square loss with \emph{optimally chosen} regularization $\lambda$ on LFMM having non-Gaussian informative factors, while the logistic loss \emph{fails} to do better than the square loss under the equivalent GMM in the left-hand figure.
Further experiments on real-world datasets are given in \Cref{sec:real_data}.

This finding on the suboptimality of square loss  under LFMM provide new insights on the impact of loss function beyond previous optimality results of square loss under GMM. 
For high-dimensional GMM data, the square loss has been proven optimal, see \citep{taheri2021sharp} for the case of unregularized ERM, and \citep{mai2019high} for ridge-regularized ERM, in the $n,p \to \infty$ limit. 
That is, the square loss not only gives the best unbiased classifier, but also allows for an \emph{optimal} bias-variance trade-off with well calibrated ridge-regularization. 
As a result of the Gaussian universality breakdown discussed above, the optimality of square loss is no longer valid under the more general LFMM. 
This motivates a few open questions on the optimal loss:
\begin{itemize}
    \item Is the square loss optimal \emph{only} under GMM, or when the Gaussian universality of in-distribution performance in \Cref{def:gaussian_universality} holds?
    \item In the case of Gaussian universality breakdown, does the optimal loss depend on the sample ratio $n/p$ as in the setting of linear regression in~\citep{bean2013optimal}? 
    \item Is it possible, in the large $n,p$ regime, to propose an optimal design of classification loss adapted to the data distribution and sample size?
\end{itemize}

\begin{figure}
\centering
\begin{tabular}{cc}
  GMM & LFMM\\
\begin{tikzpicture}[font=\footnotesize]
\renewcommand{\axisdefaulttryminticks}{4} 
\pgfplotsset{every axis legend/.append style={cells={anchor=west},fill=white, at={(0.6,0.48)}, anchor=north east, font=\scriptsize}}
\pgfplotsset{/pgfplots/error bars/error bar style={thick}}
\pgfplotsset{every axis plot/.append style={thick},}
\begin{axis}[
height=0.35\linewidth,
width=0.45\linewidth,
ymin=0.75,
ymax=1,
xmin=0.0019,
xmax=64,
xmode=log,
log basis x ={2},
grid=major,
ymajorgrids=false,
scaled ticks=true,
xlabel={$\lambda$},
ylabel={ Classification Accuracy ($\%$)  },
]
\addplot[color = GREEN, only marks, mark=triangle, mark size= 1.7pt,error bars/.cd, y dir=both,y explicit] coordinates {
(0.0000,0.904466)+-(1,0.004293)(0.0001,0.904508)+-(1,0.004038)(0.0001,0.904690)+-(1,0.004102)(0.0002,0.904664)+-(1,0.004293)(0.0005,0.904620)+-(1,0.004214)(0.0010,0.904734)+-(1,0.004206)(0.0020,0.904759)+-(1,0.004221)(0.0039,0.904887)+-(1,0.004258)(0.0078,0.905393)+-(1,0.003983)(0.0156,0.906102)+-(1,0.003922)(0.0312,0.907611)+-(1,0.003871)(0.0625,0.909753)+-(1,0.003497)(0.1250,0.913292)+-(1,0.003150)(0.2500,0.918229)+-(1,0.002869)(0.5000,0.922648)+-(1,0.002441)(1.0000,0.923557)+-(1,0.002910)(2.0000,0.916506)+-(1,0.004107)(4.0000,0.896041)+-(1,0.006122)(8.0000,0.864815)+-(1,0.007381)(16.0000,0.833070)+-(1,0.007171)(32.0000,0.809798)+-(1,0.006559)(64.0000,0.795291)+-(1,0.005890)(128.0000,0.787677)+-(1,0.005636)(256.0000,0.783641)+-(1,0.005546)(512.0000,0.781284)+-(1,0.005242)(1024.0000,0.780128)+-(1,0.005254)(2048.0000,0.779593)+-(1,0.005089)(4096.0000,0.779395)+-(1,0.005282)(8192.0000,0.779549)+-(1,0.005301)(16384.0000,0.779339)+-(1,0.005025)(32768.0000,0.779189)+-(1,0.005137)
      
      };
\addlegendentry{{ Square loss, Emp }};

\addplot[color = RED, only marks, mark=o, mark size= 1.7pt,error bars/.cd, y dir=both,y explicit] coordinates {
(0.0000,0.888416)+-(1,0.007918)(0.0001,0.889489)+-(1,0.007597)(0.0001,0.890663)+-(1,0.007572)(0.0002,0.892160)+-(1,0.007262)(0.0005,0.894107)+-(1,0.006924)(0.0010,0.896671)+-(1,0.006303)(0.0020,0.900152)+-(1,0.005709)(0.0039,0.904657)+-(1,0.004804)(0.0078,0.909744)+-(1,0.003860)(0.0156,0.914785)+-(1,0.003165)(0.0312,0.919303)+-(1,0.002769)(0.0625,0.921968)+-(1,0.002796)(0.1250,0.921980)+-(1,0.003334)(0.2500,0.917436)+-(1,0.004504)(0.5000,0.905861)+-(1,0.005957)(1.0000,0.884844)+-(1,0.007376)(2.0000,0.857071)+-(1,0.008023)(4.0000,0.829434)+-(1,0.007499)(8.0000,0.808891)+-(1,0.006491)(16.0000,0.795230)+-(1,0.005955)(32.0000,0.787911)+-(1,0.005745)(64.0000,0.783329)+-(1,0.005529)(128.0000,0.781415)+-(1,0.005028)(256.0000,0.780348)+-(1,0.005332)(512.0000,0.779814)+-(1,0.005303)(1024.0000,0.779528)+-(1,0.005125)(2048.0000,0.779357)+-(1,0.005097)(4096.0000,0.779491)+-(1,0.005182)(8192.0000,0.779272)+-(1,0.005266)(16384.0000,0.779363)+-(1,0.005247)(32768.0000,0.779598)+-(1,0.005256)
      };
      \addlegendentry{{ Log loss, Emp }};

   \addplot[color = GREEN, smooth,line width=1pt] coordinates {
(0.0010,0.904250)(0.0020,0.904405)(0.0039,0.904527)(0.0078,0.904818)(0.0156,0.905502)(0.0312,0.906766)(0.0625,0.909267)(0.1250,0.913121)(0.2500,0.917878)(0.5000,0.922472)(1.0000,0.923225)(2.0000,0.915582)(4.0000,0.895046)(8.0000,0.862972)(16.0000,0.831139)(32.0000,0.807903)(64.0000,0.794063)(128.0000,0.786072)(256.0000,0.782227)(512.0000,0.779947)(1024.0000,0.779006)

      };
\addlegendentry{{ Square loss, Theo}};
    \addplot[color = RED, dashed,line width=1pt] coordinates {

  (0.0010,0.893300)(0.0020,0.900599)(0.0039,0.905627)(0.0078,0.910659)(0.0156,0.914353)(0.0312,0.916618)(0.0625,0.918399)(0.1250,0.919643)(0.2500,0.918089)(0.5000,0.910030)(1.0000,0.892622)(2.0000,0.866317)(4.0000,0.835784)(8.0000,0.808339)(16.0000,0.790289)(32.0000,0.783062)(64.0000,0.782025)(128.0000,0.780326)(256.0000,0.776747)(512.0000,0.778945)(1024.0000,0.779756)
  
      };
\addlegendentry{{ Log loss, Theo }};

      \end{axis}
      \end{tikzpicture}
&
\begin{tikzpicture}[font=\footnotesize]
\renewcommand{\axisdefaulttryminticks}{4} 
\pgfplotsset{every axis legend/.append style={cells={anchor=west},fill=white, at={(0.2,0.2)}, anchor=north east, font=\scriptsize}}
\pgfplotsset{/pgfplots/error bars/error bar style={thick}}
\pgfplotsset{every axis plot/.append style={thick},}
\begin{axis}[
height=0.35\linewidth,
width=0.45\linewidth,
ymin=0.75,
ymax=1,
xmin=0.0019,
xmax=64,
xmode=log,
log basis x ={2},
grid=major,
ymajorgrids=false,
scaled ticks=true,
ytick={0.16,0.18,0.20},
yticklabels={$16$,$18$,$20$},
xlabel={$\lambda$},
ylabel={},
]
\addplot[color = GREEN, only marks, mark=triangle, mark size= 1.7pt,error bars/.cd, y dir=both,y explicit] coordinates {

(0.0000,0.899507)+-(1,0.005704)(0.0001,0.899695)+-(1,0.005641)(0.0001,0.899715)+-(1,0.005544)(0.0002,0.899882)+-(1,0.005522)(0.0005,0.899699)+-(1,0.005538)(0.0010,0.900079)+-(1,0.005366)(0.0020,0.900084)+-(1,0.005534)(0.0039,0.900267)+-(1,0.005552)(0.0078,0.900969)+-(1,0.005356)(0.0156,0.901864)+-(1,0.005438)(0.0312,0.903765)+-(1,0.005421)(0.0625,0.907213)+-(1,0.005144)(0.1250,0.912664)+-(1,0.004847)(0.2500,0.919991)+-(1,0.004565)(0.5000,0.925963)+-(1,0.004488)(1.0000,0.924999)+-(1,0.004618)(2.0000,0.913998)+-(1,0.004973)(4.0000,0.892375)+-(1,0.006080)(8.0000,0.862611)+-(1,0.007155)(16.0000,0.831755)+-(1,0.007172)(32.0000,0.808919)+-(1,0.007057)(64.0000,0.795085)+-(1,0.006309)(128.0000,0.787243)+-(1,0.006197)(256.0000,0.783173)+-(1,0.005980)(512.0000,0.781113)+-(1,0.005861)(1024.0000,0.780033)+-(1,0.006277)(2048.0000,0.779576)+-(1,0.005898)(4096.0000,0.779316)+-(1,0.005941)(8192.0000,0.779064)+-(1,0.006033)(16384.0000,0.778969)+-(1,0.005858)(32768.0000,0.779341)+-(1,0.005966)
      };


\addplot[color = RED, only marks, mark=o, mark size= 1.7pt,error bars/.cd, y dir=both,y explicit] coordinates {

(0.0000,0.972149)+-(1,0.015627)(0.0001,0.969655)+-(1,0.015709)(0.0001,0.966294)+-(1,0.015591)(0.0002,0.962518)+-(1,0.015279)(0.0005,0.957884)+-(1,0.014614)(0.0010,0.952537)+-(1,0.013442)(0.0020,0.946506)+-(1,0.011930)(0.0039,0.940221)+-(1,0.009908)(0.0078,0.934911)+-(1,0.007691)(0.0156,0.931506)+-(1,0.006094)(0.0312,0.929648)+-(1,0.005066)(0.0625,0.927294)+-(1,0.004723)(0.1250,0.923175)+-(1,0.005036)(0.2500,0.915023)+-(1,0.005296)(0.5000,0.901713)+-(1,0.006198)(1.0000,0.881376)+-(1,0.007285)(2.0000,0.854853)+-(1,0.007938)(4.0000,0.828303)+-(1,0.007581)(8.0000,0.807977)+-(1,0.007092)(16.0000,0.794812)+-(1,0.006421)(32.0000,0.787227)+-(1,0.006241)(64.0000,0.783281)+-(1,0.005870)(128.0000,0.781026)+-(1,0.006171)(256.0000,0.780001)+-(1,0.006105)(512.0000,0.779520)+-(1,0.006065)(1024.0000,0.779167)+-(1,0.005738)(2048.0000,0.779042)+-(1,0.006102)(4096.0000,0.778989)+-(1,0.005690)(8192.0000,0.779006)+-(1,0.006068)(16384.0000,0.778938)+-(1,0.005902)(32768.0000,0.778863)+-(1,0.005942)
      };

\addplot[color = GREEN, smooth,line width=1pt] coordinates {

(0.0000,0.899379)(0.0001,0.899573)(0.0001,0.899492)(0.0002,0.899310)(0.0005,0.899566)(0.0010,0.899432)(0.0020,0.899546)(0.0039,0.899898)(0.0078,0.900521)(0.0156,0.901514)(0.0312,0.903286)(0.0625,0.907032)(0.1250,0.912502)(0.2500,0.920015)(0.5000,0.926015)(1.0000,0.925410)(2.0000,0.914047)(4.0000,0.892651)(8.0000,0.862131)(16.0000,0.831218)(32.0000,0.808421)(64.0000,0.794180)(128.0000,0.786446)(256.0000,0.782265)(512.0000,0.780272)(1024.0000,0.779197)(2048.0000,0.778731)(4096.0000,0.778356)(8192.0000,0.778277)(16384.0000,0.778096)(32768.0000,0.778191)

      };

    \addplot[color = RED, dashed,line width=1pt] coordinates {

(0.0010,0.972301)(0.0020,0.945786)(0.0039,0.935899)(0.0078,0.931336)(0.0156,0.929568)(0.0312,0.929947)(0.0625,0.930886)(0.1250,0.929532)(0.2500,0.922725)(0.5000,0.908307)(1.0000,0.886157)(2.0000,0.858555)(4.0000,0.829752)(8.0000,0.804808)(16.0000,0.787962)(32.0000,0.780976)(64.0000,0.782019)(128.0000,0.785801)(256.0000,0.785753)(512.0000,0.779136)(1024.0000,0.775994)
  
      };

      \end{axis}
      \end{tikzpicture}
   \end{tabular}
         \caption{Empirical classification accuracy of $\hat\bbeta_{\ell,\lambda}$ computed over $10^5$ independent copies of $(\x',y')\sim\mathcal{D}_{(\x,y)}$ and averaged over $100$ trials with a width of $\pm 1$ standard deviation, versus theoretical performance given in \Cref{theo:main}, given by the square loss $\ell(\hat y, y)  = (y-\hat y)^2/2$ and the logistic loss $\ell(\hat y, y) = -\ln(1/(1+e^{-y\hat y}))$ and on $n=800$ training samples. 
      \textbf{Left}: GMM under \Cref{def:linear_factor} with $p=200$, $\rho=0.5$, $s=[1,5;0.5;\zeros_{p-2}]$ (so that $q = 2$), and $\V={\rm diag}(2,\ones_{p-1})\H$ with Haar distributed $\H$. 
      \textbf{Right}: LFMM identical to the GMM in the left, but with Rademacher $e_1$. }
         \label{fig: optimal loss}
\end{figure}

\section{Concluding Remarks}
\label{sec: conclusion}

Our analysis considered a basic framework of linear factor mixture models (LFMM) and showed that the Gaussian universality can already break down under this natural extension of GMM. Based on the precise performance characterization, we derived conditions of Gaussian universality to shed light on the limit of the widely observed and extensively studied Gaussian universality phenomenon.

Breaking the Gaussian universality in classification of mixture models allows also deeper insight into the choice of classification loss beyond the optimality of square loss under GMM~\citep{taheri2021sharp,mai2019high}. The suboptimality of square loss under LFMM can be further investigated in future works, to propose, for instance, an optimal design of loss function that takes into account the data distribution and the sample size, as done by \citet{bean2013optimal} for linear regression.

Several simplifications made in our analysis can be removed more or less easily. For instance, while the extension to multi-classification is fairly straightforward, the generalization to non-smooth losses is less direct: even though our system of equations in \eqref{eq:theta eta gamma} does not require access to the derivatives of the loss function, they are involved in the establishment of these equations.

\section*{Acknowledgments}

Z.~Liao would like to acknowledge the National Natural Science Foundation of China (via fund NSFC-62206101 and NSFC-12141107) and the Guangdong Provincial Key Laboratory of Mathematical Foundations for Artificial Intelligence (2023B1212010001) for providing partial support.
This work is also supported by the AI Cluster ANITI (ANR-19-PI3A-0004).

\bibliography{mai}
\bibliographystyle{plainnat}

\clearpage
\appendix

\section{Proofs}

Here, we present the detailed proofs of our theoretical results in the paper.
Precisely, the proof of \Cref{theo:main} is given in \Cref{sm:proof-of-main-results} and the proofs of corollaries are given in \Cref{subsec:proof_of_coro} (the proof of \Cref{cor:performance} in \Cref{subsubsec:proof_of_cor_performance}, the proof of \Cref{cor:condition for r} in \Cref{subsubsec:proof_of_condition for r} and that of \Cref{cor:condition for beta} in \Cref{subsubsec:proof_of_condition for beta}, respectively).

{\bf Notations.} Before getting into the proofs, we first introduce the following asymptotic notations. The big $O$ notation \(O(u_n)\) is understood here in probability. 
We specify that when multidimensional objects are concerned, $O(u_n)$ is understood entry-wise. 
The notation \(O_{\Vert\cdot\Vert}(\cdot)\) is understood as follows: for a vector $\v$, $\v=O_{\Vert\cdot\Vert}(u_n)$ means its Euclidean norm is $O(u_n)$ and for a square matrix $\mathbf{M}$, $\mathbf{M}=O_{\Vert\cdot\Vert}(u_n)$ means that the operator norm of $\mathbf{M}$ is $O(u_n)$. 
The small $o$ notation and Big-Theta $\Theta$ are understood likewise. Note that under Assumption~\ref{ass:growth-rate} it is equivalent to use either $O(u_n)$ or $O(u_p)$ since $n, p$ scales linearly. 
In the following we shall use constantly $O(u_p)$ for simplicity of exposition. The symbol $\simeq$ is used in the following sense: for a scalar $s=O(1)$, $s\simeq \tilde s$ indicates that $s-\tilde s=o(1)$, and for a vector $\v$ with $\Vert \v\Vert=O(1)$, $\v\simeq \tilde \v$ means $\Vert \v-\tilde \v\Vert=o(1)$. 
For random variable $r\sim\mathcal{N}(m,\sigma^2)$ with potentially random $m$ and $\sigma^2$, the expectation $\E[f(r)]$ should be understood as conditioned on $m,\sigma^2$, so that, $\E[r]$ is equal to $m$ instead of $\E[m]$. When parametrized functions $f_{\tau}(\cdot)$ are involved, $\E[f_{\tau}(r)]$ is computed by taking the integral over $r$.

\subsection{Proof of \Cref{theo:main}}\label{sm:proof-of-main-results}

In this section, we will provide the proof of \Cref{theo:main}. We start by explaining the main idea of leave-one-out and laying out the key steps as a guide of our proof.

\subsubsection{Main idea and key steps}

Taking $\lambda>0$ in the optimization problem \eqref{eq:opt-origin-reg} ensures a unique solution $\hat \bbeta$. 
Cancelling the gradient (with respect to $\bbeta$) of the objective function in \eqref{eq:opt-origin-reg}, we obtain the following stationary-point expression of $\hat\bbeta$
\begin{equation}\label{eq:lambda-beta}
\lambda \hat \bbeta = -\frac1n \sum_{i=1}^n \ell'(\hat\bbeta^\T\x_i,y_i)\x_i,
\end{equation}
where we denote $\ell'(t,y_i) = \frac{\partial \ell(t,y_i)}{\partial  t}$.

To characterize the behavior of $ \hat \bbeta$ from \eqref{eq:lambda-beta}, we need to assess the statistical behavior of $\hat\bbeta^\T\x_i$, which is not directly tractable due to the intricate dependence between $\hat\bbeta$ on $\x_i$ resulted from the implicit optimization \eqref{eq:opt-origin-reg}. 
To tackle this complication, we make use of a ``leave-one-out'' version $\hat \bbeta_{-i}$ of $\hat \bbeta$ with respect to the $i$-th training sample $(\x_i,y)$, obtained by solving \eqref{eq:opt-origin-reg} with all the remaining $n-1$ training samples $(\x_j,y_j)$ for $j \neq i$. 
Again we have
\begin{equation}\label{eq:lambda-beta-loo}
\lambda \hat \bbeta_{-i} = -\frac1n \sum_{j\neq i}  \ell'(\hat\bbeta_{-i}^\T\x_j,y_j)\x_j.
\end{equation}

This leave-one-out solution $\hat \bbeta_{-i}$ has two crucial properties: (i) it is by definition \emph{independent} of the left-out data sample $(\x_i,y_i)$; and (ii) it is asymptotically close to the original solution $\hat \bbeta$ as removing one among $n$ training samples has a negligible effect as $n \to \infty$. These two properties imply that $\hat \bbeta_{-i}$ and $\hat \bbeta$ behave similarly on all training or test samples, \emph{except} on $(\x_i,y_i)$, which is a training sample for $\hat \bbeta$ and a new observation for $\hat \bbeta_{-i}$.

Our proof relies on these two properties to derive a series of equations characterizing the limiting behavior of $\hat\bbeta^\T \x_i$ and $\hat\bbeta^\T \x'$.
Below is an overview of our key steps to guide the readers through the proof.

{\bf Key steps:}
\begin{enumerate}
    \item \label{enum:proof_step_1} Establishing the high-dimensional approximation
    \begin{equation}
    \label{eq:relation leave-one-out score}
        \hat \bbeta^\T\x_i-\hat \bbeta_{-i}^\T\x_i\simeq-\kappa\ell'(\hat\bbeta^\T\x_i,y_i),
    \end{equation}
    for some constant $\kappa$ independent of $i$.

    \item \label{enum:proof_step_2} Obtaining from \eqref{eq:relation leave-one-out score} that
    $$\hat \bbeta^\T\x_i \simeq \prox_{\kappa, \ell(\cdot,y_i)}(\hat \bbeta_{-i}^\T\x_i,y_i),$$ 
    and therefore 
    $$-\ell'(\hat\bbeta^\T\x_i,y_i)\simeq h(\hat \bbeta_{-i}^\T\x_i,y_i)\equiv\frac{\prox_{\kappa, \ell(\cdot,y_i)}(\hat\bbeta_{-i}^\T\x_i,y_i)-\hat\bbeta_{-i}^\T\x_i}{\kappa},$$
    where we recall $h(t,y) = (\prox_{\kappa, \ell(\cdot,y)} (t)-t)/\kappa$ with proximal operator
    \[
        \prox_{\tau, f} (t)=\argmin_{a \in \RR} \left[ f(a) + \frac{1}{2\tau} (a - t)^2 \right],
    \]
    for $\tau > 0$ and convex $f\colon \RR \to \RR$.

   \item \label{enum:proof_step_3} Using the approximation in Step~\ref{enum:proof_step_2} to rewrite \eqref{eq:lambda-beta} as 
   $$\lambda \hat \bbeta \simeq  \frac1n \sum_{i=1}^n h(\hat\bbeta_{-i}^\T\x_i,y_i)\x_i,$$
   for $$\x_i = \V \z_i=y_i\bmu+\V\e_i,$$    
   with $\e_i$ the noise vector $\e=[e_1,\ldots,e_p]^\T \in \RR^p$ for $\x_i$,
   thereby replacing the intractable $\hat\bbeta^\T\x_i$ in \eqref{eq:lambda-beta} with a tractable function of $\hat\bbeta_{-i}^\T\x_i$. 

   \item \label{enum:proof_step_4} Demonstrating the concentration result
   $$\frac1n \sum_{i=1}^n h(\hat\bbeta_{-i}^\T\x_i,y_i)y_i\simeq\eta,$$
   for some deterministic $\eta$.

   \item \label{enum:proof_step_5} Demonstrating the concentration results
   $$\frac1n \sum_{i=1}^n h(\hat\bbeta_{-i}^\T\x_i,y_i)[\e_i]_k\simeq \phi_k,\quad \forall k\in\{1,\ldots,q\},$$
   for some deterministic $\phi_1,\ldots,\phi_q$, and
  $$\frac1n \sum_{i=1}^n h(\hat\bbeta_{-i}^\T\x_i,y_i)[\e_i]_k\simeq 0,\quad \forall k\in\{q+1,\ldots,p\}.$$

    \item \label{enum:proof_step_6} Demonstrating with concentration arguments and CLT that
    $$\frac1n \sum_{i=1}^n h(\hat\bbeta_{-i}^\T\x_i,y_i)\tilde\e_i\simeq -\theta \cdot \V_{\rm noise}\hat\bbeta+\bepsilon,$$
    where $\tilde\e_i=[\e_i]_{q+1:p}$, $\V_{\rm noise}=[\v_{q+1},\ldots,\v_p]$ and $\bepsilon\in\RR^{p-q}$ a random vector such that, for any deterministic vector $\t=[t_{q+1},\ldots,t_p]^\T\in\RR^{p-q}$ of unit norm,
   $$ \sqrt n \t^\T\bepsilon/\gamma \to  \mathcal{N}(0,1)$$ in distribution, for some deterministic $\theta$ and $\gamma$.

    \item \label{enum:proof_step_7} Demonstrating 
    $$\kappa\simeq\frac{1}{n}\tr\bSigma(\lambda\I_p+\theta\bSigma).$$

    \item \label{enum:proof_step_8} Establishing asymptotic equations on $\eta,\theta,\gamma$ and $\phi$ from the results of the above steps, which characterize the limiting behavior of the solution $\hat\bbeta$.

\end{enumerate}

In the following, we present the detailed proof of \Cref{theo:main}.

\subsubsection{Detailed proof of \Cref{theo:main}}

We start by establishing the following bound on the difference between $\hat\bbeta$ and its leave-one-out version $\hat\bbeta_{-i}$. 

\begin{Lemma}[Bound on $\| \hat \bbeta- \hat \bbeta_{-i} \|$]\label{lem:bound_on_diff_hat_bbeta}
For $\hat \bbeta \in \RR^p$ the unique solution to \eqref{eq:opt-origin-reg} and $\hat \bbeta_{-i} \in \RR^p$ the associated leave-one-out solution as defined in \eqref{eq:lambda-beta-loo} that is independent of $\x_i$, we have,
\begin{equation}
    \left\| \hat \bbeta- \hat \bbeta_{-i}\right\|  = O(p^{-1/2}),
\end{equation}
in probability as $n,p \to \infty$.
\end{Lemma}

\begin{proof}[Proof of \Cref{lem:bound_on_diff_hat_bbeta}]
We define first 
\begin{align*}
    R_j=&\hat\bbeta^\T\x_j,\\
    c_j=&-\ell'(\hat\bbeta^\T\x_j,y_j),
\end{align*}
for all $j\in\{1,\ldots,n\}$, and their leave-one-out versions
\begin{align*}
    R_{j(-i)}=&\hat\bbeta_{-i}^\T\x_j,\\
    c_{j(-i)}=&-\ell'(\hat\bbeta_{-i}^\T\x_j,y_j),
\end{align*}
for all $j\in\{1,\ldots,n\}$.

According to \Cref{ass:loss}, $\ell(\cdot,y)$ is continuously differentiable and has bounded second derivative except on a finite points of set, therefore there exists universal constant $K$ such that $\vert\ell'(t_1)-\ell'(t_2)\vert\leq K \vert t_1-t_2\vert$. 
As a result, for every pair of $i,j\in\{1,\ldots,n\}$, there exists a finite positive (due to the convexity of $\ell(\cdot,y)$) value $a_{j(-i)}$ such that
\begin{equation}\label{eq:def_a_j-i}
    c_j-c_{j(-i)}=-a_{j(-i)}\left(\hat\bbeta^\T\x_j-\hat\bbeta_{-i}^\T\x_j\right).
\end{equation}

Taking \eqref{eq:lambda-beta}$-$\eqref{eq:lambda-beta-loo}, we obtain
\begin{align*}
\lambda \hat \bbeta-\lambda \hat \bbeta_{-i}=&\frac1n c_i\x_i+\frac1n \sum_{j\neq i}  (c_j-c_{j(-i)})\x_j\\
=&\frac1n c_i\x_i-\left(\frac1n \sum_{j\neq i}a_{j(-i)}\x_j\x_j^\T\right)(\hat \bbeta- \hat \bbeta_{-i}).
\end{align*}

Therefore
\begin{equation*}
\left(\lambda\I_p+\frac1n \sum_{j\neq i}a_{j(-i)}\x_j\x_j^\T\right)( \hat \bbeta- \hat \bbeta_{-i})=\frac1n c_i\x_i,
\end{equation*}
and
\begin{equation}\label{eq:hat_beta-beta_-i}
 \hat \bbeta- \hat \bbeta_{-i}=\left(\lambda\I_p+\frac1n \sum_{j\neq i}a_{j(-i)}\x_j\x_j^\T\right)^{-1}\frac1n c_i\x_i.
\end{equation}

Since $\frac1n \sum_{j\neq i}a_{j(-i)}\x_j\x_j^\T$ is non-negative definite, all eigenvalues of $\left(\lambda\I_p+\frac1n \sum_{j\neq i}a_{j(-i)}\x_j\x_j^\T\right)$ are greater than or equal to $\lambda$, so that
\begin{equation*}
 \left\| \hat \bbeta- \hat \bbeta_{-i}\right\| \leq\frac{1}{\lambda n} c_i \| \x_i \| =O(p^{-\frac{1}{2}}),
\end{equation*}
where we use the fact that $\hat \bbeta$ has bounded norm as $n,p \to \infty$, which is easy to check for $\lambda>0$, to prove the boundedness of $c_i$.
This concludes the proof of \Cref{lem:bound_on_diff_hat_bbeta}.
\end{proof}

As a consequence of the proof of \Cref{lem:bound_on_diff_hat_bbeta}, we have, by \eqref{eq:def_a_j-i} that
\begin{align*}
    a_{j(-i)} = \ell''(\hat\bbeta_{-i}^\T\x_j,y)+O(p^{-\frac{1}{2}})
\end{align*}
where $\ell''(t,y) = \frac{\partial\ell'(t,y)}{\partial t}$.
This equation is however only valid when $\ell''(\cdot,y)$ exists, while in \Cref{ass:loss}, we allow $\ell''(\cdot,y)$ to not exist on a finite set of points. Actually, as the number $l$ of $R_i$ falling on this set of point is also finite, we can take $\ell''(t,y)=\frac{\partial_{-}\ell'(t,y)}{\partial_{-} t}$ without any asymptotic impact on our results. To see that $l$ is finite, we use \eqref{eq:lambda-beta} to establish $n$ linearly independent equations on $R_1,\ldots,R_n$:
\begin{equation*}
    \lambda R_i= -\frac1n \sum_{j=1}^n \ell'(R_j,y_j)\x_i^\T\x_j,\quad \forall i\in\{1,\ldots,n\}.
\end{equation*}
As $\x_1,\ldots,\x_n$ are i.i.d.\@ feature vectors drawn from the high-dimensional LFMM, the number of $R_i$ having the same value is finite with probability $1$ at large $p$.

With a slight abuse of notation, let us set from now on 
\begin{align*}
    a_{j(-i)}  = \ell''(\hat\bbeta_{-i}^\T\x_j,y).
\end{align*}
Then, \eqref{eq:hat_beta-beta_-i} writes
\begin{equation}
\label{eq:diff beta}
    \hat \bbeta- \hat \bbeta_{-i} = \frac1n c_i\G_{-i}^{-1} \x_i+O_{\Vert\cdot\Vert}(p^{-1}),
\end{equation}
where
\begin{equation*}
    \G_{-i}=\lambda\I_p+\frac1n \sum_{j\neq i}a_{j(-i)}\x_j\x_j^\T.
\end{equation*}
Notice that $\G_{(-i)}$ is independent of $(\x_i,y_i)$. For $R_i = \hat \bbeta^\T \x_i$ and $r_i=R_{i(-i)}=\hat\bbeta_{-i}^\T\x_i$ , we have, by the law of large numbers on $\x_i$ (recall that $\E[\x_i \x_i^\T] = \bmu \bmu^\T + \bSigma$), that
\begin{align}
    R_i- r_i = &\frac1n c_i \x_i^\T\G_{-i}^{-1}\x_i+O(p^{-\frac{1}{2}})=\frac1n c_i \e_i^\T\V^\T\G_{-i}^{-1}\V\e_i+O(p^{-\frac{1}{2}})\nonumber\\
    =&\frac1n c_i \tr\left(\G_{-i}^{-1}\bSigma\right)+O(p^{-\frac{1}{2}})\label{eq:diff R r},
\end{align}
where the last equality is a classical concentration result as a consequence of the independent entries in $\e_i$. It is understandable that $r_i$ is significantly different from $R_i$, as the latter is the predicted score of $\hat\bbeta$ on one of its training sample and the former the predicted score of $\hat\bbeta_{-i}$ on a test sample independent of its training set. 

Let us define 
\begin{equation*}
    \kappa_i =\frac1n\tr\left(\G_{-i}^{-1}\bSigma\right),\quad \kappa= \frac1n\tr\left(\G^{-1}\bSigma\right),
\end{equation*}
where
\begin{equation*}
    \G= \lambda\I_p+\frac1n \sum_{i=1}^na_{i}\x_i\x_i^\T,
\end{equation*}
with $a_j=\ell''(\hat\bbeta^\T\x_j,y)$. It follow from \eqref{eq:diff beta} that $\hat \bbeta^\T\x_j- \hat \bbeta_{-i}^\T\x_j=O(p^{-\frac{1}{2}})$, therefore $a_j=a_{j(-i)}+O(p^{-\frac{1}{2}})$. It is then easy to check that
\begin{equation*}
    \kappa_i =\kappa+O(p^{-\frac{1}{2}}).
\end{equation*}
We can thus rewrite \eqref{eq:diff R r} as
\begin{equation}
\label{eq:diff R r with kappa}
    R_i- r_i =\kappa c_i +O(p^{-\frac{1}{2}}).
\end{equation}
We arrive thus at the end of Step~1. At this point, we do not have access to the statistical behavior of $\kappa$, only the fact that it is independent of the data index $i$.

\bigskip

Recall
\begin{equation}
    h_{\kappa}(t,y) = (\prox_{\kappa, \ell(\cdot,y)} (t)-t)/\kappa,
\end{equation}
we obtain from \eqref{eq:diff R r with kappa}
\begin{equation*}
    c_i = h_{\kappa}(r_i,y_i) +O(p^{-\frac{1}{2}}).
\end{equation*}
It then follow from the above equation and \eqref{eq:lambda-beta} that
\begin{equation}
    \lambda\hat\bbeta = \frac1n \sum_{i=1}^nh_{\kappa}(r_i,y)\x_i+O_{\Vert\cdot\Vert}(p^{-\frac{1}{2}}).
\end{equation}
Set from now on 
\begin{equation*}
    c_i = h_{\kappa}(r_i,y_i),
\end{equation*}
and rewrite \eqref{eq:lambda-beta} as
\begin{equation}
\label{eq:lambda-beta with c}
    \lambda\hat\bbeta = \frac1n \sum_{i=1}^nc_i\x_i+O_{\Vert\cdot\Vert}(p^{-\frac{1}{2}}).
\end{equation}
We arrive thus at the end of Step~\ref{enum:proof_step_3}.

\bigskip

We will now demonstrate
\begin{equation}
\label{eq:convergence sum yc}
    \frac1n \sum_{i=1}^ny_ic_i=\E[y_ic_i]+O(p^{-\frac{1}{4}})
\end{equation}
by showing the variance of $\frac1n \sum_{i=1}^ny_ic_i$ is of order $O(p^{-\frac{1}{2}})$.

To do so, we need to introduce the definition of leave-two-out solution $\hat \bbeta_{-ij}$ obtained by removing not one but two training samples $(\x_i,y_i)$ and $(\x_j,y_j)$. The subscript $-ij$ is understood similarly to the  subscript $-i$, but associated with the statistical objects dependent of $\hat\bbeta_{-ij}$. 

Similarly to \eqref{eq:diff beta}, we have
\begin{equation}
\label{eq:diff beta-i}
    \hat \bbeta_{-i}- \hat \bbeta_{-ij} = \frac1n c_{(-i)j}\G_{-ij}^{-1} \x_j+O_{\Vert\cdot\Vert}(p^{-1}),
\end{equation}
where 
\begin{equation*}
    c_{(-i)j}= h_{\kappa}(r_{(-i)j},y), \quad r_{(-i)j}=\hat\bbeta_{-ij}^\T\x_j.
\end{equation*}
Multiplying \eqref{eq:diff beta-i} with $\x_i^\T$ from the left side, we get
\begin{equation*}
    r_{i}-r_{(-j)i}=\frac1n c_{(-i)j}\x_i^\T\G_{-ij}^{-1} \x_j+O(p^{-\frac{1}{2}})=O(p^{-\frac{1}{2}}).
\end{equation*}
Then for $i\neq j$, we observe from the above equation that
\begin{align*}
    \E[y_ic_iy_jc_j]&=\E[y_iy_jh_{\kappa}(r_i,y_i)h_{\kappa}(r_j,y_j)]\\
    &=\E[y_iy_jh_{\kappa}(r_{(-j)i},y_i)h_{\kappa}(r_{(-i)j},y_j)]+O(p^{-\frac{1}{2}}).
\end{align*}
Note importantly that, conditioned on $\hat\bbeta_{-ij}$, $r_{(-j)i}$ and $r_{(-i)j}$ are independent. We have thus
\begin{align*}
    \E[y_iy_jh_{\kappa}(r_{(-j)i},y_i)h_{\kappa}(r_{(-i)j},y_j)]&=\E\left[\E[y_iy_jh_{\kappa}(r_{(-j)i},y_i)h_{\kappa}(r_{(-i)j},y_j)\vert\hat\bbeta_{-ij} ]\right]\\
    &=\E\left[\E[y_ih_{\kappa}(r_{(-j)i},y_i)\vert\hat\bbeta_{-ij}]\E[y_jh_{\kappa}(r_{(-i)j},y_j)\vert\hat\bbeta_{-ij} ]\right]\\
    &=\E[y_ih_{\kappa}(r_{(-j)i},y_i)]\E[y_jh_{\kappa}(r_{(-i)j},y_j)].
\end{align*}
Since $h_{\kappa}(r_{(-i)j},y_j)=h_{\kappa}(r_{j},y_j)+O(p^{-\frac{1}{2}})$, we get from the above equation and the one before that 
\begin{equation*}
\E[y_ic_iy_jc_j]=\E[y_ic_i]\E[y_jc_j]+O(p^{-\frac{1}{2}}).
\end{equation*}
It follow directly that
\begin{equation}
\label{eq:var sum yc}
    \var\left[\frac1n \sum_{i=1}^ny_ic_i\right]=O(p^{-\frac{1}{2}}).
\end{equation}
Therefore
\begin{equation}
\label{eq:convergence eta}
\frac1n \sum_{i=1}^ny_ic_i= \eta+O(p^{-\frac{1}{4}}),\quad \text{with}\quad \eta\equiv \E[y_ic_i].
\end{equation}
We prove thus \eqref{eq:convergence eta}, which brings us to  the end of Step~\ref{enum:proof_step_4}.

\bigskip

By the same reasoning, we obtain also
\begin{equation}
 \frac1n \sum_{i=1}^nc_i[\e_i]_k=\phi_k+O(p^{-\frac{1}{4}}),\quad \text{with}\quad \phi_k\equiv \E[c_i[\e_i]_k].
 \label{eq:convergence phi}   
\end{equation}

To see that $$\phi_k\simeq 0,\forall k\in\{q+1,\ldots,p\},$$ we define first the leave-one-variable-out classifier $\hat\bbeta^{-k}$ as the solution to \eqref{eq:opt-origin-reg} on a training set generated from a slightly differently distribution than $\mathcal{D}_{(\x,y)}$, with $e_k$ constantly equal to zero. The superscript $-k$ is understood similarly to the subscript $-i$ in the statistical objects dependent of $\hat\bbeta_{-i}$, with their leave-one-variable-out version obtained by replacing $\hat\bbeta_{-i}$ with $\hat\bbeta^{-k}$. Similarly to the leave-one-out reasoning with respect to the data samples, the same asymptotic arguments can be applied to control the difference between $\hat\bbeta$ and $\hat\bbeta^{-k}$. In the same spirit as \eqref{eq:diff beta}, we have
\begin{equation}
\label{eq:diff beta loo-var}
    \hat\bbeta - \hat\bbeta^{-k}=\frac{1}{n}\G^{-k}\e_{[k]}+O_{\Vert\cdot\Vert}(p^{-1})
\end{equation}
where $\e_{[k]}\in\RR^n$ is a vector with its $i$-th element being $[\e_i]_k$, and $\G^{-k}$ a matrix of bounded norm with high probability and independent of $[\e_i]_k$.

Note first from \eqref{eq:lambda-beta with c} that 
\begin{align}
\label{eq:beta v noise}
\lambda\V_{\rm noise}\hat\bbeta&=\frac{1}{n}\sum_{i=1}^nc_i\V_{\rm noise}\x_i+O_{\Vert\cdot\Vert}(p^{-\frac{1}{2}})\nonumber\\
&=\frac{1}{n}\sum_{i=1}^nc_i\V_{\rm noise}\V_{\rm noise}^\T\tilde\e_i+O_{\Vert\cdot\Vert}(p^{-\frac{1}{2}})
\end{align}
where we have $\V_{\rm noise}\x_i=\V_{\rm noise}\V_{\rm noise}^\T\tilde\e_i$ according to the orthogonality between the signal subspace and the noise subspace stated in Item~(ii) of \Cref{ass:LFMM}. As the eigenvalues of $\V_{\rm noise}\V_{\rm noise}^\T$ are comparable to $1$ according to Item~(ii) of \Cref{ass:growth-rate}, we get
\begin{equation}
\label{eq:sum ce noise}
 \lambda\left(\V_{\rm noise}\V_{\rm noise}^\T\right)^{-1}\V_{\rm noise}\hat\bbeta=\frac{1}{n}\sum_{i=1}^nc_i\tilde\e_i+O_{\Vert\cdot\Vert}(p^{-\frac{1}{2}}).   
\end{equation}
Similarly, we have, for $\hat\bbeta^{-k}$, that
\begin{equation}
\label{eq:sum ce noise loo-var}
\lambda\left(\V_{\rm noise}\V_{\rm noise}^\T\right)^{-1}\V_{\rm noise}\hat\bbeta^{-k}=\frac{1}{n}\sum_{i=1}^nc_i\tilde\e_i^{-k}+O_{\Vert\cdot\Vert}(p^{-\frac{1}{2}}).
\end{equation}

Combining \eqref{eq:diff beta loo-var}, \eqref{eq:sum ce noise} and \eqref{eq:sum ce noise loo-var}, we obtain that, for $k\in\{q+1,\ldots,p\}$,
\begin{align}
\label{eq:E ce noise}
\phi_k=&\lambda \left[\left(\V_{\rm noise}\V_{\rm noise}^\T\right)^{-1}\V_{\rm noise}\E[\hat\bbeta]\right]_{k-q}+O(p^{-\frac{1}{2}})\nonumber\\
=&\lambda \left[\left(\V_{\rm noise}\V_{\rm noise}^\T\right)^{-1}\V_{\rm noise}\left(\E[\hat\bbeta^{-k}]-\frac{1}{n}\E[\G^{-k}\e_{[k]}]\right)\right]_{k-q}+O(p^{-\frac{1}{2}})\nonumber\\
=&O(p^{-\frac{1}{2}}),
\end{align}
where we used the fact that
$$\left[\left(\V_{\rm noise}\V_{\rm noise}^\T\right)^{-1}\V_{\rm noise}\E[\hat\bbeta^{-k}]\right]_{k-q}=\frac{1}{n}\sum_{i=1}^n\E[c_i[\tilde\e_i^{-k}]_{k-q}]=0,$$
since $[\tilde\e_i^{-k}]_{k-q}=[\e_i^{-k}]_k=0$ for all $i\in\{1,\ldots,n\}$ according to the definition of the leave-one-variable-out classifier $\hat\bbeta^{-k}$.

We arrive thus at the end of Step~\ref{enum:proof_step_5}.

\bigskip

Recall from \eqref{eq:E ce noise} that $\phi_k=O(p^{-\frac{1}{2}})$ for $k\in\{q+1,\ldots,p\}$, we have thus from \eqref{eq:convergence phi} that
\begin{equation*}
    \frac1n \sum_{i=1}^nc_i[\e_i]_k=O(p^{-\frac{1}{4}}),\quad \forall k\in\{q+1,\ldots,p\}.
\end{equation*} 
It follows then from \eqref{eq:beta v noise} that 
\begin{equation}
\label{eq:beta v}
\hat\bbeta^\T\v_k=O(p^{-\frac{1}{4}}),\quad \forall k\in\{q+1,\ldots,p\}.
\end{equation}

We observe then, for $k\in\{q+1,\ldots,p\}$,
\begin{align}
\frac1n \sum_{i=1}^nc_i[\e_i]_k=& \frac1n \sum_{i=1}^nh_{\kappa}(r_i,y_i)[\e_i]_k=\frac1n \sum_{i=1}^nh_{\kappa}\left(\sum_{k=1}^p(\hat\bbeta^\T\v_k)[\e_i]_k,y_i\right)[\e_i]_k\nonumber\\  
=&\frac1n \sum_{i=1}^nh_{\kappa}\left(\sum_{d\neq k}(\hat\bbeta^\T\v_d)[\e_i]_d,y_i\right)[\e_i]_k\nonumber\\
&+\frac1n \sum_{i=1}^n h'_{\kappa}\left(\sum_{d\neq k}(\hat\bbeta^\T\v_d)[\e_i]_d,y_i\right)(\hat\bbeta^\T\v_k)[\e_i]_k^2\\
&+\frac1n \sum_{i=1}^n h''_{\kappa}\left(\sum_{d\neq k}(\hat\bbeta^\T\v_d)[\e_i]_d,y_i\right)(\hat\bbeta^\T\v_k)^2[\e_i]_k^3+O(p^{-\frac{3}{4}}),\label{eq:decomp sum ce}
\end{align}
where we denote $h''_\kappa(r,y)=\frac{\partial h'_\kappa(r,y)}{\partial r}$.

We denote the first term by 
\begin{equation*}
    \epsilon_k=\frac1n \sum_{i=1}^nh_{\kappa}\left(\sum_{d\neq k}(\hat\bbeta^\T\v_d)[\e_i]_d,y_i\right)[\e_i]_k.
\end{equation*}
Note from \eqref{eq:diff beta loo-var} that
\begin{equation}
\label{eq:approx epsilon}
    \epsilon_k=\frac1n \sum_{i=1}^nh_{\kappa}\left(\sum_{d\neq k}(\v_d^\T\hat\bbeta^{-k})[\e_i]_d,y_i\right)[\e_i]_k+O(p^{-1}).
\end{equation}
Since $\hat\bbeta^{-k}$ is, by definition, independent of all $[\e_i]_k, i\in\{1,\ldots,n\}$,  we notice that all
\[
    h_{\kappa}\left(\sum_{d\neq k}(\v_d^\T\hat\bbeta^{-k})[\e_i]_d,y_i\right)[\e_i]_k, i\in\{1,\ldots,n\},
\]
are independent conditioned on $\{\e^{[d]}\}_{d\in\{q+1,\ldots,p\}\setminus k}$.We assess the conditional mean of $\epsilon_k$ as follows
\begin{align}
\E\left[\epsilon_k\vert\right\{\e^{[d]}\}_{d\in\{q+1,\ldots,p\}\setminus k}]=&\frac1n \sum_{i=1}^nh_{\kappa}\left(\sum_{d\neq k}(\v_d^\T\hat\bbeta^{-k})[\e_i]_d,y_i\right)\E\left[[\e_i]_k\right]+O(p^{-\frac{1}{2}})\nonumber\\
=&O(p^{-1}).
\label{eq:convergen mean u}
\end{align}
Similarly we have
\begin{align*}
\var\left[\epsilon_k\vert\right\{\e^{[d]}\}_{d\in\{q+1,\ldots,p\}\setminus k}]=&\frac{1}{n^2} \sum_{i=1}^nh_{\kappa}\left(\sum_{d\neq k}(\v_d^\T\hat\bbeta^{-k})[\e_i]_d,y_i\right)^2\E\left[[\e_i]_k^2\right]\\
=&\frac{1}{n} \E\left[h_{\kappa}\left(\sum_{d\neq k}(\v_d^\T\hat\bbeta^{-k})+O(p^{-\frac{5}{4}})[\e_i]_d,y_i\right)^2\right]+O(p^{-\frac{5}{4}}), 
\end{align*} 
where the second line is obtained by a similar reasoning to \eqref{eq:convergence eta}
We remark thus that the concentrations of the conditional mean $\E\left[\epsilon_k\vert\right\{\e^{[d]}\}_{d\in\{q+1,\ldots,p\}\setminus k}]$ and variance $\var\left[\epsilon_k\vert\right\{\e^{[d]}\}_{d\in\{q+1,\ldots,p\}\setminus k}]$ around the same limits:
\begin{align}
\E\left[\epsilon_k\vert\right\{\e^{[d]}\}_{d\in\{q+1,\ldots,p\}\setminus k}]=&O(p^{-1})\nonumber\\ 
\var\left[\epsilon_k\vert\right\{\e^{[d]}\}_{d\in\{q+1,\ldots,p\}\setminus k}]=&\frac{1}{n}\gamma^2+O(p^{-\frac{5}{4}}) ,\quad\text{with}\quad\gamma^2\equiv\E[c_i^2].\label{eq:convergence gamma}
\end{align}
In summary, when conditioned on $\{\e^{[d]}\}_{d\in\{q+1,\ldots,p\}\setminus k}$, the sum of independent random variables
\[
    \frac{1}{n}h_{\kappa}\left(\sum_{d\neq k}(\v_d^\T\hat\bbeta^{-k})[\e_i]_d,y_i\right)[\e_i]_k, i\in\{1,\ldots,n\},
\]
is of mean asymptotically $0$ and variance asymptotically equal to $\gamma^2$.
Then, by the central limit theorem, we have
\begin{equation}
\label{eq:convergen epsilon}
\sqrt{n}\epsilon_k/\gamma \overset{\rm d}{\to}\mathcal{N}(0,1),\quad \forall k\in\{q+1,\ldots,p\},
\end{equation}
in distribution as $n,p \to \infty$ at the same pace.

Now we wish to show that 
\begin{equation}
\label{eq:convergen u}
\sqrt{n}\sum_{k=q+1}^pt_k\epsilon_k/\gamma \overset{\rm d}{\to}\mathcal{N}(0,1)
\end{equation}
for any deterministic vector $\t=[t_{q+1},\ldots,t_p]^\T \in \RR^{p - q}$ of unit norm. 
To this end, we introduce the leave-two-variables-out solution $\hat \bbeta^{-kd}$ obtained similarly to $\hat \bbeta^{-k}$ but with both $e_k,e_d$ constantly set to $0$. The superscript $-kd$ is understood similarly to the  superscript $-k$, but associated with the statistical objects dependent of $\hat\bbeta_{-kd}$. It is easy to see that
\begin{equation*}
\E\left[\sqrt{n}\sum_{k=q+1}^pt_k\epsilon_k \right] = \sum_{k=q+1}^pt_k\E\left[\epsilon_k \right] =O(p^{-\frac{1}{2}}).
\end{equation*}
To approximate the variance of $\sqrt{n}\sum_{k=q+1}^pt_k\epsilon_k$, let us define first
\begin{equation}
\label{eq:epsilon loo-var}
\epsilon_{k_1}^{-k_2}=\frac1n \sum_{i=1}^nh_{\kappa}\left(\sum_{d\neq k_1,k_2}(\v_d^\T\hat\bbeta^{-k_2})[\e_i]_d,y_i\right)[\e_i]_{k_1},
\end{equation}
for which we have
\begin{equation}
\label{eq:approx epsilon loo-var}
    \epsilon_{k_1}=\epsilon_{k_1}^{-k_2}+(p^{-\frac{3}{4}})
\end{equation}
from \eqref{eq:diff beta loo-var} and \eqref{eq:beta v}.
Similarly to \eqref{eq:approx epsilon}, we have 
\begin{align*}
\epsilon_{k_1}^{-k_2}=\frac1n \sum_{i=1}^nh_{\kappa}\left(\sum_{d\neq k_1,k_2}(\v_d^\T\hat\bbeta^{-k_1k_2})[\e_i]_d,y_i\right)[\e_i]_{k_1}+O(p^{-1}).
\end{align*}
Therefore
\begin{align*}
    \epsilon_{k_1}=\frac1n \sum_{i=1}^nh_{\kappa}\left(\sum_{d\neq k_1,k_2}(\v_d^\T\hat\bbeta^{-k_1k_2})[\e_i]_d,y_i\right)[\e_i]_{k_1}+(p^{-\frac{3}{4}})\\
    \epsilon_{k_2}=\frac1n \sum_{i=1}^nh_{\kappa}\left(\sum_{d\neq k_1,k_2}(\v_d^\T\hat\bbeta^{-k_1k_2})[\e_i]_d,y_i\right)[\e_i]_{k_2}+(p^{-\frac{3}{4}}),
\end{align*}
where we notice that the random variable $\sum_{i=1}^nh_{\kappa}\left(\sum_{d\neq k_1,k_2}(\v_d^\T\hat\bbeta^{-k_1k_2})[\e_i]_d,y_i\right)[\e_i]_{k_1}$ is independent of $\sum_{i=1}^nh_{\kappa}\left(\sum_{d\neq k_1,k_2}(\v_d^\T\hat\bbeta^{-k_1k_2})[\e_i]_d,y_i\right)[\e_i]_{k_2}$.
We obtain thus 
\begin{align*}
    \var\left[\sqrt{n}\sum_{k=q+1}^pt_k\epsilon_k \right] &= n\sum_{k_1,k_2=q+1}^pt_k^2\E\left[\epsilon_{k_1}\epsilon_{k_2} \right]=n\sum_{k=q+1}^pt_k^2\E\left[\epsilon_{k}^2\right]+n\sum_{k_1\neq k_2=q+1}^pt_k^2\E\left[\epsilon_{k_1}\epsilon_{k_2}\right]\\
    &= \gamma^2+O(p^{-\frac{1}{4}}).
\end{align*}
To obtain \eqref{eq:convergen u}, it suffices now to demonstrate the asymptotic mutual independence of $\epsilon_{q+1},\ldots,\epsilon_q$ by showing that $\epsilon_k$ is asymptotically independent of $\{\epsilon_d\}_{d\neq k=q+1}$ for all $k\in\{q+1,\ldots,p\}$. Let $k=q+1$ without the loss of generality, observe from \eqref{eq:approx epsilon} and \eqref{eq:convergen epsilon} that
\begin{align*}
    \epsilon_{q+1}\simeq \frac1n \sum_{i=1}^nh_{\kappa}\left(\sum_{d\neq q+1}^p(\v_d^\T\hat\bbeta^{-(q+1)})[\e_i]_d,y_i\right)[\e_i]_{q+1},
\end{align*}
and recall from \eqref{eq:approx epsilon loo-var} that
\begin{align*}
    \begin{bmatrix}
        \epsilon_{q+2}\\
        \epsilon_{q+3}\\
        \vdots\\
        \epsilon_{p}
    \end{bmatrix}\simeq\begin{bmatrix}
        \epsilon_{q+2}^{-(q+1)}\\
        \epsilon_{q+3}^{-(q+1)}\\
        \vdots\\
        \epsilon_{p}^{-(q+1)}
    \end{bmatrix}.
\end{align*}
It is easy to see from \eqref{eq:epsilon loo-var} that $\epsilon_{q+2}^{-(q+1)},\ldots,\epsilon_{p}^{-(q+1)}$ is independent of $\e_{q+1}$. 
Therefore, we have that  $\frac1n \sum_{i=1}^nh_{\kappa}\left(\sum_{d\neq q+1}^p(\v_d^\T\hat\bbeta^{-(q+1)})[\e_i]_d,y_i\right)[\e_i]_{q+1}$ is independent of $\epsilon_{q+2}^{-(q+1)},\ldots,\epsilon_{p}^{-(q+1)}$. We obtain thus \eqref{eq:convergen u}.

Now we turn to the second term of \eqref{eq:decomp sum ce}. In a similar manner to \eqref{eq:convergence eta}, we have 
\begin{align*}
\frac1n \sum_{i=1}^n h'_{\kappa}\left(\sum_{d\neq k}(\hat\bbeta^\T\v_d)[\e_i]_d,y_i\right)[\e_i]_k^2=&\frac1n \sum_{i=1}^n h'_{\kappa}\left(\sum_{d\neq k}(\v_d^\T\hat\bbeta^{-k})[\e_i]_d,y_i\right)[\e_i]_k^2+O(p^{-1})\\
=&\E\left[h'_{\kappa}\left(\sum_{d\neq k}(\v_d^\T\hat\bbeta^{-k})[\e_i]_d,y_i\right)\right]\E\left[[\e_i]_k^2\right]+O(p^{-\frac{1}{4}}).
\end{align*}
Consequently,
\begin{equation}
\label{eq:convergence theta}
 \frac1n \sum_{i=1}^n h'_{\kappa}\left(\sum_{d\neq k}(\hat\bbeta^\T\v_d)[\e_i]_d,y_i\right)[\e_i]_k^2\hat\bbeta^\T\v_k=\left(-\theta+O(p^{-\frac{1}{4}})\right)\hat\bbeta^\T\v_k,\text{ with }\theta\equiv-\E[h_\kappa'(r_i,y_i)].
\end{equation}

To control the third term of \eqref{eq:decomp sum ce}, it suffices to prove that its second moment is of $o(p^{-\frac{1}{2}})$ by using the concentration arguments with the leave-one-variable-out manipulation as before, and the fact that $\hat\bbeta^\T\v_k=O(p^{-\frac{1}{4}})$ for $k\in\{q+1,\ldots,p\}$.

We arrive thus at the end of Step~\ref{enum:proof_step_6}.

\bigskip

Rewrite now \eqref{eq:lambda-beta with c} as
\begin{equation}
\label{eq:lambda-beta with c 2}
    \lambda\hat\bbeta = \frac1n \sum_{i=1}^nc_iy_i\bmu+\frac1n \sum_{i=1}^nc_i\V\e_i+O_{\Vert\cdot\Vert}(p^{-\frac{1}{2}}).
\end{equation}
Summarizing \eqref{eq:convergence eta}, \eqref{eq:convergence phi}, \eqref{eq:decomp sum ce}, \eqref{eq:convergen epsilon}, nd \eqref{eq:convergence theta}, we obtain
\begin{align*}
\lambda\hat\bbeta \simeq \eta\bmu+\sum_{k=1}^q\phi_k\v_k+\theta\V_{\rm noise}\V_{\rm noise}^\T\hat\bbeta+\V_{\rm noise}\bepsilon,
\end{align*}
with $\bepsilon=[\epsilon_{q+1},\ldots,\epsilon_p]^\T$.
Therefore
\begin{equation}
\label{eq:limit beta}
\hat\bbeta\simeq\left(\lambda\I_p+\theta\bSigma\right)^{-1}\left(\eta\bmu +\sum_{k=1}^{q} \omega_k\v_k +\V_{\rm noise}\bepsilon\right),
\end{equation}
where $\omega_k\equiv \phi_k+\theta\E[\hat\bbeta]^\T\v_k$, and 
\begin{equation*}
\sqrt{n}\t^\T\bepsilon/\gamma \overset{\rm d}{\to}\mathcal{N}(0,1)
\end{equation*}
for any deterministic vector $\t=[t_{q+1},\ldots,t_p]^\T \in \RR^{p -q}$ of unit norm according to \eqref{eq:convergen u}.
Similarly, we have, for the leave-one-out solution, that
\begin{equation}
\label{eq:limit beta loo}
\hat\bbeta_{-i}\simeq\left(\lambda\I_p+\theta\bSigma\right)^{-1}\left(\eta\bmu +\sum_{k=1}^{q} \omega_k\v_k +\V_{\rm noise}\bepsilon_{-i}\right)
\end{equation}
where $\bepsilon_{-i}=[\epsilon_{q+1(-i)},\ldots,\epsilon_{p(-i)}]^\T$ with
\begin{equation*}
    \epsilon_{k(-i)}=\frac1n \sum_{j\neq i}h_{\kappa}\left(\sum_{d\neq k}(\hat\bbeta_{-i}^\T\v_d)[\e_j]_d,y_j\right)[\e_j]_k.
\end{equation*}
We get from \eqref{eq:diff beta} that 
\begin{equation*}
    \epsilon_{k}-\epsilon_{k(-i)}\simeq \frac1n h_{\kappa}\left(\sum_{d\neq k}(\hat\bbeta^\T\v_d)[\e_i]_d,y_i\right)[\e_i]_k.
\end{equation*}
Hence
\begin{equation*}
    \hat\bbeta-\hat\bbeta_{-i}\simeq\left(\lambda\I_p+\theta\bSigma\right)^{-1}\frac{1}{n}\V_{\rm noise}c_i\tilde\e_i,
\end{equation*}
leading to
\begin{align*}
     \x_i^\T\hat\bbeta-\x_i^\T\hat\bbeta_{-i}\simeq&(y_i\bmu+\V\e_i)^\T\left(\lambda\I_p+\theta\bSigma\right)^{-1}\frac{1}{n}\V_{\rm noise}c_i\tilde\e_i\\
     \simeq&\frac{c_i}{n}\tr\bSigma(\lambda\I_p+\theta\bSigma).
\end{align*}
Comparing the above equation with \eqref{eq:diff R r with kappa}, we observe
\begin{equation*}
    \kappa\simeq\frac{1}{n}\tr\bSigma(\lambda\I_p+\theta\bSigma).
\end{equation*}
We are now at the end of Step~\ref{enum:proof_step_7}.
\bigskip

It is easy to see from \eqref{eq:limit beta loo} that
\begin{equation*}
    r_i=\hat\bbeta_{-i}^\T\x_i\simeq y_im+\sum_{k=1}^p\psi_k[\e_i]_k+\sigma\tilde e,
\end{equation*}
where $\tilde e$ a random variable independent of $[\e_i]_1,\ldots,[\e_i]_q$ with
$$\tilde e\overset{\rm d}{\to}\mathcal{N}(0,1),$$
and
\begin{align*}
   m = \bmu^\T\Q\bxi,\quad \textstyle  \sigma^2=\frac{\gamma^2}p \tr \left(\Q\bSigma \right)^2,\quad \psi_k=\v_k^\T\Q\bxi,\quad \forall k\in\{1,\ldots,q\},
\end{align*}
with
\begin{equation*}
    \Q=\left( \lambda\I_p+\theta\bSigma\right)^{-1},\quad \bxi=\eta\bmu+\sum_{k=1}^q\omega_k\v_k.
\end{equation*}

We obtain thus the system of equations in \eqref{eq:theta eta gamma} from \eqref{eq:convergence eta}, \eqref{eq:convergence phi}, \eqref{eq:convergence gamma} and \eqref{eq:convergence theta}, which gives access to the values of $\theta,\eta,\gamma,\omega_1,\ldots,\omega_q$.

Set 
\begin{equation}
\label{eq:def tilde beta}
    \tilde\bbeta\equiv \left(\lambda\I_p+\theta\bSigma\right)^{-1}\left(\eta\bmu +\sum_{k=1}^{q} \omega_k\v_k +\gamma\bSigma^{\frac{1}{2}}\u\right)
\end{equation}
where $\u \sim\mathcal{N}(\zeros_p,\I_p/n)$.
We obtain \eqref{eq:convergence test score} from \eqref{eq:def tilde beta} and \eqref{eq:limit beta} by a simple application of CLT, and \eqref{eq:convergence training score} from \eqref{eq:def tilde beta}, \eqref{eq:limit beta loo} and \eqref{eq:diff R r with kappa}.

This concludes the proof of \Cref{theo:main}.

\subsection{Proofs of Corollaries}
\label{subsec:proof_of_coro}

\subsubsection{Proof of \Cref{cor:performance}}
\label{subsubsec:proof_of_cor_performance}

Recall from \Cref{theo:main} that 
\begin{equation*}
   \tilde \bbeta = \left( \lambda\I_p+\theta\bSigma\right)^{-1}\left(\eta\bmu +\sum_{k=1}^{q} \omega_k\v_k +\gamma\bSigma^{\frac{1}{2}}\u\right), 
\end{equation*}
for Gaussian vector $\u\sim\mathcal{N}(\zeros_p,\I_p/n)$ independent of $\{ (\x_i, y_i )\}_{i=1}^n$.
For $(\x,y)\sim\mathcal{D}_{(\x,y)}$ independent of $\tilde\bbeta$, we recall from \Cref{def:linear_factor} that
\begin{equation*}
    \x=y\bmu+\V\e,
\end{equation*}
with $\e=[e_1,\ldots,e_p]^\T$. Then, let $\V_{\rm noise}=[\v_{q+1},\ldots,\v_p]$ and $\tilde\e=[e_{q+1},\ldots,e_p]^\T$, we have
\begin{equation*}
    \tilde \bbeta^\T\x=ym+\sum_{k=1}^q\psi_ke_k+\bxi^\T\Q\V_{\rm noise}\tilde\e+\gamma\u^\T\bSigma^{\frac{1}{2}}\Q\V_{\rm noise}\tilde\e,
\end{equation*}
with $m,\psi_1,\ldots,\psi_q$ as given in \eqref{eq: m sigma},
and $\Q,\bxi$ as in \eqref{eq:bxi Q}.

Note importantly that 
\begin{equation*}
    \bxi^\T\Q\v_k=0,\forall k\in\{q+1,\ldots,p\}
\end{equation*}
due to the orthogonality of ${\rm Span}\{\v_1,\ldots,\v_q\}$ to ${\rm Span}\{\v_{q+1},\ldots,\v_p\}$ stated in Item(ii)~of~\Cref{ass:LFMM}

As $\u,\tilde\e$ are independent random vectors of independent entries, we have, by CLT, that
\begin{align*}
\frac{\gamma\u^\T\bSigma^{\frac{1}{2}}\Q\V_{\rm noise}\tilde\e}{\sqrt{\var\left[\gamma\u^\T\bSigma^{\frac{1}{2}}\Q\V_{\rm noise}\tilde\e\right]}}&\overset{\rm d}{\to}\mathcal{N}(0,1).
\end{align*}
Remark also that
\begin{align*}
\var\left[\gamma\u^\T\bSigma^{\frac{1}{2}}\Q\V_{\rm noise}\tilde\e\right]=\frac{\gamma^2}{n}\tr\left(\bSigma\Q\V_{\rm noise}\V_{\rm noise}^\T\Q\right)\simeq\frac{\gamma^2}{n}\tr\left(\bSigma\Q\right)^2.
\end{align*}
We thus obtain \eqref{eq:convergence tildebeta x} in \Cref{cor:performance}.

From \eqref{eq:convergence test score} in \Cref{theo:main}, we have
\begin{equation*}
    \Pr (y'\hat\bbeta^\T\x'>0\vert (\x',y')) - \Pr (y'\tilde\bbeta^\T\x'>0\vert (\x',y')) \to 0.
\end{equation*}
Taking expectation over $(\x',y')\sim\mathcal{D}_{(\x,y)}$, we get directly from the above equation that
\begin{equation*}
    \Pr (y'\hat\bbeta^\T\x'>0) - \Pr (y'\tilde\bbeta^\T\x'>0) \to 0.
\end{equation*}
It follows straightforwardly from \eqref{eq:convergence tildebeta x} that
\begin{equation*}
\Pr (y'\tilde\bbeta^\T\x'>0)-    \Pr (yr>0)  \to 0,
\end{equation*}
leading to \eqref{eq:generalization error}.

Similarly, we obtain \eqref{eq:training error} from \eqref{eq:convergence tildebeta x} and \eqref{eq:convergence training score}, which concludes the proof.

\subsubsection{Proof of \Cref{cor:condition for r}}
\label{subsubsec:proof_of_condition for r}

It is easy to see that, when $e_1,\ldots,e_q$ are normally distributed, the random variable $r$ defined in \eqref{eq:def_r} follows a Gaussian distribution $\mathcal{N}(m,\sigma^2+\sum_{k=1}^q\psi_k^2)$. For $r\sim\mathcal{N}(m,\sigma^2+\sum_{k=1}^q\psi_k^2)$ with $m,\sigma,\psi_1,\ldots,\psi_q$ as given in \eqref{eq: m sigma}, we observe that the system of equations in \eqref{eq:theta eta gamma} is invariant to the distributions of $e_1,\ldots,e_p$, thus yielding the same values of $\theta,\eta,\gamma,\omega_1,\ldots,\omega_q$, as well as the same $\kappa,m,\sigma^2,\psi_1,\ldots,\psi_q$.

Therefore, with Gaussian $e_1,\ldots,e_q$, $r$ follows a universal distribution $\mathcal{N}(m,\sigma^2+\sum_{k=1}^q\psi_k^2)$ independent of the distributions of $e_{q+1},\ldots,e_p$. 
We have also the same universality result on the distribution of $\prox_{\kappa, \ell(\cdot,y)} (r)$ as the value of $\kappa$ is also insensitive to the distributions of $e_{q+1},\ldots,e_p$.

Combining the above universal arguments on the distributions of $r$ and $\prox_{\kappa, \ell(\cdot,y)} (r)$ with \Cref{cor:performance}, we prove that the Gaussian universality of in-distribution performance in \Cref{def:gaussian_universality}
holds if $e_1,\ldots,e_q$ are Gaussian variables.

It remains to demonstrate the breakdown of Gaussian universality on in-distribution performance if $e_1,\ldots,e_q$ are non-Gaussian.

Note first $\Vert\hat\bbeta_{\ell,\lambda}\Vert=\Theta(1)$ with $\lambda=\Theta(1)$. The boundedness of $\Vert\hat\bbeta_{\ell,\lambda} \|$ in the large $n,p$ limit  is easily justified from the regularized optimization penalty \eqref{eq:opt-origin-reg} with $\lambda>0$. Recall also 
\begin{equation*}
    \lambda \hat \bbeta = -\frac1n \sum_{i=1}^n \ell'(\hat\bbeta^\T\x_i,y_i)\x_i,
\end{equation*}
from which we observe that, to ensure $\Vert\hat\bbeta_{\ell,\lambda}\Vert=o(1)$, we need $\ell'(\hat\bbeta^\T\x_i,y_i)=o(1)$. However when $\Vert\hat\bbeta_{\ell,\lambda}\Vert=o(1)$, we have $\ell'(\hat\bbeta^\T\x_i,y_i)\simeq\ell'(0,y_i)=\Theta(1)$. We thus get $\Vert\hat\bbeta_{\ell,\lambda}\Vert=\Theta(1)$ by contradiction. Consequently, we have also $\Vert\tilde\bbeta_{\ell,\lambda}\Vert=\Theta(1)$ for the high-dimensional equivalent $\tilde\bbeta$ given in \Cref{theo:main}. Since $\tilde\bbeta^\T\x$ for $\x\sim\mathcal{D}_{(\x,y)}$ independent of $\tilde\bbeta$ has asymptotically the same distribution as $r$ in \eqref{eq:def_r} according to \Cref{cor:performance}, it follows from $\Vert\tilde\bbeta_{\ell,\lambda}\Vert=\Theta(1)$ that $r=\Theta(1)$. Therefore $\eta=\E[h_\kappa(r,y)]=\Theta(1)$.

Let us reorganize the expression \eqref{eq:tilde beta} as
\begin{equation}
\label{eq:tilde beta 2}
\lambda\tilde\bbeta=\eta\bmu+\sum_{k=1}^q\phi_k\v_k-\theta\V_{\rm noise}\V_{\rm noise}^\T\tilde\bbeta+\gamma\bSigma^{\frac{1}{2}}\u,
\end{equation}
with 
$$\phi_k=\omega_k-\theta\v_k^\T\Q\bxi=\E[h_\kappa(r,y)e_k],\quad \forall k\in\{1,\ldots,q\}.$$
Recall from \eqref{eq:h} that
\begin{equation*}
    h_\kappa(t,y) = (\prox_{\kappa, \ell(\cdot,y)} (t)-t)/\kappa,
\end{equation*}
where $\prox_{\tau, f} (t)=\argmin_{a \in \RR} \left[ f(a) + \frac{1}{2\tau} (a - t)^2 \right]$ for $\tau > 0$ and convex $f\colon \RR \to \RR$. 
Due to the convexity of $\ell(\cdot,y)$ in \Cref{ass:loss}, $h_\kappa(\cdot,y)$ is a decreasing function. As $e_1,\ldots,e_q$ are standardized variables of symmetric distribution according to \Cref{def:linear_factor}, we have
\begin{equation*}
    \E\left[h_\kappa(r,y)e_k\Big\vert \{e_d\}_{d\in\{1,\ldots,q\}\setminus k},\tilde e,y\right]=\int_{-\infty}^{+\infty}h_\kappa\bigg(ym+\sigma\tilde e+\sum_{d\in\{1,\ldots,q\}\setminus k} \psi_d e_d+\psi_k e_k,y\bigg)e_kP_{e_k}(de_k),
\end{equation*}
where $P_{e_k}$ is the probability measure of $e_k$. As $e_k$ is a centered variable of symmetric probability distribution, we have
\begin{align*}
\int_{-\infty}^{+\infty}&h_\kappa\bigg(ym+\sigma\tilde e+\sum_{d\in\{1,\ldots,q\}\setminus k} \psi_d e_d+\psi_k e_k,y\bigg)e_kP_{e_k}(de_k)\\
=\int_{0}^{+\infty}&h_\kappa\bigg(ym+\sigma\tilde e+\sum_{d\in\{1,\ldots,q\}\setminus k} \psi_d e_d+\psi_k e_k,y\bigg)e_kP_{e_k}(de_k)\\
-\int_{0}^{+\infty}&h_\kappa\bigg(ym+\sigma\tilde e+\sum_{d\in\{1,\ldots,q\}\setminus k} \psi_d e_d-\psi_k e_k,y\bigg)e_kP_{e_k}(de_k).
\end{align*}
As $h_\kappa(t,y)$ decreases with $t$, there exists a positive value $a>0$ such that
\begin{equation*}
    h_\kappa\bigg(ym+\sigma\tilde e+\sum_{d\in\{1,\ldots,q\}\setminus k} \psi_d e_d+\psi_k e_k,y\bigg)-h_\kappa\bigg(ym+\sigma\tilde e+\sum_{d\in\{1,\ldots,q\}\setminus k} \psi_d e_d-\psi_k e_k,y\bigg)=-2a\psi_k e_k.
\end{equation*}
In the end, we have
\begin{equation*}
    \phi_k=\E[h_\kappa(r,y)e_k]=-\alpha_k\psi_k
\end{equation*}
where $\alpha_k>0$ for all $k\in\{1,\ldots,q\}$.

Plugging in the above expression of $\phi_k$, we rewrite \eqref{eq:tilde beta 2} as
\begin{equation*}
 \lambda\tilde\bbeta=\eta\bmu-\sum_{k=1}^q\alpha_k\psi_k\v_k-\theta\V_{\rm noise}\V_{\rm noise}^\T\tilde\bbeta+\gamma\bSigma^{\frac{1}{2}}\u.
\end{equation*}
Taking expectation at the both sides of the above equation, we get
\begin{equation*}
    \lambda\Q\bxi =\eta\bmu-\sum_{k=1}^q\alpha_k\psi_k\v_k-\theta\V_{\rm noise}\V_{\rm noise}^\T\tilde\bbeta.
\end{equation*}
Therefore
\begin{equation*}
    \lambda\V_{\rm info}^\T\Q\bxi=\lambda\bpsi =\eta\V_{\rm info}^\T\V_{\rm info}\s-\sum_{k=1}^q\V_{\rm info}^\T\V_{\rm info}{\rm diag}(\alpha_1,\ldots,\alpha_q)\bpsi,
\end{equation*}
where $\V_{\rm info}=[\v_1,\ldots,\v_q]$, $\bpsi=[\psi_1,\ldots,\psi_q]^\T$ and $\s=[s_1,\ldots,s_q]$. As $\V_{\rm info}$ and $\s$ are both deterministic with no presumed relation, we consider them to independent in some probability space. We obtain thus 
\begin{equation*}
    \bpsi=\eta\left(\lambda\I_q+\V_{\rm info}^\T\V_{\rm info}{\rm diag}(\alpha_1,\ldots,\alpha_q)\right)^{-1}\V_{\rm info}^\T\V_{\rm info}\s=\Theta(1).
\end{equation*}

Therefore $r$ is non-Gaussian unless in the case of Gaussian $e_1,\ldots,e_q$, leading to the breakdown of in-distribution performance in \Cref{def:gaussian_universality}.

\subsubsection{Proof of \Cref{cor:condition for beta}}
\label{subsubsec:proof_of_condition for beta}

As discussed in \Cref{subsubsec:proof_of_condition for r}, the system of equations in \eqref{eq:theta eta gamma}, which determines the distribution of $\tilde\bbeta$, is universal in the case of normally distributed $e_1,\ldots,e_q$. We obtain directly the Gaussian universality of classifier in \Cref{def:gaussian_universality} under the condition of Gaussian $e_1,\ldots,e_q$

Note importantly that when $\partial\ell(\hat y,y)/\partial\hat y$ is a linear function of $\hat y$ of form
\begin{equation*}
    \partial\ell(\hat y,y)/\partial\hat y=a\hat y +b(y) 
\end{equation*}
for some constant $a>0$ (due to the convexity of $\ell(\cdot,y)$) and $b(y)$ independent of $\hat y$, we have
\begin{equation*}
    \prox_{\kappa, \ell(\cdot,y)}(\hat y)=\frac{\hat y-\kappa b(y) }{1+\kappa a},
\end{equation*}
which leads to
\begin{equation*}
    h(\hat y,y)=\frac{\prox_{\kappa, \ell(\cdot,y)} (\hat y)-\hat y}{\kappa}=\frac{-a\hat y -b(y)}{1+\kappa a}.
\end{equation*}

Recall from \eqref{eq:def_r} that
\begin{equation*}
       r = ym +\sigma \tilde e+\sum_{k=1}^q \psi_k e_k.
\end{equation*}
The equations in \eqref{eq:theta eta gamma} thus become
\begin{align*}
 &\theta=-\E[\partial h_\kappa(r,y)/\partial r]=\frac{a}{1+\kappa a},\quad  \eta=\E[yh_\kappa(r,y)]=\frac{-am -\E[yb(y)]}{1+\kappa a},\quad \\
&\gamma=\sqrt{\E[h_\kappa(r,y)^2]}=\frac{\sqrt{a^2\left(m^2+\sigma^2+\sum_{k=1}^q\psi^2\right)+\E[b(y)^2]-2am\E[yb(y)]}}{1+\kappa a}\\
&\omega_k=\E[h_\kappa(r,y) e_k]+ \theta \cdot \v_k^\T\Q\bxi=\frac{-a\psi_k}{1+\kappa a}+\theta \cdot \v_k^\T\Q\bxi,
\end{align*}
which are independent of the distributions of the noise variables $e_1,\ldots,e_p$. We prove thus the Gaussian universality of classifier in \Cref{def:gaussian_universality} when $\partial\ell(\hat y,y)/\partial\hat y$ is a linear function of $\hat y$.

Conversely, when $\partial\ell(\hat y,y)/\partial\hat y$ is a nonlinear function of $\hat y$, $h_\kappa(r,y)$ is also a nonlinear function of $r$. Consequently, the values of $\theta,\eta,\gamma,\omega_1,\ldots,\omega_q$ depend on the higher-order moments of $e_1,\ldots,e_q$ besides the first two, thus leading to the breakdown of Gaussian universality on classifier in the presence of non-Gaussian $e_1,\ldots,e_q$.


\section{Experiments on Real Data}
\label{sec:real_data}

In this section, we report experimental results on Fashion-MNIST image data~\citep{xiao2017fashion} to show how the conditions of Gaussian universality provided in Corollaries~\ref{cor:condition for r}~and~\ref{cor:condition for beta} can be used to \emph{understand and predict} Gaussian universality phenomena on real data learning problems.

We have discussed two types of universality in this paper: universality on \textbf{in-distribution performance} and universality on \textbf{classifier} (see \Cref{def:gaussian_universality} for more details).
To discuss these two types of universality, we distinguish, depending on the Gaussianity of informative factors and the use of square loss, the following three scenarios: 
\begin{enumerate}
    \item \textbf{Scenario~1}: in the case of non-Gaussian informative factors and when a non-square loss is used, \emph{neither} the universality on in-distribution \emph{nor} the universality on classifier holds;
    \item \textbf{Scenario~2}: in the case of non-Gaussian informative factors and when a square loss is used, the universality on in-distribution breaks down while the universality on classifier still holds;
    \item \textbf{Scenario~3}: in the case of Gaussian informative factors and when an arbitrary (square or non-square) loss is used, \emph{both} the universality on in-distribution \emph{and} the universality on classifier hold.
\end{enumerate}

To see if these three scenarios derived under LFMM can be ``reproduced'' on realistic Fashion-MNIST image data, we conduct first a principle component analysis (PCA) on standardized Fashion-MNIST data to extract the informative factors. 
To obtain the equivalent GMM in \Cref{def:equivalent GMM} for a mixture of Fashion-MNIST data, we estimate the class mean and the class covariance for each class of Fashion-MNIST, using all available samples in that class.

Here, we consider the following two cases to illustrate the different effects of Gaussian and non-Gaussian informative factors on ERM classification: 
\begin{enumerate}
    \item \textbf{Case~1}: Classes 4\&5 of Fashion-MNIST data, as an example of \emph{non-Gaussian} informative factors; and 
    \item \textbf{Case~2}: Classes 3\&7 of Fashion-MNIST data, for which \emph{approximately Gaussian} informative factors can be observed.
\end{enumerate}
In \Cref{fig:information_factor_4_5} and \Cref{fig:information_factor_3_7} we compare, for the aforementioned two cases, the (empirical) distributions of the first two informative factors obtained from PCA.
We  observe that the informative factors of Classes 4\&5 have highly asymmetric distributions, corresponding to a strong deviation from the Gaussianity, while the distribution of informative factors in Classes 3\&7 are much closer to the form of normal density function in comparison.

\begin{figure}[htb]
\centering
\begin{subfigure}[t]{0.45\textwidth}
\centering
  \begin{tikzpicture}[font=\footnotesize]
    \renewcommand{\axisdefaulttryminticks}{4} 
    \pgfplotsset{every axis legend/.append style={cells={anchor=west},fill=white, at={(0.98,0.98)}, anchor=north east, font=\scriptsize }}
    \begin{axis}[
      ybar,
      width=.9\linewidth,
      ymin=0,
      ymax=1.8,
      bar width=3pt,
      grid=major,
      ymajorgrids=false,
      scaled ticks=true,
      ylabel={Histogram }
      ]
      \addplot+[ybar,mark=none,draw=white,fill=BLUE!60!white,area legend] coordinates{
      (-0.543138, 0.017581)(-0.486257, 0.026371)(-0.429376, 0.070323)(-0.372496, 0.123065)(-0.315615, 0.172876)(-0.258734, 0.228548)(-0.201853, 0.219758)(-0.144973, 0.257850)(-0.088092, 0.319382)(-0.031211, 0.319382)(0.025669, 0.404355)(0.082550, 0.407285)(0.139431, 0.439516)(0.196312, 0.618253)(0.253192, 0.691505)(0.310073, 0.770618)(0.366954, 0.966936)(0.423835, 1.122231)(0.480715, 1.289248)(0.537596, 1.558817)(0.594477, 1.696533)(0.651357, 1.681882)(0.708238, 1.567608)(0.765119, 1.306828)(0.822000, 0.761828)(0.878880, 0.372124)(0.935761, 0.143575)(0.992642, 0.017581)(1.049522, 0.005860)(1.106403, 0.002930)
      };
    \end{axis}
  \end{tikzpicture}
\caption{First informative factor of Class $4$}
\end{subfigure}%
~ 
\begin{subfigure}[t]{0.45\textwidth}
\centering
\begin{tikzpicture}[font=\footnotesize]
    \renewcommand{\axisdefaulttryminticks}{4} 
    \pgfplotsset{every axis legend/.append style={cells={anchor=west},fill=white, at={(0.98,0.98)}, anchor=north east, font=\scriptsize }}
    \begin{axis}[
      ybar,
      width=.9\linewidth,
      ymin=0,
      ymax=2,
      bar width=3pt,
      grid=major,
      ymajorgrids=false,
      scaled ticks=true,
      ylabel={ }
      ]
      \addplot+[ybar,mark=none,draw=white,fill=BLUE!60!white,area legend] coordinates{
      (-0.539284, 0.060609)(-0.498036, 0.359615)(-0.456789, 0.585890)(-0.415541, 0.727312)(-0.374293, 0.888937)(-0.333045, 0.897018)(-0.291797, 0.925303)(-0.250550, 0.880856)(-0.209302, 0.993993)(-0.168054, 1.078846)(-0.126806, 1.345527)(-0.085559, 1.628371)(-0.044311, 1.693021)(-0.003063, 1.818280)(0.038185, 1.648574)(0.079432, 1.543518)(0.120680, 1.507152)(0.161928, 1.276837)(0.203176, 1.288959)(0.244424, 0.993993)(0.285671, 0.840449)(0.326919, 0.501037)(0.368167, 0.351534)(0.409415, 0.206072)(0.450662, 0.117178)(0.491910, 0.036366)(0.533158, 0.032325)(0.574406, 0.012122)(0.615653, 0.000000)(0.656901, 0.004041)
      };
    \end{axis}
  \end{tikzpicture}
\caption{Second informative factor of Class $4$}
\end{subfigure}
\\
\begin{subfigure}[t]{0.45\textwidth}
\centering
\begin{tikzpicture}[font=\footnotesize]
    \renewcommand{\axisdefaulttryminticks}{4} 
    \pgfplotsset{every axis legend/.append style={cells={anchor=west},fill=white, at={(0.98,0.98)}, anchor=north east, font=\scriptsize }}
    \begin{axis}[
      ybar,
      width=.9\linewidth,
      ymin=0,
      ymax=4,
      bar width=3pt,
      grid=major,
      ymajorgrids=false,
      scaled ticks=true,
      ylabel={Histogram }
      ]
      \addplot+[ybar,mark=none,draw=white,fill=BLUE!60!white,area legend] coordinates{
      (-0.711830, 0.018068)(-0.674933, 0.198752)(-0.638037, 0.948591)(-0.601140, 2.398580)(-0.564243, 3.704022)(-0.527346, 3.591094)(-0.490449, 3.243277)(-0.453552, 2.818670)(-0.416656, 2.443751)(-0.379759, 1.964938)(-0.342862, 1.549365)(-0.305965, 1.156377)(-0.269068, 0.822112)(-0.232171, 0.587223)(-0.195275, 0.438159)(-0.158378, 0.334265)(-0.121481, 0.243923)(-0.084584, 0.180684)(-0.047687, 0.121962)(-0.010790, 0.085825)(0.026106, 0.103893)(0.063003, 0.040654)(0.099900, 0.018068)(0.136797, 0.022585)(0.173694, 0.022585)(0.210591, 0.009034)(0.247487, 0.009034)(0.284384, 0.013551)(0.321281, 0.009034)(0.358178, 0.004517)
      };
    \end{axis}
  \end{tikzpicture}
\caption{First informative factor of Class $5$}
\end{subfigure}
~ 
\begin{subfigure}[t]{0.45\textwidth}
\centering
\begin{tikzpicture}[font=\footnotesize]
    \renewcommand{\axisdefaulttryminticks}{4} 
    \pgfplotsset{every axis legend/.append style={cells={anchor=west},fill=white, at={(0.98,0.98)}, anchor=north east, font=\scriptsize }}
    \begin{axis}[
      ybar,
      width=.9\linewidth,
      ymin=0,
      ymax=3,
      bar width=3pt,
      grid=major,
      ymajorgrids=false,
      scaled ticks=true,
      ylabel={Histogram }
      ]
      \addplot+[ybar,mark=none,draw=white,fill=BLUE!60!white,area legend] coordinates{
      (-0.262735, 0.085744)(-0.227747, 0.371557)(-0.192759, 1.381428)(-0.157771, 2.496098)(-0.122784, 2.810492)(-0.087796, 2.529443)(-0.052808, 2.377009)(-0.017820, 2.034034)(0.017168, 1.895891)(0.052156, 1.805384)(0.087144, 1.857783)(0.122132, 1.633896)(0.157120, 1.233758)(0.192108, 1.190886)(0.227096, 1.024162)(0.262084, 0.981290)(0.297072, 0.647842)(0.332060, 0.528754)(0.367048, 0.428719)(0.402036, 0.314394)(0.437024, 0.295340)(0.472012, 0.185778)(0.507000, 0.100034)(0.541987, 0.142906)(0.576975, 0.052399)(0.611963, 0.057163)(0.646951, 0.071453)(0.681939, 0.028581)(0.716927, 0.014291)(0.751915, 0.004764)
      };
    \end{axis}
  \end{tikzpicture}
\caption{Second informative factor of Class $5$}
\end{subfigure}
\caption{{ Histogram of the first and second information factors of Class $4$ and $5$, estimated using all samples from the Fashion-MNIST dataset. }}
\label{fig:information_factor_4_5}
\end{figure}
%
%
%
\begin{figure}[htb]
\centering
\begin{subfigure}[t]{0.45\textwidth}
\centering
  \begin{tikzpicture}[font=\footnotesize]
    \renewcommand{\axisdefaulttryminticks}{4} 
    \pgfplotsset{every axis legend/.append style={cells={anchor=west},fill=white, at={(0.98,0.98)}, anchor=north east, font=\scriptsize }}
    \begin{axis}[
      ybar,
      width=.9\linewidth,
      ymin=0,
      ymax=2,
      bar width=3pt,
      grid=major,
      ymajorgrids=false,
      scaled ticks=true,
      ylabel={Histogram }
      ]
      \addplot+[ybar,mark=none,draw=white,fill=BLUE!60!white,area legend] coordinates{
      (-0.181625, 0.008835)(-0.143898, 0.030924)(-0.106170, 0.097189)(-0.068443, 0.159036)(-0.030716, 0.282730)(0.007012, 0.348995)(0.044739, 0.653814)(0.082466, 0.689155)(0.120194, 0.728914)(0.157921, 0.998391)(0.195649, 1.024897)(0.233376, 1.117668)(0.271103, 1.272286)(0.308831, 1.418069)(0.346558, 1.312045)(0.384285, 1.506422)(0.422013, 1.643369)(0.459740, 1.563851)(0.497468, 1.643369)(0.535195, 1.837746)(0.572922, 1.939352)(0.610650, 1.696381)(0.648377, 1.630116)(0.686104, 1.223692)(0.723832, 0.843773)(0.761559, 0.490360)(0.799286, 0.260642)(0.837014, 0.057430)(0.874741, 0.022088)(0.912469, 0.004418)
      };
    \end{axis}
  \end{tikzpicture}
\caption{First informative factor of Class $3$}
\end{subfigure}%
~ 
\begin{subfigure}[t]{0.45\textwidth}
\centering
\begin{tikzpicture}[font=\footnotesize]
    \renewcommand{\axisdefaulttryminticks}{4} 
    \pgfplotsset{every axis legend/.append style={cells={anchor=west},fill=white, at={(0.98,0.98)}, anchor=north east, font=\scriptsize }}
    \begin{axis}[
      ybar,
      width=.9\linewidth,
      ymin=0,
      ymax=2,
      bar width=3pt,
      grid=major,
      ymajorgrids=false,
      scaled ticks=true,
      ylabel={ }
      ]
      \addplot+[ybar,mark=none,draw=white,fill=BLUE!60!white,area legend] coordinates{
      (-0.601771, 0.003206)(-0.549790, 0.038476)(-0.497810, 0.141078)(-0.445829, 0.420028)(-0.393848, 0.705390)(-0.341868, 1.016403)(-0.289887, 1.183132)(-0.237906, 1.372305)(-0.185926, 1.567891)(-0.133945, 1.567891)(-0.081964, 1.776301)(-0.029984, 1.535827)(0.021997, 1.478114)(0.073978, 1.330623)(0.125958, 1.138244)(0.177939, 0.981134)(0.229920, 0.769517)(0.281900, 0.593169)(0.333881, 0.561106)(0.385862, 0.375139)(0.437842, 0.259712)(0.489823, 0.182760)(0.541804, 0.092983)(0.593784, 0.070539)(0.645765, 0.041682)(0.697746, 0.019238)(0.749726, 0.009619)(0.801707, 0.000000)(0.853688, 0.003206)(0.905668, 0.003206)
      };
    \end{axis}
  \end{tikzpicture}
\caption{Second informative factor of Class $3$}
\end{subfigure}
\\
\begin{subfigure}[t]{0.45\textwidth}
\centering
\begin{tikzpicture}[font=\footnotesize]
    \renewcommand{\axisdefaulttryminticks}{4} 
    \pgfplotsset{every axis legend/.append style={cells={anchor=west},fill=white, at={(0.98,0.98)}, anchor=north east, font=\scriptsize }}
    \begin{axis}[
      ybar,
      width=.9\linewidth,
      ymin=0,
      ymax=7,
      xmax=0,
      bar width=4pt,
      grid=major,
      ymajorgrids=false,
      scaled ticks=true,
      ylabel={Histogram }
      ]
      \addplot+[ybar,mark=none,draw=white,fill=BLUE!60!white,area legend] coordinates{
      (-0.599992, 0.023330)(-0.571416, 0.110818)(-0.542841, 0.565753)(-0.514265, 2.029711)(-0.485690, 4.275224)(-0.457114, 6.444914)(-0.428539, 6.538234)(-0.399963, 5.715852)(-0.371388, 3.954436)(-0.342812, 2.426321)(-0.314237, 1.370638)(-0.285661, 0.723230)(-0.257086, 0.291625)(-0.228510, 0.192473)(-0.199935, 0.110818)(-0.171359, 0.069990)(-0.142784, 0.040828)(-0.114208, 0.017498)(-0.085633, 0.017498)(-0.057057, 0.029163)(-0.028482, 0.005833)(0.000094, 0.023330)(0.028669, 0.011665)(0.057245, 0.000000)(0.085820, 0.000000)(0.114396, 0.000000)(0.142971, 0.000000)(0.171547, 0.000000)(0.200122, 0.000000)(0.228698, 0.005833)
      };
    \end{axis}
  \end{tikzpicture}
\caption{First informative factor of Class $7$}
\end{subfigure}
~ 
\begin{subfigure}[t]{0.45\textwidth}
\centering
\begin{tikzpicture}[font=\footnotesize]
    \renewcommand{\axisdefaulttryminticks}{4} 
    \pgfplotsset{every axis legend/.append style={cells={anchor=west},fill=white, at={(0.98,0.98)}, anchor=north east, font=\scriptsize }}
    \begin{axis}[
      ybar,
      width=.9\linewidth,
      ymin=0,
      ymax=2.5,
      bar width=3pt,
      grid=major,
      ymajorgrids=false,
      scaled ticks=true,
      ylabel={Histogram }
      ]
      \addplot+[ybar,mark=none,draw=white,fill=BLUE!60!white,area legend] coordinates{
      (-0.508044, 0.015444)(-0.464876, 0.011583)(-0.421709, 0.092663)(-0.378542, 0.162160)(-0.335374, 0.312737)(-0.292207, 0.517367)(-0.249040, 0.679527)(-0.205872, 1.034734)(-0.162705, 1.378358)(-0.119538, 1.806924)(-0.076371, 1.961361)(-0.033203, 2.007693)(0.009964, 1.903447)(0.053131, 1.803063)(0.096299, 1.602293)(0.139466, 1.490326)(0.182633, 1.382219)(0.225800, 1.131258)(0.268968, 1.038595)(0.312135, 0.938210)(0.355302, 0.583003)(0.398470, 0.478758)(0.441637, 0.281849)(0.484804, 0.227796)(0.527972, 0.177604)(0.571139, 0.084941)(0.614306, 0.042470)(0.657473, 0.011583)(0.700641, 0.003861)(0.743808, 0.003861)
      };
    \end{axis}
  \end{tikzpicture}
\caption{Second informative factor of Class $7$}
\end{subfigure}
\caption{{ Histogram of the first and second information factors of Class $3$ and $7$, estimated using all samples from the Fashion-MNIST dataset. }}
\label{fig:information_factor_3_7}
\end{figure}

Figures~\ref{fig:diff_loss_F-MNIST}~to~\ref{fig:real_VS_Gaussian_equiv_3_7_classifier} then provide empirical results on these two cases to demonstrate the universal or non-universal behavior with respect to the in-distribution performance and the ERM classifier, under the three scenarios on informative factors and loss function listed at the beginning of this section. 

We discuss first Scenario~1 with non-Gaussian informative factors (Case~1 on Classes 4\&5) and non-square losses. Under this scenario, the in-distribution performance is predicted, as per \Cref{cor:condition for r}, to be different from that under the equivalent GMM, as can be observed in the middle and right plots of \Cref{fig:real_VS_Gaussian_equiv_4_5_in_dist}. 
According to \Cref{cor:condition for beta}, the universality on classifier does \emph{not} hold in this case either. 
In other words, the classifier trained on Fashion-MNIST data and the one trained on data drawn from the equivalent GMM give different performances on the \emph{same} test Fashion-MNIST data. 
This is empirically manifested in middle and right plots of \Cref{fig:real_VS_Gaussian_equiv_4_5_classifier} and suggests an effective learning using non-square losses from high-order moment information beyond the class mean and covariance.

It is interesting to compare Scenario~1 with Scenario~2, where we use square and non-square losses on non-Gaussian informative factors. 
In Scenario~2, while we still do \emph{not} have a universal in-distribution performance as evidenced in the left column of \Cref{fig:real_VS_Gaussian_equiv_4_5_in_dist}, the classifier trained on equivalent GMM data gives practically the \emph{same} performance on test Fashion-MNIST data as the classifier trained on realistic Fashion-MNIST data, as shown in the left plot of \Cref{fig:real_VS_Gaussian_equiv_4_5_classifier}.
This means that the square loss \emph{fails} to learn from Fashion-MNIST data beyond the information contained in the equivalent GMM (i.e., the class mean and covariance). 
Consequently, the square loss yields suboptimal performance as shown in the left plot of \Cref{fig:diff_loss_F-MNIST}.

In Scenario~3 with (approximately) Gaussian informative factors (Case~2 on Classes 3\&7), the two types of Gaussian universality (on in-distribution performance and on classifier) hold for any choice of loss function. 
The universality on in-distribution performance is demonstrated in \Cref{fig:real_VS_Gaussian_equiv_3_7_in_dist}, where we observe a much closer match between the in-distribution performance on Fashion-MNIST and the equivalent GMM, in comparison with the drastically different in-distribution performances reported in \Cref{fig:real_VS_Gaussian_equiv_4_5_in_dist} for Case~1 on Classes 4\&5 with non-Gaussian informative factors. 
The universality on classifier that holds for any loss (square or non-square) in this scenario is empirically confirmed by comparing \Cref{fig:real_VS_Gaussian_equiv_3_7_classifier} with \Cref{fig:real_VS_Gaussian_equiv_4_5_classifier}. 
Otherwise speaking, in this case, Fashion-MNIST data are treated by the ERM classifier \emph{as if they were Gaussian mixture data}. 
As a result, we predict that there is little room in trying to do better than the square loss (that is identified by \cite{taheri2021sharp,mai2019high} to be the optimal loss under GMM). 
This is consistent with the empirical performances given by different losses displayed in the right plot of \Cref{fig:diff_loss_F-MNIST}.

\medskip

As a side note, the empirical universal results provided in Figure~2 of \citep{dandi2024universality} also involved Fashion-MNIST data. However as \cite{dandi2024universality} trained the classifier on \emph{synthetic data} generated by a conditional GAN learned from the Fashion-MNIST dataset, these experimental results are not directly comparable to ours, which were obtained from a direct training on Fashion-MNIST data.

\begin{figure}
\centering
\begin{tabular}{cc}
\begin{tikzpicture}[font=\footnotesize]
\renewcommand{\axisdefaulttryminticks}{4} 
\pgfplotsset{every axis legend/.append style={cells={anchor=west},fill=white, at={(0,0)}, anchor=south west, font=\scriptsize}}
\pgfplotsset{/pgfplots/error bars/error bar style={thick}}
\pgfplotsset{every axis plot/.append style={thick},}
\begin{axis}[
height=0.35\linewidth,
width=0.45\linewidth,
ymin=0.97,
ymax=0.995,
xmin=0.0019,
xmax=32,
xmode=log,
log basis x ={2},
grid=major,
ymajorgrids=false,
scaled ticks=true,
xlabel={$\lambda$},
ylabel={ Classification Accuracy ($\%$)  },
]
\addplot[color = BLUE, mark=x, mark size= 1.7pt,error bars/.cd, y dir=both,y explicit] coordinates {
(0.001953,0.970435)(0.003906,0.971436)(0.007812,0.973047)(0.015625,0.975684)(0.031250,0.980078)(0.062500,0.984033)(0.125000,0.987793)(0.250000,0.989868)(0.500000,0.990894)(1.000000,0.990234)(2.000000,0.988110)(4.000000,0.983765)(8.000000,0.976440)(16.000000,0.969043)(32.000000,0.963086)(64.000000,0.958252)      
      };
\addlegendentry{{ Square loss }};
\addplot[color = GREEN, mark=triangle, mark size= 1.7pt,error bars/.cd, y dir=both,y explicit] coordinates {
(0.001953,0.990283)(0.003906,0.989771)(0.007812,0.989185)(0.015625,0.988623)(0.031250,0.987891)(0.062500,0.986499)(0.125000,0.984985)(0.250000,0.981958)(0.500000,0.978149)(1.000000,0.973242)(2.000000,0.967969)(4.000000,0.963989)(8.000000,0.960083)(16.000000,0.956689)(32.000000,0.954492)(64.000000,0.952344)   
      };
\addlegendentry{{ Logistic loss }};
\addplot[color = RED, mark=o, mark size= 1.7pt,error bars/.cd, y dir=both,y explicit] coordinates {
(0.001953,0.993042)(0.003906,0.993042)(0.007812,0.993042)(0.015625,0.992969)(0.031250,0.992896)(0.062500,0.992822)(0.125000,0.992749)(0.250000,0.992432)(0.500000,0.991431)(1.000000,0.989404)(2.000000,0.986255)(4.000000,0.980347)(8.000000,0.973877)(16.000000,0.967822)(32.000000,0.962842)(64.000000,0.958228)
      };
      \addlegendentry{{ Square hinge loss }};
      \end{axis}
      \end{tikzpicture}
&
\begin{tikzpicture}[font=\footnotesize]
\renewcommand{\axisdefaulttryminticks}{4} 
\pgfplotsset{every axis legend/.append style={cells={anchor=east},fill=white, at={(1,0)}, anchor=south east, font=\scriptsize}}
\pgfplotsset{/pgfplots/error bars/error bar style={thick}}
\pgfplotsset{every axis plot/.append style={thick},}
\begin{axis}[
height=0.35\linewidth,
width=0.45\linewidth,
ymin=0.98,
ymax=1,
xmin=0.0019,
xmax=64,
xmode=log,
log basis x ={2},
grid=major,
ymajorgrids=false,
scaled ticks=true,
ytick={0.98,0.99,1},
xlabel={$\lambda$},
ylabel={},
]
\addplot[color = BLUE, mark=x, mark size= 1.7pt,error bars/.cd, y dir=both,y explicit] coordinates {
(0.001953,0.981421)(0.003906,0.982129)(0.007812,0.983643)(0.015625,0.985742)(0.031250,0.988672)(0.062500,0.991846)(0.125000,0.994873)(0.250000,0.997192)(0.500000,0.998315)(1.000000,0.998926)(2.000000,0.998975)(4.000000,0.999097)(8.000000,0.999048)(16.000000,0.998926)(32.000000,0.998071)(64.000000,0.997070)    
      };
\addlegendentry{{ Square loss }};
\addplot[color = GREEN, mark=triangle, mark size= 1.7pt,error bars/.cd, y dir=both,y explicit] coordinates {
(0.001953,0.999121)(0.003906,0.999121)(0.007812,0.999097)(0.015625,0.999048)(0.031250,0.999048)(0.062500,0.999048)(0.125000,0.999023)(0.250000,0.998999)(0.500000,0.998975)(1.000000,0.998877)(2.000000,0.998364)(4.000000,0.997998)(8.000000,0.997412)(16.000000,0.996655)(32.000000,0.995923)(64.000000,0.995215)  
      };
\addlegendentry{{ Logistic loss }};
\addplot[color = RED, mark=o, mark size= 1.7pt,error bars/.cd, y dir=both,y explicit] coordinates {
(0.001953,0.999072)(0.003906,0.999072)(0.007812,0.999072)(0.015625,0.999072)(0.031250,0.999072)(0.062500,0.999072)(0.125000,0.999072)(0.250000,0.999072)(0.500000,0.999072)(1.000000,0.999072)(2.000000,0.999048)(4.000000,0.999023)(8.000000,0.998950)(16.000000,0.998462)(32.000000,0.997998)(64.000000,0.997070)
      };
      \addlegendentry{{ Square hinge loss }};
      \end{axis}
      \end{tikzpicture}
   \end{tabular}
   \caption{ Classification accuracies as a function of the regularization penalty, for square, logistic, and square hinge loss, on Fashion-MNIST data of sample size $n = 512$.
    \textbf{Left}: Class $4$ versus $5$, as an example of non-Gaussian information factors showed in \Cref{fig:information_factor_4_5}.
    \textbf{Right}: Class $3$ versus $7$, as an example of (close-to) Gaussian information factors showed in \Cref{fig:information_factor_3_7}. }
    \label{fig:diff_loss_F-MNIST}
\end{figure} 
\begin{figure}
\centering
\begin{tabular}{ccc}
\begin{tikzpicture}[font=\footnotesize]
\renewcommand{\axisdefaulttryminticks}{4} 
\pgfplotsset{every axis legend/.append style={cells={anchor=west},fill=white, at={(0,0)}, anchor=south west, font=\scriptsize}}
\pgfplotsset{/pgfplots/error bars/error bar style={thick}}
\pgfplotsset{every axis plot/.append style={thick},}
\begin{axis}[
height=0.33\linewidth,
width=0.33\linewidth,
ymin=0.95,
ymax=1,
xmin=0.0019,
xmax=512,
xmode=log,
log basis x ={2},
grid=major,
ymajorgrids=false,
scaled ticks=true,
xlabel={$\lambda$},
ylabel={ Classification Accuracy ($\%$)  },
]
\addplot[color = BLUE, mark=x, mark size= 1.7pt,error bars/.cd, y dir=both,y explicit] coordinates {
(0.001953,0.967505)(0.003906,0.968628)(0.007812,0.970288)(0.015625,0.973120)(0.031250,0.977197)(0.062500,0.981177)(0.125000,0.985059)(0.250000,0.987231)(0.500000,0.988623)(1.000000,0.988208)(2.000000,0.986377)(4.000000,0.982324)(8.000000,0.977075)(16.000000,0.970215)(32.000000,0.964380)(64.000000,0.958472)(128.000000,0.955444)(256.000000,0.953516)(512.000000,0.952271)
      };
\addlegendentry{{ Fashion-MNIST }};
\addplot[color = RED, mark=triangle, mark size= 1.7pt,error bars/.cd, y dir=both,y explicit] coordinates {
(0.001953,0.976294)(0.003906,0.977124)(0.007812,0.978931)(0.015625,0.981787)(0.031250,0.985376)(0.062500,0.989478)(0.125000,0.993091)(0.250000,0.994897)(0.500000,0.996094)(1.000000,0.996143)(2.000000,0.995654)(4.000000,0.993994)(8.000000,0.991187)(16.000000,0.987305)(32.000000,0.982983)(64.000000,0.978589)(128.000000,0.976172)(256.000000,0.974487)(512.000000,0.973462)
      };
\addlegendentry{{ Equivalent GMM }};
      \end{axis}
      \end{tikzpicture}
&
\begin{tikzpicture}[font=\footnotesize]
\renewcommand{\axisdefaulttryminticks}{4} 
\pgfplotsset{every axis legend/.append style={cells={anchor=west},fill=white, at={(0,0)}, anchor=south west, font=\scriptsize}}
\pgfplotsset{/pgfplots/error bars/error bar style={thick}}
\pgfplotsset{every axis plot/.append style={thick},}
\begin{axis}[
height=0.33\linewidth,
width=0.33\linewidth,
ymin=0.95,
ymax=1,
xmin=0.0019,
xmax=512,
xmode=log,
log basis x ={2},
grid=major,
ymajorgrids=false,
scaled ticks=true,
xlabel={$\lambda$},
ylabel={},
]
\addplot[color = BLUE, mark=x, mark size= 1.7pt,error bars/.cd, y dir=both,y explicit] coordinates {
(0.001953,0.988354)(0.003906,0.987842)(0.007812,0.987280)(0.015625,0.986475)(0.031250,0.985645)(0.062500,0.984351)(0.125000,0.982764)(0.250000,0.980469)(0.500000,0.976636)(1.000000,0.972559)(2.000000,0.968384)(4.000000,0.963696)(8.000000,0.959473)(16.000000,0.956226)(32.000000,0.954224)(64.000000,0.952563)(128.000000,0.951758)(256.000000,0.951001)(512.000000,0.950684)
      };
\addlegendentry{{ Fashion-MNIST }};
\addplot[color = RED, mark=triangle, mark size= 1.7pt,error bars/.cd, y dir=both,y explicit] coordinates {
(0.001953,0.994092)(0.003906,0.993774)(0.007812,0.993579)(0.015625,0.993262)(0.031250,0.992896)(0.062500,0.992334)(0.125000,0.991406)(0.250000,0.990210)(0.500000,0.988550)(1.000000,0.986401)(2.000000,0.984180)(4.000000,0.981812)(8.000000,0.979370)(16.000000,0.977148)(32.000000,0.975293)(64.000000,0.974146)(128.000000,0.973340)(256.000000,0.972876)(512.000000,0.972656)
      };
\addlegendentry{{ Equivalent GMM }};
\end{axis}
\end{tikzpicture}
&
\begin{tikzpicture}[font=\footnotesize]
\renewcommand{\axisdefaulttryminticks}{4} 
\pgfplotsset{every axis legend/.append style={cells={anchor=west},fill=white, at={(0,0)}, anchor=south west, font=\scriptsize}}
\pgfplotsset{/pgfplots/error bars/error bar style={thick}}
\pgfplotsset{every axis plot/.append style={thick},}
\begin{axis}[
height=0.33\linewidth,
width=0.33\linewidth,
ymin=0.95,
ymax=.995,
xmin=0.0019,
xmax=512,
xmode=log,
log basis x ={2},
grid=major,
ymajorgrids=false,
scaled ticks=true,
xlabel={$\lambda$},
ylabel={},
]
\addplot[color = BLUE, mark=x, mark size= 1.7pt,error bars/.cd, y dir=both,y explicit] coordinates {
(0.001953,0.991211)(0.003906,0.991211)(0.007812,0.991211)(0.015625,0.991284)(0.031250,0.991235)(0.062500,0.991187)(0.125000,0.991016)(0.250000,0.990771)(0.500000,0.990259)(1.000000,0.988672)(2.000000,0.984644)(4.000000,0.979468)(8.000000,0.973975)(16.000000,0.967798)(32.000000,0.962769)(64.000000,0.957080)(128.000000,0.954395)(256.000000,0.952393)(512.000000,0.951196)
      };
\addlegendentry{{ Fashion-MNIST }};
\addplot[color = RED, mark=triangle, mark size= 1.7pt,error bars/.cd, y dir=both,y explicit] coordinates {
(0.001953,0.994629)(0.003906,0.994604)(0.007812,0.994580)(0.015625,0.994604)(0.031250,0.994580)(0.062500,0.994556)(0.125000,0.994434)(0.250000,0.994141)(0.500000,0.993872)(1.000000,0.992993)(2.000000,0.991357)(4.000000,0.989478)(8.000000,0.987061)(16.000000,0.983765)(32.000000,0.980981)(64.000000,0.978223)(128.000000,0.976001)(256.000000,0.974609)(512.000000,0.973608)
      };
\addlegendentry{{ Equivalent GMM }};
\end{axis}
\end{tikzpicture}
\end{tabular}
\caption{ \textbf{In-distribution} classification accuracies as a function of the regularization penalty $\gamma$, for Fashion-MNIST data (Class $4$ versus $5$, as an illustrating example of non-Gaussian informative factors as shown in \Cref{fig:information_factor_4_5}) and Equivalent GMM of sample size $n = 512$, with square (\textbf{left}), logistic (\textbf{middle}), and square hinge (\textbf{right}) losses. }
\label{fig:real_VS_Gaussian_equiv_4_5_in_dist}
\end{figure} 
\begin{figure}
\centering
\begin{tabular}{ccc}
\begin{tikzpicture}[font=\footnotesize]
\renewcommand{\axisdefaulttryminticks}{4} 
\pgfplotsset{every axis legend/.append style={cells={anchor=west},fill=white, at={(0,0)}, anchor=south west, font=\scriptsize}}
\pgfplotsset{/pgfplots/error bars/error bar style={thick}}
\pgfplotsset{every axis plot/.append style={thick},}
\begin{axis}[
height=0.33\linewidth,
width=0.33\linewidth,
xmin=0.0019,
xmax=512,
xmode=log,
log basis x ={2},
grid=major,
ymajorgrids=false,
scaled ticks=true,
ytick={0.99,1},
xlabel={$\lambda$},
ylabel={ Classification Accuracy ($\%$)  },
]
\addplot[color = BLUE, mark=x, mark size= 1.7pt,error bars/.cd, y dir=both,y explicit] coordinates {
(0.001953,0.981787)(0.003906,0.982275)(0.007812,0.983545)(0.015625,0.985840)(0.031250,0.988745)(0.062500,0.992090)(0.125000,0.995264)(0.250000,0.997412)(0.500000,0.998560)(1.000000,0.998975)(2.000000,0.999243)(4.000000,0.999341)(8.000000,0.999268)(16.000000,0.998950)(32.000000,0.997974)(64.000000,0.996802)(128.000000,0.995825)(256.000000,0.994702)(512.000000,0.994165)
      };
\addlegendentry{{ Fashion-MNIST }};
\addplot[color = RED, mark=triangle, mark size= 1.7pt,error bars/.cd, y dir=both,y explicit] coordinates {
(0.001953,0.987573)(0.003906,0.988281)(0.007812,0.989624)(0.015625,0.991260)(0.031250,0.993408)(0.062500,0.995654)(0.125000,0.997754)(0.250000,0.998877)(0.500000,0.999438)(1.000000,0.999634)(2.000000,0.999658)(4.000000,0.999561)(8.000000,0.999243)(16.000000,0.998706)(32.000000,0.997681)(64.000000,0.996143)(128.000000,0.994922)(256.000000,0.994189)(512.000000,0.993677)
      };
\addlegendentry{{ Equivalent GMM }};
      \end{axis}
      \end{tikzpicture}
&
\begin{tikzpicture}[font=\footnotesize]
\renewcommand{\axisdefaulttryminticks}{4} 
\pgfplotsset{every axis legend/.append style={cells={anchor=west},fill=white, at={(0,0)}, anchor=south west, font=\scriptsize}}
\pgfplotsset{/pgfplots/error bars/error bar style={thick}}
\pgfplotsset{every axis plot/.append style={thick},}
\begin{axis}[
height=0.33\linewidth,
width=0.33\linewidth,
xmin=0.0019,
xmax=512,
xmode=log,
log basis x ={2},
grid=major,
ymajorgrids=false,
scaled ticks=true,
ytick={0.99,1},
xlabel={$\lambda$},
ylabel={},
]
\addplot[color = BLUE, mark=x, mark size= 1.7pt,error bars/.cd, y dir=both,y explicit] coordinates {
(0.001953,0.999170)(0.003906,0.999170)(0.007812,0.999170)(0.015625,0.999146)(0.031250,0.999121)(0.062500,0.999121)(0.125000,0.999048)(0.250000,0.998950)(0.500000,0.998853)(1.000000,0.998535)(2.000000,0.998096)(4.000000,0.997485)(8.000000,0.996582)(16.000000,0.995142)(32.000000,0.994507)(64.000000,0.993628)(128.000000,0.992993)(256.000000,0.992773)(512.000000,0.992651)
      };
\addlegendentry{{ Fashion-MNIST }};
\addplot[color = RED, mark=triangle, mark size= 1.7pt,error bars/.cd, y dir=both,y explicit] coordinates {
(0.001953,0.999512)(0.003906,0.999463)(0.007812,0.999438)(0.015625,0.999414)(0.031250,0.999390)(0.062500,0.999292)(0.125000,0.999170)(0.250000,0.998999)(0.500000,0.998828)(1.000000,0.998413)(2.000000,0.998022)(4.000000,0.997461)(8.000000,0.996851)(16.000000,0.996387)(32.000000,0.995752)(64.000000,0.994995)(128.000000,0.994385)(256.000000,0.994141)(512.000000,0.994019)
      };
\addlegendentry{{ Equivalent GMM }};
\end{axis}
\end{tikzpicture}
&
\begin{tikzpicture}[font=\footnotesize]
\renewcommand{\axisdefaulttryminticks}{4} 
\pgfplotsset{every axis legend/.append style={cells={anchor=west},fill=white, at={(0,0)}, anchor=south west, font=\scriptsize}}
\pgfplotsset{/pgfplots/error bars/error bar style={thick}}
\pgfplotsset{every axis plot/.append style={thick},}
\begin{axis}[
height=0.33\linewidth,
width=0.33\linewidth,
xmin=0.0019,
xmax=512,
xmode=log,
log basis x ={2},
grid=major,
ymajorgrids=false,
scaled ticks=true,
ytick={0.99,1},
xlabel={$\lambda$},
ylabel={},
]
\addplot[color = BLUE, mark=x, mark size= 1.7pt,error bars/.cd, y dir=both,y explicit] coordinates {
(0.001953,0.999341)(0.003906,0.999341)(0.007812,0.999341)(0.015625,0.999341)(0.031250,0.999341)(0.062500,0.999341)(0.125000,0.999365)(0.250000,0.999390)(0.500000,0.999390)(1.000000,0.999316)(2.000000,0.999292)(4.000000,0.999268)(8.000000,0.999146)(16.000000,0.998828)(32.000000,0.998071)(64.000000,0.997095)(128.000000,0.995850)(256.000000,0.994995)(512.000000,0.994653)
      };
\addlegendentry{{ Fashion-MNIST }};
\addplot[color = RED, mark=triangle, mark size= 1.7pt,error bars/.cd, y dir=both,y explicit] coordinates {
(0.001953,0.999707)(0.003906,0.999707)(0.007812,0.999707)(0.015625,0.999707)(0.031250,0.999707)(0.062500,0.999707)(0.125000,0.999707)(0.250000,0.999707)(0.500000,0.999683)(1.000000,0.999658)(2.000000,0.999609)(4.000000,0.999390)(8.000000,0.998950)(16.000000,0.998511)(32.000000,0.997803)(64.000000,0.996606)(128.000000,0.995190)(256.000000,0.994312)(512.000000,0.993774)
      };
\addlegendentry{{ Equivalent GMM }};
\end{axis}
\end{tikzpicture}
\end{tabular}
\caption{ \textbf{In-distribution} classification accuracies as a function of the regularization penalty $\gamma$, for Fashion-MNIST data (Class $3$ versus $7$, as an illustrating example of close-to-Gaussian informative factors as shown in \Cref{fig:information_factor_3_7}) and Equivalent GMM of sample size $n = 512$, with square (\textbf{left}), logistic (\textbf{middle}), and square hinge (\textbf{right}) losses. }
\label{fig:real_VS_Gaussian_equiv_3_7_in_dist}
\end{figure} 
\begin{figure}
\centering
\begin{tabular}{ccc}
\begin{tikzpicture}[font=\footnotesize]
\renewcommand{\axisdefaulttryminticks}{4} 
\pgfplotsset{every axis legend/.append style={cells={anchor=west},fill=white, at={(0,0)}, anchor=south west, font=\scriptsize}}
\pgfplotsset{/pgfplots/error bars/error bar style={thick}}
\pgfplotsset{every axis plot/.append style={thick},}
\begin{axis}[
height=0.33\linewidth,
width=0.33\linewidth,
ymin=0.95,
ymax=.99,
xmin=0.0019,
xmax=512,
xmode=log,
log basis x ={2},
grid=major,
ymajorgrids=false,
scaled ticks=true,
xlabel={$\lambda$},
ylabel={ Classification Accuracy ($\%$)  },
]
\addplot[color = BLUE, mark=x, mark size= 1.7pt,error bars/.cd, y dir=both,y explicit] coordinates {
(0.001953,0.969116)(0.003906,0.970190)(0.007812,0.971729)(0.015625,0.974658)(0.031250,0.978418)(0.062500,0.982202)(0.125000,0.985669)(0.250000,0.987915)(0.500000,0.988501)(1.000000,0.988159)(2.000000,0.986475)(4.000000,0.981787)(8.000000,0.975195)(16.000000,0.968896)(32.000000,0.963135)(64.000000,0.957788)(128.000000,0.954639)(256.000000,0.952686)(512.000000,0.951514)
      };
\addlegendentry{{ Fashion-MNIST }};
\addplot[color = RED, mark=triangle, mark size= 1.7pt,error bars/.cd, y dir=both,y explicit] coordinates {
(0.001953,0.969775)(0.003906,0.970923)(0.007812,0.972607)(0.015625,0.974951)(0.031250,0.979053)(0.062500,0.983350)(0.125000,0.986816)(0.250000,0.989038)(0.500000,0.989990)(1.000000,0.989258)(2.000000,0.987305)(4.000000,0.982837)(8.000000,0.976343)(16.000000,0.969409)(32.000000,0.963477)(64.000000,0.958032)(128.000000,0.954688)(256.000000,0.953003)(512.000000,0.951660)
      };
\addlegendentry{{ Equivalent GMM }};
      \end{axis}
      \end{tikzpicture}
&
\begin{tikzpicture}[font=\footnotesize]
\renewcommand{\axisdefaulttryminticks}{4} 
\pgfplotsset{every axis legend/.append style={cells={anchor=west},fill=white, at={(0,0)}, anchor=south west, font=\scriptsize}}
\pgfplotsset{/pgfplots/error bars/error bar style={thick}}
\pgfplotsset{every axis plot/.append style={thick},}
\begin{axis}[
height=0.33\linewidth,
width=0.33\linewidth,
ymin=0.95,
ymax=.99,
xmin=0.0019,
xmax=512,
xmode=log,
log basis x ={2},
grid=major,
ymajorgrids=false,
scaled ticks=true,
xlabel={$\lambda$},
ylabel={},
]
\addplot[color = BLUE, mark=x, mark size= 1.7pt,error bars/.cd, y dir=both,y explicit] coordinates {
(0.001953,0.987695)(0.003906,0.987500)(0.007812,0.987036)(0.015625,0.986572)(0.031250,0.985693)(0.062500,0.984473)(0.125000,0.983032)(0.250000,0.980713)(0.500000,0.976953)(1.000000,0.972144)(2.000000,0.967383)(4.000000,0.964014)(8.000000,0.959888)(16.000000,0.956909)(32.000000,0.955029)(64.000000,0.953296)(128.000000,0.952368)(256.000000,0.951953)(512.000000,0.951807)
      };
\addlegendentry{{ Fashion-MNIST }};
\addplot[color = RED, mark=triangle, mark size= 1.7pt,error bars/.cd, y dir=both,y explicit] coordinates {
(0.001953,0.982031)(0.003906,0.981396)(0.007812,0.980615)(0.015625,0.979834)(0.031250,0.978491)(0.062500,0.977100)(0.125000,0.975049)(0.250000,0.973267)(0.500000,0.970508)(1.000000,0.967603)(2.000000,0.964453)(4.000000,0.961792)(8.000000,0.958911)(16.000000,0.956860)(32.000000,0.955054)(64.000000,0.953369)(128.000000,0.952393)(256.000000,0.951978)(512.000000,0.951660)
      };
\addlegendentry{{ Equivalent GMM }};
\end{axis}
\end{tikzpicture}
&
\begin{tikzpicture}[font=\footnotesize]
\renewcommand{\axisdefaulttryminticks}{4} 
\pgfplotsset{every axis legend/.append style={cells={anchor=west},fill=white, at={(0,0)}, anchor=south west, font=\scriptsize}}
\pgfplotsset{/pgfplots/error bars/error bar style={thick}}
\pgfplotsset{every axis plot/.append style={thick},}
\begin{axis}[
height=0.33\linewidth,
width=0.33\linewidth,
ymin=0.95,
ymax=.995,
xmin=0.0019,
xmax=512,
xmode=log,
log basis x ={2},
grid=major,
ymajorgrids=false,
scaled ticks=true,
xlabel={$\lambda$},
ylabel={},
]
\addplot[color = BLUE, mark=x, mark size= 1.7pt,error bars/.cd, y dir=both,y explicit] coordinates {
(0.001953,0.990356)(0.003906,0.990356)(0.007812,0.990356)(0.015625,0.990332)(0.031250,0.990332)(0.062500,0.990234)(0.125000,0.990015)(0.250000,0.989697)(0.500000,0.989136)(1.000000,0.987500)(2.000000,0.984253)(4.000000,0.978882)(8.000000,0.972290)(16.000000,0.966626)(32.000000,0.961914)(64.000000,0.957007)(128.000000,0.953931)(256.000000,0.952002)(512.000000,0.950488)
      };
\addlegendentry{{ Fashion-MNIST }};
\addplot[color = RED, mark=triangle, mark size= 1.7pt,error bars/.cd, y dir=both,y explicit] coordinates {
(0.001953,0.985400)(0.003906,0.985425)(0.007812,0.985449)(0.015625,0.985400)(0.031250,0.985327)(0.062500,0.985059)(0.125000,0.984766)(0.250000,0.984473)(0.500000,0.983228)(1.000000,0.981128)(2.000000,0.978003)(4.000000,0.973291)(8.000000,0.968677)(16.000000,0.964673)(32.000000,0.960376)(64.000000,0.956494)(128.000000,0.953687)(256.000000,0.951855)(512.000000,0.950659)
      };
\addlegendentry{{ Equivalent GMM }};
\end{axis}
\end{tikzpicture}
\end{tabular}
\caption{ Classification accuracies on test Fashion-MNIST data as a function of the regularization penalty $\gamma$, given by classifiers trained on Fashion-MNIST data (Class $4$ versus $5$, as an illustrating example of non-Gaussian informative factors as shown in \Cref{fig:information_factor_4_5}) and on Equivalent GMM data of sample size $n = 512$, with square (\textbf{left}), logistic (\textbf{middle}), and square hinge (\textbf{right}) losses. }
\label{fig:real_VS_Gaussian_equiv_4_5_classifier}
\end{figure} 
\begin{figure}
\centering
\begin{tabular}{ccc}
\begin{tikzpicture}[font=\footnotesize]
\renewcommand{\axisdefaulttryminticks}{4} 
\pgfplotsset{every axis legend/.append style={cells={anchor=west},fill=white, at={(0,0)}, anchor=south west, font=\scriptsize}}
\pgfplotsset{/pgfplots/error bars/error bar style={thick}}
\pgfplotsset{every axis plot/.append style={thick},}
\begin{axis}[
height=0.33\linewidth,
width=0.35\linewidth,
xmin=0.0019,
xmax=512,
xmode=log,
log basis x ={2},
grid=major,
ymajorgrids=false,
scaled ticks=true,
ytick={0.99,1},
xlabel={$\lambda$},
ylabel={ Classification Accuracy ($\%$)  },
]
\addplot[color = BLUE, mark=x, mark size= 1.7pt,error bars/.cd, y dir=both,y explicit] coordinates {
(0.001953,0.984961)(0.003906,0.985547)(0.007812,0.986499)(0.015625,0.988428)(0.031250,0.990967)(0.062500,0.993823)(0.125000,0.996118)(0.250000,0.997803)(0.500000,0.998560)(1.000000,0.998853)(2.000000,0.998975)(4.000000,0.998999)(8.000000,0.999121)(16.000000,0.998828)(32.000000,0.998096)(64.000000,0.997119)(128.000000,0.995605)(256.000000,0.994482)(512.000000,0.994116)
      };
\addlegendentry{{ Fashion-MNIST }};
\addplot[color = RED, mark=triangle, mark size= 1.7pt,error bars/.cd, y dir=both,y explicit] coordinates {
(0.001953,0.986499)(0.003906,0.987134)(0.007812,0.988159)(0.015625,0.989746)(0.031250,0.992114)(0.062500,0.994800)(0.125000,0.997021)(0.250000,0.998218)(0.500000,0.998950)(1.000000,0.999268)(2.000000,0.999292)(4.000000,0.999219)(8.000000,0.999219)(16.000000,0.999097)(32.000000,0.998315)(64.000000,0.997144)(128.000000,0.995581)(256.000000,0.994751)(512.000000,0.994165)
      };
\addlegendentry{{ Equivalent GMM }};
      \end{axis}
      \end{tikzpicture}
&
\begin{tikzpicture}[font=\footnotesize]
\renewcommand{\axisdefaulttryminticks}{4} 
\pgfplotsset{every axis legend/.append style={cells={anchor=west},fill=white, at={(0,0)}, anchor=south west, font=\scriptsize}}
\pgfplotsset{/pgfplots/error bars/error bar style={thick}}
\pgfplotsset{every axis plot/.append style={thick},}
\begin{axis}[
height=0.33\linewidth,
width=0.35\linewidth,
xmin=0.0019,
xmax=512,
xmode=log,
log basis x ={2},
grid=major,
ymajorgrids=false,
scaled ticks=true,
ytick={0.99,1},
xlabel={$\lambda$},
ylabel={},
]
\addplot[color = BLUE, mark=x, mark size= 1.7pt,error bars/.cd, y dir=both,y explicit] coordinates {
(0.001953,0.999463)(0.003906,0.999463)(0.007812,0.999463)(0.015625,0.999463)(0.031250,0.999463)(0.062500,0.999463)(0.125000,0.999438)(0.250000,0.999463)(0.500000,0.999414)(1.000000,0.999268)(2.000000,0.999048)(4.000000,0.998340)(8.000000,0.997607)(16.000000,0.996484)(32.000000,0.995435)(64.000000,0.994629)(128.000000,0.994116)(256.000000,0.993848)(512.000000,0.993701)
      };
\addlegendentry{{ Fashion-MNIST }};
\addplot[color = RED, mark=triangle, mark size= 1.7pt,error bars/.cd, y dir=both,y explicit] coordinates {
(0.001953,0.999390)(0.003906,0.999390)(0.007812,0.999390)(0.015625,0.999390)(0.031250,0.999390)(0.062500,0.999341)(0.125000,0.999316)(0.250000,0.999268)(0.500000,0.999097)(1.000000,0.999023)(2.000000,0.998730)(4.000000,0.998340)(8.000000,0.997583)(16.000000,0.996631)(32.000000,0.995703)(64.000000,0.994971)(128.000000,0.994604)(256.000000,0.994458)(512.000000,0.994312)
      };
\addlegendentry{{ Equivalent GMM }};
\end{axis}
\end{tikzpicture}
&
\begin{tikzpicture}[font=\footnotesize]
\renewcommand{\axisdefaulttryminticks}{4} 
\pgfplotsset{every axis legend/.append style={cells={anchor=west},fill=white, at={(0,0)}, anchor=south west, font=\scriptsize}}
\pgfplotsset{/pgfplots/error bars/error bar style={thick}}
\pgfplotsset{every axis plot/.append style={thick},}
\begin{axis}[
height=0.33\linewidth,
width=0.35\linewidth,
xmin=0.0019,
xmax=512,
xmode=log,
log basis x ={2},
grid=major,
ymajorgrids=false,
scaled ticks=true,
ytick={0.99,1},
xlabel={$\lambda$},
ylabel={},
]
\addplot[color = BLUE, mark=x, mark size= 1.7pt,error bars/.cd, y dir=both,y explicit] coordinates {
(0.001953,0.999414)(0.003906,0.999414)(0.007812,0.999414)(0.015625,0.999414)(0.031250,0.999414)(0.062500,0.999438)(0.125000,0.999438)(0.250000,0.999390)(0.500000,0.999438)(1.000000,0.999414)(2.000000,0.999438)(4.000000,0.999414)(8.000000,0.999243)(16.000000,0.998950)(32.000000,0.998193)(64.000000,0.996899)(128.000000,0.995239)(256.000000,0.994263)(512.000000,0.993628)
      };
\addlegendentry{{ Fashion-MNIST }};
\addplot[color = RED, mark=triangle, mark size= 1.7pt,error bars/.cd, y dir=both,y explicit] coordinates {
(0.001953,0.999097)(0.003906,0.999097)(0.007812,0.999097)(0.015625,0.999121)(0.031250,0.999121)(0.062500,0.999121)(0.125000,0.999121)(0.250000,0.999121)(0.500000,0.999146)(1.000000,0.999121)(2.000000,0.999194)(4.000000,0.999194)(8.000000,0.999048)(16.000000,0.998657)(32.000000,0.997998)(64.000000,0.996851)(128.000000,0.995581)(256.000000,0.994214)(512.000000,0.993750)
      };
\addlegendentry{{ Equivalent GMM }};
\end{axis}
\end{tikzpicture}
\end{tabular}
\caption{ Classification accuracies on test Fashion-MNIST data as a function of the regularization penalty $\gamma$, given by ERM classifiers trained on Fashion-MNIST data (Class $3$ versus $7$, as an illustrating example of close-to-Gaussian informative factors as shown in \Cref{fig:information_factor_3_7}) on Equivalent GMM data of sample size $n = 512$, with square (\textbf{left}), logistic (\textbf{middle}), and square hinge (\textbf{right}) losses. }
\label{fig:real_VS_Gaussian_equiv_3_7_classifier}
\end{figure}

\end{document}